\documentclass[accepted]{uai2024} % after acceptance, for a revised version; 
\usepackage[american]{babel}
\usepackage{amsmath}
\usepackage{algorithm}
\usepackage{algorithmic}
\usepackage{comment}
\usepackage{apacite}
\usepackage{hyperref}
\usepackage{bm}
\usepackage{amssymb}
\usepackage{graphicx}
\usepackage{authblk}
\usepackage{amssymb}
\usepackage[table]{xcolor}
\usepackage{natbib} % has a nice set of citation styles and commands
\usepackage{mathtools} % amsmath with fixes and additions
\usepackage{booktabs} % commands to create good-looking tables
\usepackage{tikz} % nice language for creating drawings and diagrams

\title{Finite-Time Analysis of Three-Timescale Constrained Actor-Critic and Constrained Natural Actor-Critic Algorithms}

\author[1]{\href{mailto:<prashansap@iisc.ac.in>}{Prashansa Panda}}
\author[1]{\href{mailto:<shalabh@iisc.ac.in>}{Shalabh Bhatnagar}}

\affil[1]{%
   Department of Computer Science and Automation\\
    Indian Institute of Science\\
    Bangalore, India%\\ \\
}

\begin{document}
\newcommand{\Ab}{\mathbf{A}}
\newcommand{\EE}{\mathbb{E}}
\newcommand{\btheta}{\bm{\theta}}
\newcommand{\bgamma}{\bm{\gamma}}
\newcommand{\bbb}{\mathbf{b}}
\newcommand{\cP}{\mathcal{P}}
\newcommand{\RR}{\mathbb{R}}
\newcommand{\cS}{{\mathcal{S}}}
\newcommand{\cE}{\mathcal{E}}
\newcommand{\PP}{\mathbb{P}}
\newcommand{\norm}[1]{\|#1\|}
\newcommand{\cO}{\mathcal{O}}
\newcommand{\customcite}[1]{\textbf{\citeauthor{1} (\citeyear{1})}}

\ifx\theorem\undefined
\newtheorem{theorem}{Theorem}

\newtheorem{lemma}{Lemma}
\newtheorem{sublemma}{lemma}[lemma]

\ifx\remark\undefined
\newtheorem{remark}{Remark}

\ifx\proposition\undefined
\newtheorem{proposition}{Proposition}

\ifx\assumption\undefined
\newtheorem{assumption}{Assumption}

\ifx\corollary\undefined
\newtheorem{corollary}{Corollary}  
\maketitle

\begin{abstract}
  Actor Critic methods have found immense applications on a wide range of Reinforcement Learning tasks especially when the state-action space is large. In this paper, we consider actor critic and natural actor critic algorithms with function approximation for constrained Markov decision processes (C-MDP) involving inequality constraints and carry out a non-asymptotic analysis for both  of these algorithms in a non-i.i.d (Markovian) setting. We consider the long-run average cost criterion where both the objective and the constraint functions are suitable policy-dependent long-run averages of certain prescribed cost  functions. We handle the inequality constraints using the Lagrange multiplier method.  We prove that these algorithms are guaranteed to find a first-order stationary point (i.e., $\Vert \nabla L(\theta,\gamma)\Vert_2^2 \leq \epsilon$) of the  performance (Lagrange) function $L(\theta,\gamma)$, with a sample complexity of $\mathcal{\tilde{O}}(\epsilon^{-2.5})$  in the case of both Constrained Actor Critic (C-AC)  and Constrained Natural Actor Critic (C-NAC) algorithms. 
  We also show the results of experiments on three different Safety-Gym environments.
\end{abstract}

\section{Introduction}\label{sec:intro}
In recent times, there has been significant research activity on constrained reinforcement learning algorithms,  motivated largely from applications in safe reinforcement learning (Safe-RL). Each state transition here not just receives a single-stage cost indicating the desirability of the action and the resulting next state but also receives additional constraint (single-stage) costs that may account for safety of the chosen action and the resulting next state. The goal then is to minimize a `long-term cost' criterion while ensuring at the same time that the `long-term constraint costs' stay within certain prescribed thresholds. The problem setting could generally involve more than one constraint cost. Problems of Safe-RL can in general be formulated in the setting of constrained Markov decision processes (C-MDP),
see \cite{hasan,jayant,wachi}. \cite{altman} provides a textbook treatment on C-MDP. As an example, one may consider the problem of navigation of an autonomous vehicle such as a drone or a self-driving car where the goal is to reach the destination in as short a time as possible while ensuring there are no collisions with obstacles or accidents on the way. Such problems can be well formulated in the setting of C-MDPs.
Constrained MDPs find several applications in various diverse domains. Some of these applications include finding optimal bandwidth allocation in resource constrained communication networks, determining strategic building and maintenance policies, safe navigation for self-driving cars, drones and robots, optimal energy sharing strategies in energy harvesting networks etc. 

In situations where the system model is unavailable, but one has access to data in the form of tuples of states, actions, rewards, penalties, as well as next states, one may formulate and present constrained reinforcement learning (C-RL) algorithms for finding the optimal policies. An important algorithm in this category is the constrained actor critic (C-AC) algorithm originally presented by \cite{borkarc-ac} for the long-run average cost setting but for look-up table representations. In \cite{bhatnagar_2010, cmdpsa}, C-AC algorithms with function approximation have been presented and analysed for the infinite horizon discounted cost and the long-run average cost objectives respectively.
In \cite{cmdpsa}, an application of the presented algorithm is also studied empirically on a constrained multi-stage routing problem.

The key idea in the aforementioned algorithms has been to relax the constraints into the objective by forming a Lagrangian and then perform a gradient ascent step in the policy parameter while simultaneously performing a descent in the Lagrange parameter. Note that usual actor-critic (\cite{first_actor_critic}) or natural actor critic algorithms (\cite{nac})  ordinarily require two timescale recursions. This is because these algorithms try and mimic the policy iteration procedure whereby the actor or policy parameter update proceeds on the slower timescale while the critic or value function parameter is updated on the faster timescale. In constrained actor critic and constrained natural actor critic algorithms, one needs to introduce an additional (the slowest) timescale over which the Lagrange multiplier is updated.

\begin{table*}
  \centering
  \caption{Comparison of  Finite-Time Analysis of Various Actor-Critic Algorithms.  }\label{table:comparision_fta}
  \begin{tabular}{|p{3.7 cm}|p{3 cm}|p{2cm}|p{1.5 cm}|p{3 cm}|p{2 cm}|}
  \toprule
  \textbf{Reference} & \textbf{Algorithm} & \textbf{Sample Complexity} & \textbf{Function Approximation} & \textbf{Timescales} & \textbf{Critic estimation}\\ \hline
  \cite{fta_2_timescale} & Actor-Critic & $\widetilde{\mathcal{O}}(\epsilon^{-2.5})$ & \checkmark & two-timescale & TD(0) \\ \hline
  \cite{chen_zhao} & Actor-Critic & $\widetilde{\mathcal{O}}(\epsilon^{-2})$ & \checkmark & single timescale & TD(0) \\
    \hline
    \cite{online_prime_dual} & Constrained NAC & $\mathcal{O}(\epsilon^{-6})$ & $\times$ & two time-scale & TD(0)\\ 
    \hline
    \cite{suttleetal} & Actor-Critic & $\widetilde{\mathcal{O}}(\tau^{2}_{mix}\epsilon^{-2})$ & \checkmark & two-timescale & Monte-Carlo\\
    \hline
    \rowcolor{blue!10} Our work & Constrained AC & $\widetilde{\mathcal{O}}(\epsilon^{-2.5})$ & \checkmark & three-timescale & TD(0)\\ \hline
    \rowcolor{blue!10} Our work & Constrained NAC & $\widetilde{\mathcal{O}}(\epsilon^{-2.5})$ & \checkmark & three-timescale & TD(0)\\
  \bottomrule
  \end{tabular}
\end{table*}

In this paper, we carry out a finite-time (non-asymptotic) convergence analysis of three-timescale constrained actor-critic and constrained natural actor-critic algorithms to find the sample complexity  of these algorithms (in the constrained setting). We assume that the system model via the transition probabilities is not known and linear function approximation is used for the critic recursion. A non-asymptotic analysis helps provide an estimate of the number of  samples needed for the algorithm to converge as well as helps provide  appropriate learning rates for the algorithm. In the case of the C-NAC algorithm, the natural gradient is estimated by linearly transforming the regular gradient by making use of the inverse Fisher information matrix of the policy which is clearly positive definite, see \cite{kakade_2001}. It is generally observed that using natural gradients speeds up the performance of these algorithms. %The result of the policy update should not be affected by  policy parameterization. This brings out the need for using natural gradient.
To the best of our knowledge, a non-asymptotic convergence analysis of three-timescale constrained actor-critic and constrained natural actor-critic algorithms using linear function approximation has not been carried out in the past.

 %In this paper, we provide a non-asymptotic analysis analysed this algorithm in non asymptotic manner  using TD(0) for critic recursion.
We summarise our principal contributions below:\\
\textit{(a)} We carry out the first finite-time analyses for the Constrained  Actor-Critic and the Constrained Natural Actor-Critic algorithms  with linear function approximation in the long-run average cost setting. \\
\textit{(b)} We conduct the aforementioned analyses under the general assumption of Markovian sampling using TD(0) for the critic recursion and obtain a sample complexity of $\widetilde{\mathcal{O}}(\epsilon^{-2.5})$ for both algorithms to find an $\epsilon$-optimal stationary point of the performance function.\\
\textit{(c)} It is important to note here that the sample complexity of both our constrained algorithms matches exactly the one obtained by \cite{fta_2_timescale} which is also $\mathcal{\tilde{O}}(\epsilon^{-2.5})$ even though the latter has been obtained for the case of two-timescale (unconstrained) regular actor-critic algorithms. Further, our setting is more general as we consider random single-stage costs and also constraint costs having distributions that are dependent on the current state, action and next state. This is unlike \cite{fta_2_timescale} where the single-stage reward is assumed fixed for the given current state and action. Our result thus shows that under a random cost structure, with  inequality constraints in the formulation, and having three-timescale algorithms as a result (instead of two-timescale algorithms) has no impact on the sample complexity which we believe is a significant outcome of our study.\\
\textit{(d)} We show the results of experiments on three different Safety-Gym environments, namely SafetyPointGoal1-v0, SafetyCarGoal1-v0 and SafetyPointPush1-v0, respectively, where we compare the performance of C-AC and C-NAC with Constrained DQN (C-DQN). For the latter, we incorporate the Lagrange based procedure and update the Lagrange multipliers in the same way as the C-AC and C-NAC procedures for a fair comparison. We observe that C-NAC shows the best results on all three settings while C-AC is the second-best performer on two of the three settings. Further, all algorithms satisfy the specified average cost constraint. 

\noindent\textbf{Notation:}
For two sequences $\{c_n\}$ and $\{d_n\}$, we can write $c_n = \mathcal{O}(d_n)$ if there exists a constant $P>0$ such that ${\displaystyle \frac{|c_n|}{|d_n|} \le P}$, $\forall n$. To further hide logarithm factors,we use the notation $\tilde{\mathcal{O}}(\cdot)$. Without any other specification, $\|\cdot\|$ denotes the $\ell_2$-norm of Euclidean vectors. Further, 
$d_{TV}(M,N)$ denotes the total variation norm between two probability measures $M$ and $N$, and is defined as $d_{TV}(M,N) = 1/2 \int_{\mathcal{X}}|M(dx)-N(dx)|$.

% \begin{table*}
%   \centering
%   \caption{Comparison of  Finite-Time Analysis of Various Actor-Critic Algorithms.  }\label{table:comparision_fta}
%   \begin{tabular}{|p{3.7 cm}|p{3 cm}|p{2cm}|p{1.5 cm}|p{3 cm}|p{2 cm}|}
%   \toprule
%   \textbf{Reference} & \textbf{Algorithm} & \textbf{Sample Complexity} & \textbf{Function Approximation} & \textbf{Timescales} & \textbf{Critic estimation}\\ \hline
%   \cite{fta_2_timescale} & Actor-Critic & $\widetilde{\mathcal{O}}(\epsilon^{-2.5})$ & \checkmark & two-timescale & TD(0) \\ \hline
%   \cite{chen_zhao} & Actor-Critic & $\widetilde{\mathcal{O}}(\epsilon^{-2})$ & \checkmark & single timescale & TD(0) \\
%     \hline
%     \cite{online_prime_dual} & Constrained NAC & $\mathcal{O}(\epsilon^{-6})$ & $\times$ & two time-scale & TD(0)\\ 
%     \hline
%     \cite{suttleetal} & Actor-Critic & $\widetilde{\mathcal{O}}(\tau^{2}_{mix}\epsilon^{-2})$ & \checkmark & two-timescale & Monte-Carlo\\
%     \hline
%     \rowcolor{blue!10} Our work & Constrained AC & $\widetilde{\mathcal{O}}(\epsilon^{-2.5})$ & \checkmark & three-timescale & TD(0)\\ \hline
%     \rowcolor{blue!10} Our work & Constrained NAC & $\widetilde{\mathcal{O}}(\epsilon^{-2.5})$ & \checkmark & three-timescale & TD(0)\\
%   \bottomrule
%   \end{tabular}
% \end{table*}

\section{Related Work}
The Actor-Critic Algorithm was first analyzed theoretically for its asymptotic convergence by \cite{kondaborkar}. This was however for the case of lookup table representations. In the case when function approximation is used, \cite{first_actor_critic} analyzed the actor critic algorithm for its asymptotic convergence. \cite{kakade_2001} proposed the natural gradient based algorithm. Under various settings, the asymptotic convergence of actor critic algorithms has also been studied in (\cite{kakade_2001}, \cite{castro_2010}, \cite{zhang_2020}). In \cite{nac}, natural actor critic algorithms that bootstrap in both the actor and the critic have been proposed and analyzed for their asymptotic convergence.  

In recent times, there have been many works focusing primarily on performing finite time analyses of reinforcement learning algorithms. Such analyses are important as they provide sample complexity estimates and non-asymptotic convergence bounds for these algorithms. More recently, such analyses for actor critic algorithms have also been carried out though in the unconstrained (regular MDP) setting. 
\cite{ding} obtain finite time bounds for a natural policy gradient algorithm for discounted cost MDP with constraints.
\cite{fta_2_timescale} show  a non-asymptotic analysis of a two time-scale actor-critic algorithm assuming non-i.i.d samples and obtain a sample complexity of $\tilde{\mathcal{O}}(\epsilon^{-2.5})$ for convergence to an  $\epsilon$-approximate stationary point of the performance function.  \cite{multi_agent} consider a fully decentralized multi-agent reinforcement learning (MARL) setting and show a finite-time convergence analysis for the actor-critic algorithm in the average reward MDP scenario. \cite{chen_zhao} carried out finite time analysis of single time-scale actor critic algorithm and obtained a sample complexity of $\tilde{\mathcal{O}}(\epsilon^{-2})$ for convergence to an  $\epsilon$-approximate stationary point of the performance function. \cite{suttleetal} studied the non-asymptotic convergence properties of  Multi-level
Monte Carlo Actor-Critic (MAC) algorithm. \cite{mondal_aggarwal} proposed and studied the convergence properties of Accelerated
Natural Policy Gradient (ANPG) algorithm. There have also been other recent works  that have analysed Natural Actor-Critic Algorithms for their non-asymptotic convergence, see for instance, \cite{nac_1}, \cite{nac_2}, \cite{nac_3}, \cite{nac_4}, \cite{nac_5}. Table \ref{table:comparision_fta} summarises a comparison of our work with a few related works in terms of sample complexity. %(for a definition of $\tau_{mix}$ please refer \cite{suttleetal}).

In early work, \cite{Borkar} proposed the first actor-critic algorithm  for constrained Markov decision processes in the long-run average cost setting and proved its asymptotic convergence in the lookup table setting. \cite{bhatnagar_2010} presented the first actor-critic algorithm with function approximation for the infinite horizon discounted cost problem under multiple inequality constraints and proved its asymptotic convergence. \cite{cmdpsa} presented an actor-critic algorithm in the constrained long-run average cost MDP setting with function approximation under policy gradient actor and temporal difference critic and also analysed its asymptotic convergence. 

\cite{online_prime_dual} recently showed a finite-time analysis for the non-asymptotic convergence of the natural actor-critic algorithm to the global optimum of a CMDP problem in the case of lookup table representations and the infinite horizon discounted cost setting. In this paper, we consider the long-run average cost problem with function approximation that has been analysed for its asymptotic convergence in (\cite{cmdpsa}). We use TD(0) for the critic recursion and use projection for the critic. Further, we present the C-NAC algorithm, where we use natural gradients in the actor recursion along with a TD(0) critic. As  mentioned earlier, there are no prior finite time analyses for average cost/reward constrained actor critic algorithms with function approximation, so our work plugs in an important gap that previously existed in this direction.

\section{Preliminaries}
In this section, we present the C-MDP framework and algorithms that we analyze.

\subsection{Constrained Markov Decision Processes} 
We consider a discrete-time Markov Decision Process with finite state and action spaces. We first explain below the notations used.

$\bullet$  $S$ represents the state space and $A$ the action space. Further, we let $A(j)\subset A$ denote the set of feasible actions in state $j \in S$.

$\bullet$ Let $p(s,s',a)$ denote the probability of transition from state $s$ to $s'$ under action $a$.

$\bullet$ We shall consider only randomized policies $\pi$ in this work. Further, policies are assumed parameterized via a parameter $\theta\in {\cal R}^d$. Thus, given $\theta$, $\pi_{\theta}(a|s)$ is the probability of selecting an action $a\in A(s)$ in state $s$.

$\bullet$ The stationary distribution (over states) induced by the policy $\pi_{\theta}$ is denoted $\mu_{\pi_\theta}$ or simply $\mu_\theta$ (by an abuse of notation) and is assumed unique for any $\theta$.

%We consider $\vert A \vert < \infty $ and $\vert S\vert < \infty$. 
Let $q(n),h_1(n),...,h_N(n),n \geq 0$, denote a set of costs obtained upon transitioning from state $s_n$ to state $s_{n+1}$ under action  $a_n\in A(s_n)$.  At any time instant $n$, the single-stage costs $q(n),h_k(n), k = 1,...,N$, do not depend on prior states and actions $s_m,a_m, m<n$ given the current state-action pair ($s_n,a_n$). For any $i\in S$, $a\in A(i)$, let $d(i,a), h_k(i,a)$ be defined as $d(i,a) =
E[q(n) | s_n = i,a_n = a], h_k(i,a) = E[h_k(n) | s_n = i,a_n = a], k = 1,...,N$, respectively. Note the abuse of notation here. We assume that the single-stage costs are real-valued, non-negative and mutually independent. Further, we assume that all the single-stage costs $q(n),h_1(n),\ldots,h_N(n)$ are absolutely bounded by a constant $U_{c}>0$.

\subsection{The Objective and Lagrange Relaxation}

Our aim here is to minimize $J(\pi)$ where
\begin{align}
    J(\pi) &= \lim_{n\rightarrow \infty}\frac{1}{n}\EE\big[\sum\limits_{m=0}^{n-1}q(m)|\pi\big]\notag\\
    \label{eq1}
    &=\sum\limits_{s \in S}\mu_\pi(s)\sum\limits_{a \in A(s)}\pi(s,a) d(s,a),
\end{align}
subject to the constraints
\begin{align}
    G_k(\pi) &= \lim_{n\rightarrow\infty}\frac{1}{n}\EE\big[\sum\limits_{m=0}^{n-1}h_k(m)|\pi\big]\notag\\
    \label{eq2}
    &=\sum\limits_{s \in S}\mu_\pi(s)\sum\limits_{a \in A(s)}\pi(s,a) h_k(s,a) \le \alpha_k,
\end{align}
$k = 1,\ldots,N$, where $\alpha_1,\ldots,\alpha_N$ are certain prescribed (positive) constant thresholds.
We consider here  that the Markov process $\{s_
n\}$ under any given policy  is ergodic. Hence, the limits in (\ref{eq1})-(\ref{eq2}) are well-defined.

Consider a vector  $\gamma = (\gamma_1,\ldots,\gamma_N )^T$ representing a set of Lagrange multipliers with $\gamma_1,\ldots,\gamma_N  \in R^{+} \cup \{0\}$.  We define the  Lagrangian  $L(\pi, \gamma)$ according to 
\begin{align*}
   & L(\pi, \gamma) \notag\\& = J(\pi) + \sum\limits_{k=1}^{N}\gamma_k(G_k(\pi) - \alpha_k)\\
    &= \sum\limits_{s \in S}\mu_\pi(s)\sum\limits_{a \in A(s)}\pi(s,a) (d(s,a) + \sum\limits_{k=1}^{N}\gamma_k(h_k(s,a) -\alpha_k)).
\end{align*}
We now have the unconstrained MDP problem with single-stage cost being $q(t) + \sum\limits_{k=1}^{N}\gamma_k(h_k(t) - \alpha_k)$ at instant $t$. 
\begin{comment}
From the Bellman equation for optimality, we have
\[
    \beta^{*,\overline{\gamma}} + V^{*,\overline{\gamma}}(s) 
=
    \min_{a}\bigg( d(s,a) + \sum\limits_{k=1}^{N}\gamma_k(h_k(s,a)-\alpha_k) \]
    \[
    \qquad+ \sum\limits_{j \in S}p(s,j,a)V^{*,\overline{\gamma}}(j) \bigg),
\]
$\forall s \in S$, where $\beta^{*,\overline{\gamma}} := \min_{\pi}L(\pi,\overline{\gamma})$ and $V^{*,\overline{\gamma}}(\cdot)$ denotes the corresponding differential value function.
\end{comment}
The differential action value function in the relaxed control setting is defined as follows: 
\begin{align*}
   &M^{\pi,\gamma}(s,a)\\ &= \sum\limits_{t=1}^{\infty}\EE\bigg[q(t) + \sum\limits_{i=1}^{N}\gamma_i(h_i(t) - \alpha_i)\\
   &- \bigg( J(\btheta) + \sum\limits_{i=1}^{N}\gamma_i(G_i(\btheta) - \alpha_i) \bigg) | s_0 = s,a_0 = a,\pi\bigg]\\
   &=\sum\limits_{t=1}^{\infty}\EE\bigg[q(t) + \sum\limits_{i=1}^{N}\gamma_i h_i(t)\\
   &- \bigg( J(\btheta) + \sum\limits_{i=1}^{N}\gamma_i G_i(\btheta) \bigg) | s_0 = s,a_0 = a,\pi\bigg].
\end{align*}
 As mentioned by \cite{cmdpsa}, in the constraint scenario, the policy gradient of the Lagrangian would correspond to 
\begin{equation}
\label{pgt}
\nabla_{\btheta}L(\btheta,\gamma) = \sum\limits_{s \in S}\mu_\pi(s)\sum\limits_{a \in A(s)}\nabla\pi(a|s)\textit{A}^{\pi,\gamma}(s,a),
\end{equation}
where $\textit{A}^{\pi,\gamma}(s,a) = M^{\pi,\gamma}(s,a) - V^{\pi,\gamma}(s)$ is the advantage function for the relaxed setting. Here $V^{\pi,\gamma}(s)$ is the differential cost for a given policy $\pi$ and a set of Lagrange parameters $\gamma$.
We use linear function approximation for $M^{\pi,\gamma}(s,a)$. Thus, the same is approximated as follows:
\[\hat{M}_{w}^{\pi,\gamma}(s,a) \approx w^{\pi,\gamma^T}\Psi_{sa},\]
where $w^{\pi,\gamma} \in \RR^d$ are suitable parameters and 
$\Psi_{sa} \in \RR^d$ are  compatible features for the  tuples $(s,a)$. Thus, $\Psi_{sa} = \nabla \log \pi(a|s)$, $\forall s\in S, a\in A(s)$.

We also use linear function approximation for the differential value function $V^{\pi,\gamma}(s)$ as follows: We let 
\begin{align*}
    \hat{V}_{v}^{\pi,\gamma}(s) \approx v^{\pi,\gamma^T}f_{s},
\end{align*}
where $f_s$ is a $d_1$-dimensional feature vector $f_s = (f_s(1),f_s(2),.....,f_s(d_1))^T$ associated with state $s$ and $v^{\pi,\gamma} = (v^{\pi,\gamma}(1),v^{\pi,\gamma}(2),....,v^{\pi,\gamma}(d_1))^T$ is the corresponding weight vector.

\subsection{The Constrained Actor-Critic (C-AC) and Constrained Natural Actor-Critic (C-NAC) Algorithms}

\begin{algorithm}[H]
\caption{The Three-Timescale Actor-Critic Algorithm for Constrained MDP}\label{subsec:alg1}
\begin{algorithmic}[1]

\STATE \textbf{Input} $\theta_{0}$,  $v_{0}$, 
$L_{0}$, $U_k(0)$ for $1 \leq k \leq N$, $\gamma_k(0)$  for $1 \leq k \leq N$, step-size $a(n)$ for critic and average cost estimate, $b(n)$ for actor and $c(n)$ for Lagrange parameter.
\STATE Draw $s_0$ from some initial distribution
\FOR{$n > 0 $ and $k = 1,2,\ldots,N$}
    \STATE Sample $a_n \sim \pi_{\theta_n}(\cdot |s_n)$, $s_{n+1}\sim p(s_n, \cdot, a_n)$ 
    \STATE Observe  the costs $q(n),h_1(n),h_2(n),\ldots,h_N(n)$
    \STATE  $L_{n+1} = L_n + a(n)(q(n) + \sum_{k=1}^{N}\gamma_k(n)(h_k(n)-\alpha_k) - L_n)$
    \STATE $\delta_{n} = q(n) + \sum_{k=1}^{N}\gamma_k(n)(h_k(n)-\alpha_k) - L_n + v_{n}^T(f_{s_{n+1}} - f_{s_n})$
    \STATE $v_{n+1} = \Gamma(v_{n} + a(n)\delta_{n}f_{s_n})$
    \STATE $\theta_{n+1} = \theta_{n} + b(n)\delta_{n}\Psi_{s_{n}a_{n}}$
    \STATE $U_{k}(n+1) = U_{k}(n) + a(n)(h_{k}(n) - U_{k}(n))$
    \STATE $\gamma_{k}(n+1) = \hat{\Gamma}(\gamma_k(n) + c(n)(U_k(n) - \alpha_k))$
\ENDFOR
\end{algorithmic}
\end{algorithm}
We present here the two algorithms that we analyze in our work for their non-asymptotic convergence -- the constrained actor-critic algorithm (Algorithm 1) and the constrained natural actor-critic algorithm (Algorithm 2), respectively. At time instant $t$, we have $\theta_t$ as the actor parameter, $v_t$ as the critic parameter, $L_t$ as the average cost estimate, $U_{k}(t)$ as the average constraint cost estimate for $k = 1,2,\ldots,N$, $\gamma(t) = (\gamma_1(t),\gamma_2(t),\ldots,\gamma_N(t))^T$ as the vector of Lagrange multiplier estimates and $G(t)$ as the estimate of the Fisher information matrix  respectively.

Let $\Gamma : \RR^{d_1} \rightarrow C$  project any point in $\RR^{d_1}$ to the closest point within the set $C$ which is assumed compact and convex. For any point $h$ contained in the set $C$, 
 $\Vert h\Vert \leq U_{v}$ where $U_{v} > 0$ is a constant. Further, $\hat{\Gamma} : \RR \rightarrow [0,M]$ indicates the operation $\hat{\Gamma}(y) = \max(0,\min(y,M))$ for any $y \in \RR$ where $M < \infty$ represents a large positive constant. This projection operator guarantees the Lagrange multiplier to stay both non-negative and bounded. 

For the natural actor-critic algorithm, we take $G(0) = pI$, where $I$ is a $d \times d$-identity matrix and $p > 0$ is a constant. It can be concluded  that $G(n),n \geq 1$ are positive definite and symmetric matrices as from the update rule, we can see that these result from addition of $(1-a(n))G(n-1)$ and $a(n)\Psi_{s_{n}a_{n}}\Psi_{s_{n}a_{n}}^T$. Hence, $G(n)^{-1},n \geq 1$ are positive definite and symmetric matrices as well.
Let the smallest eigenvalue of $G(i)^{-1}$ be $\lambda_{i}>0$, where $i \geq 1$. Let $\lambda$ be the minimum of all such eigenvalues, i.e., $\lambda = \min\limits_{i}{\lambda_i} >0$.

\begin{algorithm}[H]
\caption{The Three-Timescale Natural Actor-Critic Algorithm for Constrained MDP}\label{subsec:alg2}
\begin{algorithmic}[1]

\STATE \textbf{Input} $\theta_{0}$, $v_{0}$, $L_{0}$, $U_k(0)$ for $1 \leq k \leq N$, $\gamma_k(0)$  for $1 \leq k \leq N$, $G(0)$, step-size $a(n)$ for critic and average cost estimate, $b(n)$ for actor and  $c(n)$ for Lagrange parameter.
\STATE Draw $s_0$ from some initial distribution
\FOR{$n > 0 $ and $k = 1,2,\ldots,N$}
    \STATE Sample $a_n \sim \pi_{\theta_n}(\cdot |s_n)$, $s_{n+1}\sim p(s_n,\cdot,a_n)$ 
    \STATE Observe  the costs $q(n),h_1(n),h_2(n),\ldots,h_N(n)$
    \STATE  $L_{n+1} = L_n + a(n)(q(n) + \sum_{k=1}^{N}\gamma_k(n)(h_k(n)-\alpha_k) - L_n)$
    \STATE $\delta_{n} = q(n) + \sum_{k=1}^{N}\gamma_k(n)(h_k(n)-\alpha_k) - L_n + v_{n}^T(f_{s_{n+1}} - f_{s_n})$
    \STATE $v_{n+1} = \Gamma(v_{n} + a(n)\delta_{n}f_{s_n})$
    \STATE $\theta_{n+1} = \theta_{n} + b(n)\delta_{n}G(n)^{-1}\Psi_{s_{n}a_{n}}$
    \STATE $U_{k}(n+1) = U_{k}(n) + a(n)(h_{k}(n) - U_{k}(n))$
    \STATE $\gamma_{k}(n+1) = \hat{\Gamma}(\gamma_k(n) + c(n)(U_k(n) - \alpha_k))$
    \STATE $G(n+1) = (1-a(n))G(n) + a(n)\Psi_{s_{n}a_{n}}\Psi_{s_{n}a_{n}}^T$
\ENDFOR
\end{algorithmic}
\end{algorithm}

\section{Finite-Time Convergence Results}

We provide in this section the main theoretical results for non-asymptotic convergence as well as provide the  convergence rate and sample complexity for the two algorithms. Due to lack of space, we provide the detailed proofs of these results in the appendix. We emphasize here that asymptotic convergence analysis of these algorithms has not been analysed here since for the case of Constrained Actor-Critic, it has been analysed in \cite{cmdpsa}. Further, the same for Constrained Natural Actor-Critic, it will carry through in a similar manner using the results of \cite{nac}.

\subsection{Assumptions and Basic Results}\label{assumandprop}

We consider TD(0) with function approximation for the critic recursion that estimates the state-value function. Let $v^{*}(\theta,\gamma)$ be the convergence point of the critic under the behavior policy $\pi_{\btheta}$ (for given actor and Lagrange parameters $\theta$ and $\gamma$ respectively), and define $\Ab$ and $\bbb$ as follows:
\begin{align*}
    \Ab &:= \EE_{s_n,a_n,s_{n+1}} \big[ f_{s_{n}} \big( f_{s_{n+1}} - f_{s_{n}}\big)^{\top} \big], \\
    \bbb &:= \EE_{s_n,a_n,s_{n+1}} [(C(s_n,a_n,\gamma)- L(\btheta,\gamma))f_{s_{n}} ],
\end{align*}
where $s_n \sim \mu_{\btheta}(\cdot), a_n \sim \pi_{\btheta}(\cdot | s), s_{n+1}\sim p(s_n, \cdot, a_n)$ and $C(s_n,a_n,\gamma) = d(s_n,a_n) + \sum\limits_{k=1}^{N}\gamma_k(h_k(s_n,a_n) - \alpha_k)$ corresponds to the single-stage cost for the relaxed problem.
Analogous to the unconstrained setting, it can be seen that (see \cite{cmdpsa})
\begin{align*}
    \Ab v^{*}(\theta,\gamma) + \bbb
    & = 
    \mathbf{0}.
\end{align*}

\begin{assumption} \label{assum:bounded_feature_norm}
    The norm of each state feature is bounded by 1,  i.e., $\Vert f_{i}\Vert \le 1$.
\end{assumption}

The following assumption is required for the existence and uniqueness of $v^{*}(\theta,\gamma)$ .

\begin{assumption} \label{assum:negative-definite}
    The matrix $\Ab$ (defined above) is negative definite with  maximum eigenvalue as $- \lambda_e <0$ for all values of $\btheta$.
    % \begin{align*}
    %     \Ab \preceq - \lambda \Ib.
    % \end{align*}
\end{assumption}

The approximation error for the feature mapping can vary depending on its complexity.
We define the approximation error that arises due to linear function approximation as follows.
\begin{align*}
    \epsilon_{\text{app}}(\btheta,\bgamma) := 
    \sqrt{
    \EE_{s \sim \mu_{\btheta}} \big( f_s^{\top} v^{*}(\theta,\bgamma) - V^{\pi_{\btheta},\bgamma}(s) \big)^2
    }.
\end{align*} 

\begin{assumption}\label{epsilon_bound}
    \begin{align*}
     \forall \btheta ,\forall \bgamma,  \mbox{ } \epsilon_{\text{app}}(\btheta ,\bgamma) \le \epsilon_{\text{app}},
\end{align*}
where  $\epsilon_{\text{app}}\geq0$ is some constant.
\end{assumption}

Assumption \ref{epsilon_bound} is useful in finding upper bounds of some of the error terms.

\begin{assumption}[Uniform ergodicity] \label{assum:ergodicity}
     For a given  $\btheta$,  we consider the policy $\pi_{\btheta}(\cdot|s)$ and the transition probability measure $p(s,\cdot,a)$ that induce a stationary distribution $\mu_{\btheta}(\cdot)$. There exists $b > 0$ and $k \in (0,1)$ for the Markov chain where $a_t \sim \pi_{\btheta}(\cdot | s_t), s_{t+1} \sim p(s_t,\cdot,a_t)$ such that
    \begin{align*}
        d_{TV}\big(p^\tau(x,y,\cdot), \mu_{\btheta}(y)\big) \le b k^{\tau}, \forall \tau \ge 0, \forall x,y \in \cS.
    \end{align*}
\end{assumption}
%With this assumption, we can remove the unrealistic assumption that each tuple is drawn from the stationary-state distribution, at the cost of only logarithm factors in the convergence rate.  
Assumption \ref{assum:ergodicity} is needed to tackle the issue of Markov sampling in TD learning. It has been used in analyses of TD learning, for instance, in \citet{bhandari}. Refer to \cite{meyntweedie} for various results related to uniform ergodicity as well as other notions of ergodicity of Markov chains.

\begin{assumption} \label{assum:policy-lipschitz-bounded}
 There exist constants $L,D$, $M_{m}$ such that $\forall \btheta_1,\btheta_2 ,\btheta \in \RR^d$, we have
\begin{enumerate}
\item[(a)] $\big\|\nabla \log \pi_{\btheta}(a|i) \big\| \le D$, $\forall i,\forall a$,
\item[(b)] $\big\|\nabla \log \pi_{\btheta_1}(a|i) - \nabla \log \pi_{\btheta_2}(a|i) \big\| \le M_{m} \Vert\btheta_1 - \btheta_2\Vert$, $\forall i,\forall a$, 
\item[(c)] There exist scalars $\check{K}, \hat{K}>0$ such that for any $x\not=0$ and all $s_n,a_n$,
\[
\check{K}\| x\|^2 \leq x^T \Psi_{s_na_n}\Psi_{s_na_n}^T x \leq \hat{K}\|x\|^2.
\]
\end{enumerate}
\end{assumption}
\begin{remark}
\label{rem1}
As a consequence of Assumption~\ref{assum:policy-lipschitz-bounded}(a), it follows that
$\big|\pi_{\btheta_1}(a|i) - \pi_{\btheta_2}(a|i) \big| \le L \Vert \btheta_1 - \btheta_2\Vert$, $\forall i,\forall a$. In other words, the policy for any given $(i,a)$ tuple is Lipschitz continuous in $\theta$.
\end{remark}

 Assumption \ref{assum:policy-lipschitz-bounded} provides smoothness of the parameterized policies and can be seen to be verified by many policies. This assumption is useful for finding upper bounds for some of the error terms while proving the convergence of actor and critic recursions. 

\begin{proposition}\label{bound_G}
    The updates $G(t)$ satisfy $\sup\limits_{t}\Vert G(t) \Vert < \infty$ and $\sup\limits_{t}\Vert G(t)^{-1} \Vert$ $< \infty$, respectively.
\end{proposition}

\noindent {\em Proof:} 
Since $a(n)\rightarrow 0$ as $n\rightarrow\infty$, $\exists N_0\geq 1$ such that for all $n\geq N_0$, $G(n+1)$ is a convex combination of $G(n)$ and $\Psi_{s_na_n}\Psi^T_{s_na_n}$, with $G_0 = pI$, $p>0$, see step 12 of 
 Algorithm \ref{subsec:alg2}. Without loss of generality, assume that $a(n)\leq 1$, $\forall n$. Thus, observe that for $n=0$,
 \[
 (1-a(0))p\|x\|^2 + a(0) \check{K}\|x\|^2 \leq
 x^T G(1) x^T\]
 \[\leq (1-a(0))p\|x\|^2 + a(0) \hat{K}\|x\|^2. \]
Letting $\check{M} = \min(p,\check{K})$ and $\hat{M} = \max(p,\hat{K})$, it can be verified from induction 
 that
\[
\check{M}\|x\|^2 \leq x^T G(n) x \leq \hat{M} \|x\|^2,
\]
uniformly over $n\geq 0$. The claim now follows from
arguments on page 35 of \cite{bertn}.

\begin{proposition} \label{prop:optimal-lipschitz1}
     There exists a constant $L_{1}>0$    such that $\forall \bgamma \in \RR^N $ with $\bgamma = (\bgamma_1,\ldots,\bgamma_N )^T$ and $0 \leq \bgamma_j \leq M$ for $j=1,2,\ldots,N$,
\begin{align*}
    \big \|
    v^{*}(\btheta_{1},\bgamma) - v^{*}(\btheta_{2},\bgamma)
    \big\|
    \le L_{1} \Vert\btheta_{1} - \btheta_{2}\Vert, \forall \btheta_{1}, \btheta_{2}\in\RR^d.
\end{align*}
\end{proposition}

\textit{Proof Sketch}

Let
\begin{align*}
    \Ab_{\btheta} &:= \EE_{s_n,a_n,s_{n+1}} \big[ f_{s_{n}} \big( f_{s_{n+1}} - f_{s_{n}}\big)^{\top} \big], \\
    \bbb_{\btheta,\bgamma} &:= \EE_{s_n,a_n,s_{n+1}} [(C(s_n,a_n,\gamma)- L(\btheta,\gamma))f_{s_{n}} ],
\end{align*}
where $s_n \sim \mu_{\btheta}(\cdot), a_n \sim \pi_{\btheta}(\cdot | s), s_{n+1}\sim p(s_n, \cdot, a_n)$ .

We have,
\begin{align*}
    \Ab_{\theta} v^{*}(\theta,\gamma) + \bbb_{\theta,\gamma}
    & = 
    \mathbf{0}.
\end{align*}

Thus, 
\begin{align*}
    &\big \|v^{*}(\theta_{1},\bgamma) - v^{*}(\theta_{2},\bgamma)
    \big\| \\
    =&
    \Vert \Ab_{\theta_1}^{-1}\bbb_{\theta_1,\gamma} - \Ab_{\theta_2}^{-1}\bbb_{\theta_2,\gamma}   \Vert \\
    \leq & 
    \Vert \Ab_{\theta_1}^{-1}\bbb_{\theta_1,\gamma} - \Ab_{\theta_2}^{-1}\bbb_{\theta_1,\gamma}   \Vert  + \Vert \Ab_{\theta_2}^{-1}\bbb_{\theta_1,\gamma} - \Ab_{\theta_2}^{-1}\bbb_{\theta_2,\gamma}   \Vert\\
     \leq & 
     \Vert \Ab_{\theta_1}^{-1} - \Ab_{\theta_2}^{-1} \Vert\Vert \bbb_{\theta_1,\gamma} \Vert + \Vert \Ab_{\theta_2}^{-1} \Vert\Vert \bbb_{\theta_1,\gamma} - \bbb_{\theta_2,\gamma} \Vert.
\end{align*}

It can be shown that
\begin{align*}
    \Vert \bbb_{\theta_1,\gamma} \Vert &\leq 2U_r,\\
    \Vert \Ab_{\theta_2}^{-1} \Vert &\leq \lambda_{e}^{-1},\\
    \Ab_{\theta_1}^{-1} - \Ab_{\theta_2}^{-1} &= \Ab_{\theta_1}^{-1}(\Ab_{\theta_2} - \Ab_{\theta_1})\Ab_{\theta_2}^{-1}.
\end{align*}

We have from section B.2 of \cite{fta_2_timescale} the following:
\begin{align*}
    \Vert \Ab_{\theta_1} - \Ab_{\theta_2} \Vert 
    &\le 
    4|A|L \bigg(1 + \lceil \log_{k}b^{-1} \rceil + \frac{1}{1-k} \bigg) \norm{\btheta_1 - \btheta_2},\\
    \Vert \bbb_{\theta_1,\gamma} - \bbb_{\theta_2,\gamma} \Vert &\leq 6|A|U_rL \bigg(1 + \lceil \log_{k}b^{-1} \rceil + \frac{1}{1-k} \bigg) \norm{\btheta_1 - \btheta_2}.
\end{align*}

After combining all the terms, we have
\begin{align*}
    \big \|
    v^{*}(\btheta_{1},\bgamma) - v^{*}(\btheta_{2},\bgamma)
    \big\|
    \le L_{1} \norm{\btheta_{1} - \btheta_{2}}, \forall \btheta_{1}, \btheta_{2}\in\RR^d, 
\end{align*}
where $L_1 = (8\lambda_{e}^{-2} + 6\lambda_{e}^{-1})U_{r}|A|L \bigg(1 + \lceil \log_{k}b^{-1} \rceil + \frac{1}{1-k} \bigg) $.

\begin{proposition}\label{prop:optimal-lipschitz2}
 Let $\bgamma^{1} = (\bgamma^1_1,\ldots,\bgamma^1_N )^T$ and $\bgamma^{2} = (\bgamma^2_1,\ldots,\bgamma^2_N )^T$ be any two vectors in $\RR^N$ with $0 \leq \bgamma^i_j \leq M$ for $i=1,2$ and $j=1,2,\ldots,N$. There exists a constant $L_{2}>0$ such that
\begin{align*}
    \big \|
    v^{*}(\btheta,\bgamma^1) - v^{*}(\btheta,\bgamma^2)
    \big\|
    \le L_{2}\vert \bgamma_p^1 - \bgamma_p^2\vert , \forall \btheta \in\RR^d,
\end{align*}
where $\vert \bgamma_p^1 - \bgamma_p^2\vert = \max\limits_{i = 1,2,\ldots,N}\vert \bgamma_i^1 - \bgamma_i^2\vert$. 
\end{proposition}

\textit{Proof Sketch}

We have,
\begin{align*}
    \big \|
    v^{*}(\theta,\bgamma^1) - v^{*}(\theta,\bgamma^2)
    \big\| &=  \Vert \Ab_{\theta}^{-1}\bbb_{\theta,\bgamma^1} - \Ab_{\theta}^{-1}\bbb_{\theta,\bgamma^2}   \Vert \\
    &\leq \Vert \Ab_{\theta}^{-1} \Vert\underbrace{\Vert \bbb_{\theta,\bgamma^1} - \bbb_{\theta,\bgamma^2}\Vert}_{I_1}.
\end{align*}
Now for the term $I_1$, note that 
\begin{align*}
    &\Vert \bbb_{\theta,\bgamma^1} - \bbb_{\theta,\bgamma^2}\Vert\\
    &= \Vert E_{\{s_n \sim \mu_{\btheta}(\cdot), a_n \sim \pi_{\btheta}(\cdot | s), s_{n+1}\sim p(s_n, \cdot, a_n)\}}[(C(s_n,a_n,\bgamma^1)\\
    &\qquad- C(s_n,a_n,\bgamma^2) +  L(\btheta,\bgamma^2) - L(\btheta,\bgamma^1))f_{s_n}]\Vert\\
    &\leq 2N(U_c + U_\alpha)\vert \bgamma_p^1 - \bgamma_p^2\vert,
\end{align*}
%\hfill\break
where, $\vert \bgamma_p^1 - \bgamma_p^2\vert = \max\limits_{i = 1,2,...,N}\vert \bgamma_i^1 - \bgamma_i^2\vert$.
Hence,
\begin{align*}
    \big \|
    v^{*}(\btheta,\bgamma^1) - v^{*}(\btheta,\bgamma^2)
    \big\| \leq L_{2}\vert \bgamma_p^1 - \bgamma_p^2\vert, \forall \bgamma_{1}, \bgamma_{2}\in\RR^N, 
\end{align*}
where, $L_2 = (2N(U_c + U_\alpha))/\lambda_e$.

Let $\tau_{t}$ denote the mixing time of our ergodic Markov chain. So we have 
\begin{align}\label{eq:def_mixing_time}
    \tau_t & := 
    \min 
    \big \{
   m \ge 0 |
    bk^{m-1} \le
    \min \{ a(t), b(t),c(t) \}
    \big \},
\end{align}
where $b,k$ are defined as in Assumption \ref{assum:ergodicity}.  

We now present the  result of non-asymptotic analysis of constrained actor-critic methods. We consider $a(t) = c_a(1+t)^{-\omega}$, $b(t) = c_{b}(1+t)^{-\sigma}$ and $c(t) = c_{c}(1+t)^{-\beta}$, where $0 < \omega< \sigma <\beta \leq1$, %$\beta > \sigma $, 
with $c_{a}$, $c_b$ and $c_{c}$  being positive constants.

\subsection{Finite-Time Convergence Results for Algorithm \ref{subsec:alg1}}
We provide here the non-asymptotic convergence results for both the actor and the critic recursions in Algorithm \ref{subsec:alg1}. We also present the convergence rate and sample complexity of the algorithm.

\subsubsection{Convergence of the actor recursion for Algorithm \ref{subsec:alg1}}

We have the  following result after carrying out the non-asymptotic analysis of the actor.

\begin{theorem} \label{thm:actor_1}
 At the $t$-th iteration we have,
%\begin{align*}
\[
    \min_{0\leq m \le t} \EE \big\|\nabla_{\btheta} L(\btheta_{m},\bgamma(m))\big\|^2 %\\
    %&\qquad
    =
    \cO(\epsilon_{\text{app}})
    +
    \cO\big(t^{ \sigma -\beta}\big)\]
\[   % &\qquad+
    +\cO\bigg(\frac{\log^2 t}{t^{\sigma}}\bigg)
    +
     \cO\bigg(\frac{\sum_{k=\tau_t}^{t}E\Vert A_k\Vert^2}{1 + t -\tau_t} \bigg)\]
\[    %&\qquad
+ \cO\bigg(\frac{\sum_{k=\tau_t}^{t}E\Vert B_k\Vert^2}{1 + t -\tau_t} \bigg),
\]
%\end{align*}
where 
\begin{align}\label{eq:def_average_critic1}
    A_k &= L_k - L(\btheta_k,\bgamma(k)),\\
    B_k &= v_{k} - v^{*}(\theta_k,\gamma(k)).
\end{align}
\end{theorem}
\begin{figure*}
    \centering
    \includegraphics[scale =0.2]{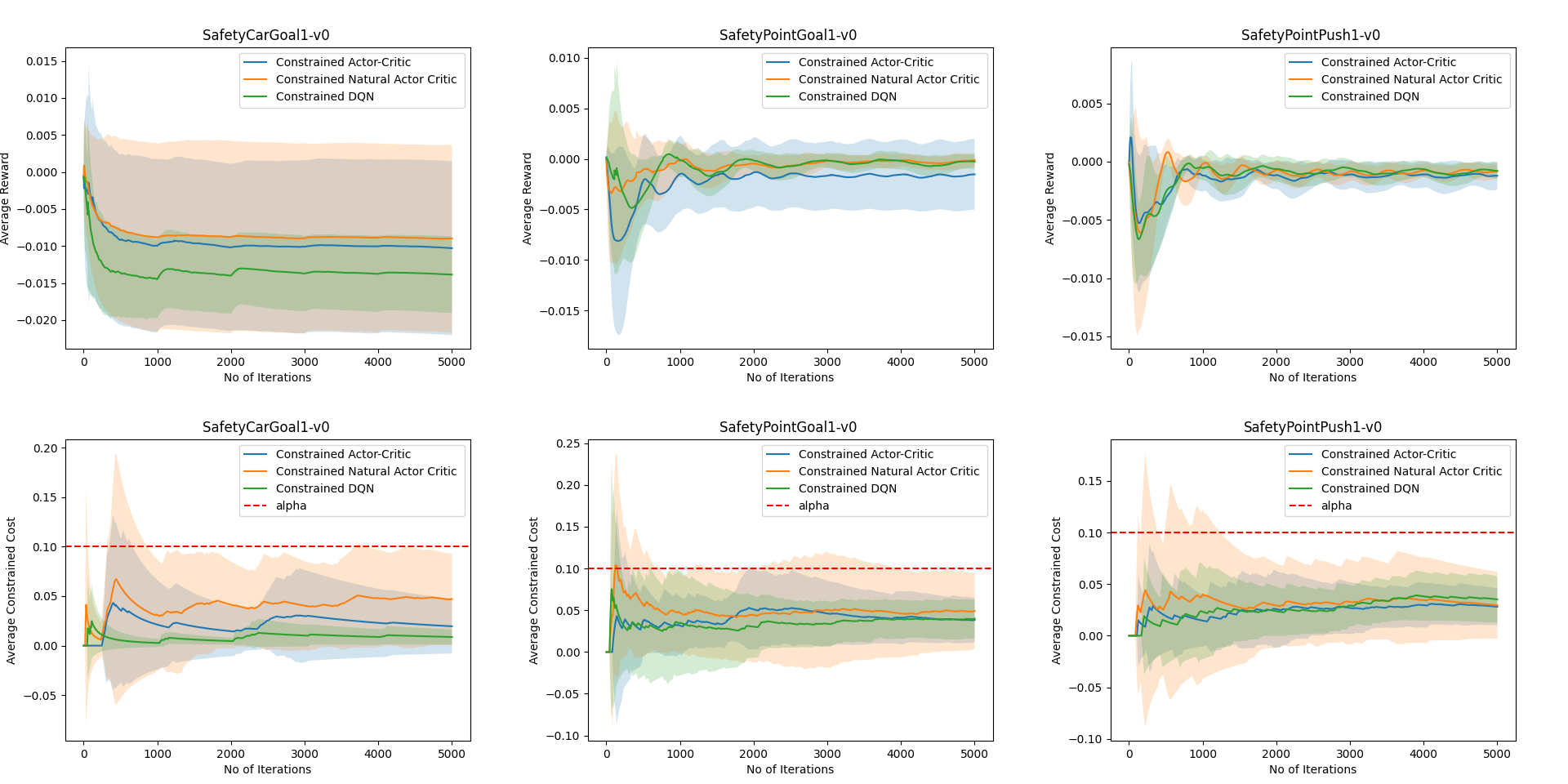}
    \caption{Comparison of C-AC , C-NAC and C-DQN: Plots in the top row are for the average reward performance while those in the bottom row are for the constraint costs for the three environments. These are plotted as functions of the number of iterations.}
    \label{fig:experiment}
\end{figure*}

\subsubsection{Convergence of the critic recursion for Algorithm \ref{subsec:alg1}}

For the critic recursion, we obtain the following result for the average estimation error. 

\begin{theorem} \label{thm:critic_1}
We have 
\begin{eqnarray} 
\nonumber
& \frac{1}{1+t-\tau_t} \sum_{k=\tau_t}^{t} \EE \norm{v_{k} - v^{*}(\theta_k,\gamma(k))}^2 \\
& = 
    \cO
    \bigg(\frac{1}{t^{1-\omega}}
    \bigg)
    +
    \cO
    \bigg(\frac{\log t}{t^{\omega}}
    \bigg)
    +
    \cO
    \bigg(\frac{1}{t^{2(\sigma - \omega)}}
    \bigg), \label{eq:critic_converge_a1}\\
    &
    \nonumber 
    \frac{1}{1+t-\tau_t} \sum_{k=\tau_t}^{t} \EE \big(L_k - L(\btheta_k,\bgamma(k))\big)^2 \\
    & = \cO
    \bigg(\frac{1}{t^{1-\omega}}
    \bigg)
    +
    \cO
    \bigg(\frac{\log t}{t^{\omega}}
    \bigg)
    +
    \cO
    \bigg(\frac{1}{t^{2(\sigma - \omega)}}
    \bigg). \label{eq:critic_converge_b1}
\end{eqnarray}
\end{theorem}

\subsubsection{Convergence rate and sample complexity for Algorithm \ref{subsec:alg1}}\label{sec:theory_complexity1}
We finally provide the convergence rate of the algorithm and characterize the sample complexity of the same in Corollary \ref{col:sample-complexity_1}. 
 
\begin{corollary} \label{col:sample-complexity_1}
We have,
%\begin{align*} 
    %&
    \[\min_{0 \le k \le t} \EE \norm{\nabla_{\btheta} L(\btheta_{k},\bgamma(k))}^2 
    %&\qquad
    =
    \cO(\epsilon_{\text{app}})
    +
    \cO\bigg(\frac{1}{t^{\beta- \sigma}}\bigg)\]
\[    +
    \cO
    \bigg(\frac{\log^2 t}{t^{\omega}}
    \bigg)
    +
    %&\qquad\qquad
    \cO
    \bigg(\frac{1}{t^{2(\sigma - \omega)}}
    \bigg).
    \]
%\end{align*}
If we set $\omega = 0.4, \sigma = 0.6 ,\beta =1 $, 
Algorithm \ref{subsec:alg1} needs $T =\tilde{\cO}(\epsilon^{-2.5})$ steps to  obtain the following:
\begin{align*}
    \min_{0 \le k \le T} \EE \big\|\nabla_{\btheta} L(\btheta_{k},\gamma(k))\big\|^2
    & \le 
    \cO(\epsilon_{\text{app}})
    +
    \epsilon.
\end{align*}
\end{corollary}

\begin{remark}
\label{remark1}
In Corollary \ref{col:sample-complexity_1}, the results of Theorems \ref{thm:actor_1} and \ref{thm:critic_1} are combined which results in the convergence rate of Algorithm \ref{subsec:alg1} as $\tilde{\cO}(t^{-0.4})$.  The sample complexity of the constrained actor-critic algorithm is $\tilde{\cO}(\epsilon^{-2.5})$ as we have exactly one sample per iteration.
\end{remark}

\subsection{Finite-Time Convergence Results for Algorithm \ref{subsec:alg2}}

As with Algorithm \ref{subsec:alg1}, we provide here the non-asymptotic convergence results for both the actor and the critic recursions in Algorithm \ref{subsec:alg2}. Further, we present the convergence rate and sample complexity of the algorithm. 

\subsubsection{Convergence of the Actor for Algorithm \ref{subsec:alg2}}
We have the  following result after carrying out non-asymptotic analysis of the actor.

\begin{theorem} \label{thm:actor_2}
 At the $t$th iteration, we have,
\begin{align*}
    &\min_{0\leq m \le t} \EE \big\|\nabla_{\btheta} L(\btheta_{m},\bgamma(m))\big\|^2 \\
    &\qquad=
    \cO(\epsilon_{\text{app}})
    +
    \cO\big(t^{\sigma -\beta}\big)\\
    &\qquad+
    \cO\bigg(\frac{\log^2 t}{t^{\omega}}\bigg)
    +
     \cO\bigg(\frac{\sum_{k=\tau_t}^{t}E\Vert A_k\Vert^2}{1 + t -\tau_t} \bigg) \\
     &\qquad+ \cO\bigg(\frac{\sum_{k=\tau_t}^{t}E\Vert B_k\Vert^2}{1 + t -\tau_t} \bigg),
\end{align*}
where 
\begin{align}\label{eq:def_average_critic2}
    A_k &= L_k - L(\btheta_k,\bgamma(k)),\\
    B_k &= v_{k} - v^{*}(\btheta_k,\bgamma(k)).
\end{align}
\end{theorem}
\begin{table*}
  \centering
  \caption{Comparision of C-AC , C-NAC and C-DQN in terms of average reward $\pm$ standard error upon convergence.}
  \begin{tabular}{|c|c|c|c|}
    \hline
  Algorithm &  SafetyPointGoal1-v0 & SafetyCarGoal1-v0 & SafetyPointPush1-v0\\ \hline 
  C-AC &  $-0.0015 \pm 0.0035$ & $-0.01 \pm 0.0117$ & $-0.0012 \pm 0.0011$\\ \hline
  C-NAC & {\bf $-0.00012 \pm 0.0006$} & {\bf $-0.009 \pm 0.0127$} & {\bf $-0.0008 \pm 0.0005$} \\ \hline
  C-DQN & $-0.0003 \pm 0.0007$ & $-0.014 \pm 0.005$ & $-0.0008 \pm 0.0007$\\ \hline
  \end{tabular}
  \label{tab:experiment_avgreward}
\end{table*}

\subsubsection{Convergence of the Critic for Algorithm \ref{subsec:alg2}}
We now provide a result analysing the average estimation error for the critic.

\begin{theorem} \label{thm:critic_2}
 We have the following:
\begin{align}
    \frac{1}{1+t-\tau_t} \sum_{k=\tau_t}^{t} \EE \norm{v_{k} - v^{*}(\btheta_k,\bgamma(k))}^2 \notag \\
    = 
    \cO
    \bigg(\frac{1}{t^{1-\omega}}
    \bigg)
    +
    \cO
    \bigg(\frac{\log t}{t^{\omega}}
    \bigg)
    +
    \cO
    \bigg(\frac{1}{t^{2(\sigma - \omega)}}
    \bigg), \label{eq:critic_converge_a2}\\
    \frac{1}{1+t-\tau_t} \sum_{k=\tau_t}^{t} \EE \big(L_k - L(\btheta_k,\bgamma(k))\big)^2 \notag\\
     = 
    \cO
    \bigg(\frac{1}{t^{1-\omega}}
    \bigg)
    +
    \cO
    \bigg(\frac{\log t}{t^{\omega}}
    \bigg)
    +
    \cO
    \bigg(\frac{1}{t^{2(\sigma - \omega)}}
    \bigg).
    \label{eq:critic_converge_b2}
\end{align}

\end{theorem}

\subsubsection{Convergence rate and sample complexity for Algorithm \ref{subsec:alg2}} \label{sec:theory_complexity2}
We finally provide the convergence rate of the algorithm and characterize the sample complexity of the same in Corollary \ref{col:sample-complexity_2}.

\begin{corollary} \label{col:sample-complexity_2}
We have
\[
    \min_{0 \le k \le t} \EE \norm{\nabla_{\btheta} L(\btheta_{k},\bgamma(k))}^2 
    =
    \cO(\epsilon_{\text{app}})
    +
    \cO\bigg(\frac{1}{t^{\beta- \sigma}}\bigg)\]
 \[   +
    \cO
    \bigg(\frac{\log^2 t}{t^{\omega}}
    \bigg)
    +
    \cO
    \bigg(\frac{1}{t^{2(\sigma - \omega)}}
    \bigg).\]
If we set $\omega = 0.4, \sigma = 0.6 ,\beta = 1$, 
Algorithm \ref{subsec:alg2} needs $T =\tilde{\cO}(\epsilon^{-2.5})$ steps to  obtain the following,
\begin{align*}
    \min_{0 \le k \le T} \EE \big\|\nabla_{\btheta} L(\btheta_{k},\gamma(k))\big\|^2
    & \le 
    \cO(\epsilon_{\text{app}})
    +
    \epsilon, 
\end{align*}

\end{corollary}

\begin{remark}
    \label{remark2}
Analogous to Remark~\ref{remark1},  %Corollary \ref{col:sample-complexity_1}, the results of Theorems \ref{thm:actor_1} and \ref{thm:critic_1} are combined and 
in Corollary \ref{col:sample-complexity_2}, the results of Theorems \ref{thm:actor_2} and \ref{thm:critic_2} are combined which gives the convergence rate of (natural actor critic) Algorithm \ref{subsec:alg2} as $\tilde{\cO}(t^{-0.4})$ and a sample complexity of $\tilde{\cO}(\epsilon^{-2.5})$ as we have one per-iteration sample in this algorithm as well. It is important to also mention that for the results of both algorithms,  $\cO(\cdot)$ hides the terms that do not depend on the iteration number.
\end{remark}

\subsection{Proof Sketch for Theorems \ref{thm:actor_1} and \ref{thm:actor_2}}

We provide here an overview of the manner in which the proofs of Theorems \ref{thm:actor_1} and  \ref{thm:actor_2} proceed. This also helps us to describe the connection between the various results mentioned above. The detailed arguments are nonetheless provided in the appendix. The proofs of Theorems \ref{thm:actor_1} and  \ref{thm:actor_2} rely crucially on Lemma~\ref{Lemma:L-smooth} below.

\begin{lemma} 
\label{Lemma:L-smooth1}
 For all $\bgamma \in R^N$ with $0 \leq \bgamma_{i} \leq M$ where $i \in \{1,2,\ldots,N\}$, there exists a constant $M_{L}$ greater than 0 such  that for all $\theta_1, \theta_2 \in {R^d}$, 
 \begin{align*}
      \Vert \nabla_{\theta} L(\theta_1,\bgamma) - \nabla_{\theta} L(\theta_2,\bgamma)\Vert \leq M_{L}\Vert\theta_1 - \theta_2\Vert.
 \end{align*}
\end{lemma}

As a result of this lemma, we have the following inequality (see  \cite{fta_2_timescale}, Lemma C.1) $\forall t > 0$, that we use in the proof of Theorem  \ref{thm:actor_1}. %and \ref{thm:actor_2}. %For the sake of completeness, the proof of Lemma~\ref{Lemma:L-smooth} has also been included in the appendix.
\[
    L(\theta_{t+1},\bgamma(t)) \geq L(\theta_t,\bgamma(t))\notag\]
    \[
    + b(t) \langle \nabla_{\theta} L(\theta_t,\bgamma(t)), \delta_{t}\nabla \log\pi_{\theta_t}(a_t|s_t) \rangle\]
    \begin{equation}
\label{inequality}
 - \frac{M_{L}}{2}b(t)^2\Vert\delta_{t}\nabla \log\pi_{\theta_t}(a_t|s_t)\Vert^2.
\end{equation}

 The key idea in the proof (see Appendix for details) is to split the middle term in the RHS of (\ref{inequality}) into a few terms, one of which is $b(t)\Vert \nabla L(\btheta_t,\bgamma(t))\Vert^2$. We obtain an upper bound for $\Vert \nabla L(\btheta_t,\bgamma(t))\Vert^2$. After summing the expectation of terms on both sides, we analyse each term in the bound to get the desired result for Theorem \ref{thm:actor_1}. Note also that the result for Theorem \ref{thm:actor_1} depends on the convergence of the critic parameter and the average cost estimator. So we find a bound on the averaged estimation errors by the critic and average cost estimator in Theorem \ref{thm:critic_1}. We then obtain an inequality similar to (\ref{inequality}) for Theorem \ref{thm:actor_2} (see Appendix) and carry out an  analysis along similar lines for Theorems \ref{thm:actor_2} and \ref{thm:critic_2}. %As mentioned before, the details of all the proofs are available in the supplementary material.

 \subsection{Proof Sketch for Theorems \ref{thm:critic_1} and \ref{thm:critic_2}}
 We provide here an overview of the manner in which the proofs of Theorems \ref{thm:critic_1} and  \ref{thm:critic_2} proceed.

\textit{Proof Sketch for Average Reward estimate Error}\\
 For proving the convergence of average reward estimate in theorems \ref{thm:critic_1} and \ref{thm:critic_2}, we start off with expanding $ y_{t+1}^2 = (L_{t+1} - L_{t+1}^{*})^2$ which gives us the following inequality.
 \[
     y_{t+1}^2 \leq (1-2a(t))y_t^2 + 2a(t)y_t(C_t - L_t^*) \]\[ + 2y_t(L_t^* - L_{t+1}^*)
     + 2(L_t^* - L_{t+1}^*)^2 + 2a(t)^2(C_t - L_t)^2,
 \]
where $C_t = q(t) + \sum_{k=1}^{N}\gamma_{k}(t)(h_k(t)-\alpha_k)$.
After rearranging and summing the expectation of both sides from $\tau_t$ to $t$ we have, 
\begin{align*}
  &\sum_{k = \tau_t}^{t}E[y_k^2]
  \leq \underbrace{\sum_{k = \tau_t}^{t} \frac{1}{2a(k)}E(y_k^2 - y_{k+1}^2)}_{I_1}\\
  &\qquad+ \underbrace{\sum_{k=\tau_t}^{t}E[\hat{\Xi}(L_k,\theta_k,\gamma(k),q(k),h(k))] }_{I_2}\\
&\qquad+ \underbrace{\sum_{k = \tau_t}^{t} \frac{1}{a(k)}E[y_k(L_k^* - L_{k+1}^*)]}_{I_3}\\
&\qquad+ \underbrace{\sum_{k = \tau_t}^{t} \frac{1}{a(k)}E[(L_k^* - L_{k+1}^*)^2]}_{I_4}\\
&\qquad+ \underbrace{\sum_{k = \tau_t}^{t}a(k)E[(C_k - L_k)^2]}_{I_5}.
\end{align*}

Upon bounding $I_1 - I_5$ and simplifying, we get the desired result. For the definition of the  notations involved, please refer Section \ref{thm2_avg_reward}.

\textit{Proof Sketch for Critic Convergence}\\
For proving the convergence of critic in Theorems \ref{thm:critic_1} and \ref{thm:critic_2}, we start off with the equation $\Vert m_{t+1} \Vert^2 = \Vert \Gamma(v_{t} + a(t)\delta_{t}f_{s_t}) - v^*(t+1)\Vert^2$ and after expanding the RHS, we get the following inequality:
\begin{align*}
    \Vert m_{t+1}\Vert^2 & \leq \Vert m_t \Vert^2 + 2a(t)\langle m_t,\overline{g}(v_{t},\theta_t,\gamma(t)) \rangle \\
    &+ 2a(t)\Lambda(O_t,v_{t},\theta_t,\gamma(t),q(t),h(t)) \\
  & + 2 a(t)\langle m_t , \Delta g(O_t,L_t,\theta_t,\gamma(t)) \rangle \\
  &+ 2 \langle m_t,v^*(t) - v^*(t+1) \rangle \\
  &+ 8 a(t)^2(U_r + U_v)^2+ 2\Vert v^*(t) - v^*(t+1)\Vert^2.
\end{align*}
After rearranging and summing the expectation on both sides from $\tau_t$ to $t$ we obtain (see Appendix),
\begin{align*}
    &2\lambda_e \sum_{k = \tau_t}^{t}E\Vert m_k \Vert^2\\
    &\leq \underbrace{\sum_{k = \tau_t}^{t}\frac{1}{a(k)}\big{(}E\Vert m_k \Vert^2  - E\Vert m_{k+1} \Vert^2 \big{)} }_{I_1}\\
    &\qquad + \underbrace{2\sum_{k = \tau_t}^{t}E\Lambda(O_k,v_{k},\theta_k,\gamma(k),q(k),h(k))}_{I_2}\\
    &\qquad+ \underbrace{2\sum_{k = \tau_t}^{t}\sqrt{E \Vert m_k \Vert^2}\sqrt{ E [y_k^2]}}_{I_3}\\
   &\qquad+ \underbrace{B_q\sum_{k = \tau_t}^{t}\frac{b(k)}{a(k)} ) \sqrt{E \Vert m_k \Vert^2}}_{I_4}+ \underbrace{C_q\sum_{k = \tau_t}^{t}a(k)}_{I_5}.
\end{align*}

After analysing terms $I_1 - I_5$, we get the desired result. For definition of the  notations involved,  please refer section \ref{thm2_critic_converge}.

\begin{remark}
    \label{remark3}
    As mentioned previously, the asymptotic stability and almost sure convergence of the three-timescale C-AC algorithm has been shown in \cite{cmdpsa}. A similar analysis combining the results of \cite{nac} would also provide similar stability and asymptotic convergence results for the three-timescale C-NAC algorithm. 
    Note that while the non-asymptotic convergence results that we have shown are to the stationary points of the Lagrangian,
it is shown in \cite{cmdpsa} that for C-AC, stationary-point convergence indeed results in a locally optimal policy, that gives the set of local minima of the objective in the constraint set, see Proposition 4.3 and Remarks 4.3-4.5 there. This is because stationary points that are not local minima are unstable attractors of the underlying ODE. The same is also true of the C-NAC algorithm. Some works such as \cite{online_prime_dual} provide bounds on the optimality gap
but they primarily consider the look-up table setting and not function approximation. Defining optimality gap precisely in our setting is hard due to the presence of multiple local minima in the constraint set and so obtaining such bounds is not easy unless one has a setting of convex objectives and constraints (unlike us). 
From results in stochastic approximation theory (\cite{kushnerclark}), if the stochastic recursion enters a compact neighbourhood of a local minimum infinitely often, it will converge to it w.p.1. The compact neighbourhood of which of the minima the recursion enters in will depend on the initial condition and noise.
\end{remark}

\section{Experimental Results}

In this section, we present the results of our experiments on three different OpenAI Safety-Gym environments:  (a) SafetyPointGoal1-v0, (b) SafetyCarGoal1-v0 and (c) SafetyPointPush1-v0, respectively.   
The performance comparisons on these environments can be seen in Figure \ref{fig:experiment} and Tables \ref{tab:experiment_avgreward} and \ref{tab:experiment_avgconstraint}, respectively. %The upper row of plots in Figure \ref{fig:experiment} depict the average reward performance of the algorithms on various environments as a function of the number of iterations while the lower row corresponds to the constraint cost performance. 
The two tables summarize the performance obtained upon convergence of the algorithms. Please note that performance may vary depending on the initial seeds. In our experiments, we used a randomly generated set of seeds, and the reported results correspond to that particular selection. Note that Table \ref{tab:experiment_avgconstraint} has been placed in the appendix for lack of space. We also explain, in the appendix, the  Constrained DQN algorithm that we implemented in addition.

All the plots of our experiments are obtained after averaging over 10 different initial seeds. The performance of the algorithms is compared by plotting the average reward along with standard errors. %The plots in the upper row in Figure \ref{fig:experiment} indicate the average reward performance while those in the lower row indicate the constraint cost performance. 
The dotted flat red-line in each of the plots in the lower row of plots in Figure \ref{fig:experiment} corresponds to the constraint cost threshold. All algorithms are seen to asymptotically satisfy the constraint threshold while optimizing on the average reward performance. %It has been observed that the average reward is maximised and the long run average constraint cost is below the required threshold level. %We have given detailed explanation of the environmental settings and the algorithms in the Appendix. 

It can be observed that our C-NAC algorithm performs better than the other two algorithms on all three settings and C-AC shows the second best results on two of the three environments. Moreover, the cost threshold is met by all the three algorithms\footnote{The code for all of our experiments is available at https://github.com/prashu1306/Constrained-Actor-Critic}.

%For the Grid World setting , corresponding to each state there are four possible actions to take, viz., explore up, down, right or left. The agent probabilistically moves one step at any instant based on the current state and action and receives a cost and a constraint cost upon transitioning to a particular state. We consider a single constraint in this setting. The environment consists of a goal state as well as a fixed starting state. When the agent reaches the goal state, we assume that any action sends the agent back to the same starting state. Some states are assigned a high constraint cost (of 5 units) that the agent must avoid. The goal state has both single-stage cost and constraint cost as 0 (thereby making it the most desirable state). We experimentally compare C-AC and C-NAC with C-DQN .The plots of our experiments are averaged over 10 different initial seeds.The performance of the algorithms is compared by plotting the average reward along with standard errors.  It has been observed that the average reward is maximised and the long run average constraint cost is below the required threshold level.

\section{Conclusions}
We presented the first (non-asymptotic) finite time 
convergence analysis of three-timescale constrained actor-critic and constrained natural actor-critic algorithms using linear function approximation and obtained a sample complexity of $\mathcal{\tilde{O}}(\epsilon^{-2.5})$ for both algorithms. Our sample complexity result is significant as for both our algorithms it matches the sample complexity of regular (unconstrained) actor-critic (two-timescale) algorithms analysed in \cite{fta_2_timescale}. We also showed the results of experiments on three different Safety-Gym environments, where we observed that the C-NAC algorithm is better than both the C-AC and the C-DQN algorithms in the average reward performance and the C-AC algorithm is the second best performer on two of these three settings. Further, the average cost constraint is met by all the three algorithms on each of the settings.

\subsubsection*{Acknowledgements}
The authors were supported by
the Walmart Centre for Tech Excellence, Indian Institute of Science.
In addition, S.~Bhatnagar was supported by the J.~C.~ Bose National Fellowship of SERB, Project No. DFTM/02/3125/M/04/AIR-04 from
DRDO under DIA-RCOE 
and the Robert Bosch Centre for Cyber Physical Systems, Indian Institute of Science.
% References
%\bibliography{uai2024-template}

\newpage
\ifx\BlackBox\undefined
\newcommand{\BlackBox}{\rule{1.5ex}{1.5ex}}  % end of proof

\ifx\theorem\undefined
\newtheorem{theorem}{Theorem}
\fi
\ifx\proposition\undefined
\newtheorem{proposition}[theorem]{Proposition}
\fi
\ifx\assumption\undefined
\newtheorem{assumption}[theorem]{Assumption}
\fi
\ifx\corollary\undefined
\newtheorem{corollary}[theorem]{Corollary}
\fi
\ifx\proof\undefined
\newenvironment{proof}{\par\noindent{\bf Proof\ }}{\hfill\BlackBox\\[2mm]}

\onecolumn

\title{Appendix}
\maketitle

\appendix
\setcounter{section}{0}
\section{Preliminaries}
\label{prel}

For our problem, we have considered random single-stage costs and also constraint costs having distributions that are dependent on the current state, action and next state. For a given tuple $O_t = (s_t,a_t,s_{t+1})$, we consider

\begin{align*}
     q(t) \sim \bar{p}(.|s_t,a_t,s_{t+1}), \mbox{ } h_{i}(t) \sim p_{i}(.|s_t,a_t,s_{t+1}),  
\end{align*}
 for $i=1,2,3,..,N$.
For a given $\gamma = (\gamma_1,\gamma_2,...,\gamma_N)^T$ with $0 \leq \gamma_i \leq M$ and for $i=1,2,..,    N$, we define  $c(t) = q(t) + \sum\limits_{k=1}^{N}\gamma_{k}(h_k(t) - \alpha_{k})$ such that
\begin{align*}
    c(t) \sim \hat{p}(.|s_t,a_t,s_{t+1}).
\end{align*}
Since the single stage costs are mutually independent for any state-action-next state tuple,  we have
\begin{align*}
    \hat{p}(c(t)|s_t,a_t,s_{t+1}) = \bar{p}(q(t)|s_t,a_t,s_{t+1}) \prod\limits_{i=1}^{N}p_{i}(h_i(t)|s_t,a_t,s_{t+1}). 
\end{align*}

Let $\Vert \nabla L(\theta,\gamma) \Vert < U_L , \forall \theta,\gamma$ where $U_L > 0$ is a positive constant. Also, 
$\sup\limits_{t}\Vert G(t)^{-1} \Vert \leq U_G$ where $U_G$ is a positive constant.Recall that the single stage costs $q(n),h_1(n),\ldots,h_N(n)$ are  non-negative and  absolutely bounded by a constant $U_{c}>0$.We define $U_r = U_c + NM(U_c + U_\alpha)$ where  $U_\alpha = \max\limits_{k}\vert \alpha_{k}\vert$.
We assume $0 \leq L_0 \leq U_r$. This implies $\vert L_{t} \vert \leq U_r ,\forall t > 0$. We also assume $0 \leq U_{k}(0) \leq U_c , \forall k \in {1,2,..,N}$, which implies $\vert U_{k}(t)\vert \leq U_c ,\forall t >0 $. Moreover, we assume that the projection set $C$ to which the recursion $v_t$ is projected is a compact and convex set.

Given time indices $t$ and $\tau$ such that $t \geq \tau > 0$, we consider the following auxiliary Markov chain starting from $s_{t-\tau}$ to deal with Markov noise in the iterates.
\begin{align}\label{chain-au}
    s_{t-\tau}\xrightarrow{\theta_{t-\tau}}a_{t-\tau}\xrightarrow{\mathcal{P}}s_{t-\tau+1}\xrightarrow{\theta_{t-\tau}}\widetilde{a}_{t-\tau+1}\xrightarrow{\mathcal{P}}\widetilde{s}_{t-\tau+2}\xrightarrow{\theta_{t-\tau}}\widetilde{a}_{t-\tau+2}\cdots \xrightarrow{\mathcal{P}}\widetilde{s}_t\xrightarrow{\theta_{t-\tau}}\widetilde{a}_t\xrightarrow{\mathcal{P}}\widetilde{s}_{t+1}.
\end{align}

The original Markov chain has the following transitions:
\begin{align}\label{chain-or}
    s_{t-\tau}\xrightarrow{\theta_{t-\tau}}a_{t-\tau}\xrightarrow{\mathcal{P}}s_{t-\tau+1}\xrightarrow{\theta_{t-\tau+1}}a_{t-\tau+1}\xrightarrow{\mathcal{P}}s_{t-\tau+2}\xrightarrow{\theta_{t-\tau+2}}a_{t-\tau+2}\cdots \xrightarrow{\mathcal{P}}s_t\xrightarrow{\theta_{t}}a_t\xrightarrow{\mathcal{P}}s_{t+1}.
\end{align}
In (\ref{chain-au}), for any time instant $k > t-\tau$, $\widetilde{a}_{k}$ denotes the action taken in the auxiliary Markov chain. Similarly, for any time instant $l > t-\tau + 1$, $\widetilde{s}_{l}$ denotes the state in the auxiliary Markov chain. 
In the following, without any other specification, $E[.]$ will denote the expectation w.r.t the joint distribution of all the random variables involved.
Finally, we mention that for our analysis we interchangeably use $\pi(a|s)$ and $\pi(s,a)$ to mean the action chosen in state $s$ according to the policy $\pi$. Thus, both these notations are one and the same.

\section{Proof of Theorems}
\subsection{Proof of Theorem 1}\label{thmm:actor_1}

We first define several notations to clarify the dependence in the various quantities involved. Let
\begin{align}
\begin{split}
O_{t} 
:&=
(s_t,a_t,s_{t+1}),\\
L^{*} :
&=
L(\theta,\gamma) = E_{s \sim \mu_{\theta},a \sim \pi_{\theta}}[d(s,a) + \sum_{k=1}^{N}\gamma(k)(h_{k}(s,a)-\alpha_{k})],\\
v(k)^{*}
:&=
v^{*}(\theta_{k},\gamma(k)),\\
\Delta H(O,L,v,\theta,\gamma) 
:&=
( L(\theta,\gamma) - L + ({f_{s^{'}}}^{T} - f_s^T)(v-v^*(\theta,\gamma)))\nabla \log\pi_{\theta}(a|s),\\
H(O,\theta,\gamma,q,h) 
:&=
(q -L(\theta,\gamma) + \sum_{k=1}^{N}\gamma(k)(h_k-\alpha_k) +({f_{s^{'}}}^{T} - f_s^T)v^*(\theta,\gamma))\nabla \log\pi_{\theta}(a|s),\\
\Delta H^{'}(O,\theta,\gamma) 
:&=
( {f_{s^{'}}}^{T}v^{*}(\theta,\gamma) - V^{\theta,\gamma}(s^{'}) -({f_{s}}^{T}v^{*}(\theta,\gamma) - V^{\theta,\gamma}(s)))\nabla \log\pi_{\theta}(a|s),\\
\check{\Gamma}(O,\theta,\gamma,q,h) :&=
\langle \nabla L(\theta,\gamma) , H(O,\theta,\gamma,q,h) - E_{O^{'},q,h}[H(O^{'},\theta,\gamma,q,h) ]\rangle.\\
\end{split}
\end{align}

 In the above, $O = (s, a, s^{'})$, $h = (h_1,h_2,h_3,...,h_N)$.$ O^{'} = (s,a,s^{'})$ denotes the independent sample $s \sim \mu_{\theta},$ $a \sim \pi_{\theta}$, $s^{'} \sim p(s,.,a)$ .Hence $E_{O^{'},q,h}[\cdot]$ denotes expectation w.r.t.~the joint distribution of $s \sim \mu_{\theta},$ $a \sim \pi_{\theta}$, $s^{'} \sim p(s,.,a)$, $q \sim \bar{p}(.|s,a,s^{'})$, $h_{i} \sim p_{i}(.|s,a,s^{'})$, $i=1,\ldots,N$. We now have,
\begin{align*}
 & E_{O^{'},q,h}[(H(O^{'},\theta,\gamma,q,h) - \Delta H^{'}(O^{'},\theta,\gamma) ] \\
 & =  E_{O^{'},q,h}[q + \sum_{k=1}^{N}\gamma(k)(h_k-\alpha_k) - L(\theta,\gamma) + V^{\theta,\gamma}(s^{'}) - V^{\theta,\gamma}(s))\nabla \log\pi_{\theta}(a|s)] \\
 &= \nabla L(\theta,\gamma), 
\end{align*}

Next, observe that 
\begin{align*}
E_{O^{'}} \Vert \Delta H^{'}(O,\theta,\gamma)\Vert^2
&= E_{O^{'}} \Vert({f_{s^{'}}}^{T}v^{*}(\theta,\gamma) - V^{\theta,\gamma}(s^{'}) -({f_{s}}^{T}v^{*}(\theta,\gamma)  - V^{\theta,\gamma}(s)))\nabla \log\pi_{\theta}(a|s)\Vert^2\\
&\leq E_{O^{'}}[D^2( {f_{s^{'}}}^{T}v^{*}(\theta,\gamma) - V^{\theta,\gamma}(s^{'}) -({f_{s}}^{T}v^{*}(\theta,\gamma)  - V^{\theta,\gamma}(s)))^2]\\
&\leq  E_{O^{'}}[2D^2(({f_{s^{'}}}^{T}v^{*}(\theta,\gamma) - V^{\theta,\gamma}(s^{'}))^2  + ({f_{s}}^{T}v^{*}(\theta,\gamma) - V^{\theta,\gamma}(s))^2)]\\
&\leq 4D^2\epsilon_{app}^2.
\end{align*}

Before we proceed further, we first state and prove Lemmas 1--4 below that will be used in the proof of Theorem 1. Moreover, the proof of Lemma 4 shall rely on Lemmas \ref{Lemma:Gamma_gamma} -- \ref{Lemma:Gamma_second} that we also state and prove in the following. Finally, collecting all these results together, we shall obtain the claim for Theorem 1.
%There are several lemmas that will be used in this proof.

\setcounter{lemma}{0}

\begin{lemma} \label{Lemma:L-smooth}
For all $\bgamma \in R^N$, with $0 \leq \bgamma_{i} \leq M$, where $i \in \{1,2,...,N\}$, there exists a constant $M_{L}$ greater than 0 such  that for all $\theta_1, \theta_2 \in {R^d}$, 
 \begin{align*}
      \Vert \nabla L(\theta_1,\bgamma) - \nabla L(\theta_2,\bgamma)\Vert \leq M_{L}\Vert\theta_1 - \theta_2\Vert,
 \end{align*}
which implies
\begin{align*}
L(\theta_2,\gamma) > L(\theta_1,\gamma) + \langle \nabla L(\theta_1,\gamma) ,\theta_2 - \theta_1 \rangle - \frac{M_{L}}{2} \Vert\theta_1 - \theta_2\Vert^2.
\end{align*}

\end{lemma}

\begin{proof}
    Note that 
\begin{align*}
   &L(\theta, \gamma)  = J(\theta) + \sum\limits_{k=1}^{N}\gamma(k)(G_k(\theta) - \alpha_k).\\
   &\Rightarrow \nabla_{\theta} L(\theta, \gamma) = \nabla_{\theta}J(\theta) + \sum\limits_{k=1}^{N}\gamma(k)\nabla_{\theta}G_k(\theta). 
\end{align*}

From Lemma C.1 of \cite{fta_2_timescale}, we have $\forall \theta_1,\theta_2 \in R^d$, there exist positive constants $L_{J},L_{G_1},L_{G_2},...,L_{G_N}$ such that
\begin{align*}
    \Vert \nabla_{\theta}J(\theta_1) - \nabla_{\theta}J(\theta_2) \Vert &\leq L_{J}\Vert \theta_1 - \theta_2 \Vert,\\
    \Vert \nabla_{\theta}G_i(\theta_1) - \nabla_{\theta}G_i(\theta_2) \Vert &\leq L_{G_i}\Vert \theta_1 - \theta_2 \Vert ,\forall i \in {1,2,..,N}.
\end{align*}
Hence, 
\begin{align*}
    \Vert \nabla_{\theta} L(\theta_1, \gamma) - \nabla_{\theta} L(\theta_2, \gamma) \Vert &= \Vert \nabla_{\theta}J(\theta_1) - \nabla_{\theta}J(\theta_2) + \sum\limits_{k=1}^{N}\gamma_{k}(\nabla_{\theta}G_k(\theta_1)-\nabla_{\theta}G_k(\theta_2))  \Vert\\
    &\leq \Vert \nabla_{\theta}J(\theta_1) - \nabla_{\theta}J(\theta_2)\Vert + \sum\limits_{k=1}^{N}\gamma_{k}\Vert\nabla_{\theta}G_k(\theta_1)-\nabla_{\theta}G_k(\theta_2)  \Vert\\
    &\leq L_{J}\Vert \theta_{1} - \theta_{2}\Vert + \sum\limits_{k=1}^{N}\gamma_{k} L_{G_{k}}\Vert \theta_{1} - \theta_{2}\Vert\\
    &=(L_{J} + \sum\limits_{k=1}^{N}\gamma_{k} L_{G_{k}})\Vert \theta_{1} - \theta_{2} \Vert\\
    &\leq M_{L}\Vert \theta_{1} - \theta_{2} \Vert.
\end{align*}

So we have,
\begin{align*}
    \Vert \nabla_{\theta} L(\theta_1, \gamma) - \nabla_{\theta} L(\theta_2, \gamma) \Vert \leq M_{L}\Vert \theta_{1} - \theta_{2} \Vert ,\forall \theta_1,\theta_2 \in R^d,
\end{align*}
where $M_{L} = L_{J} + M\sum\limits_{k=1}^{N} L_{G_{k}}$. The claim follows. 
\end{proof}

\begin{lemma} \label{Lemma:L-smooth2}
 For all $\bgamma^1,\bgamma^2 \in R^N$ with $0 \leq \bgamma_{i}^{j} \leq M$, where $i \in \{1,2,...,N\}$, $j = 1,2$, there exists a constant $C$ greater than 0 such  that for all $\btheta \in {R^d}$,
 \begin{align*}
      \Vert \nabla L(\btheta,\bgamma^1) - \nabla L(\btheta,\bgamma^2)\Vert \leq C\vert \bgamma^1_m - \bgamma^2_m\vert,\
 \end{align*}
where $\vert \bgamma^1_m - \bgamma^2_m\vert = \max\limits_{i=1,2,..,N}\vert \bgamma^1_i - \bgamma^2_i\vert$.
\end{lemma}

\begin{proof}
    We have,
\begin{align*}
         \Vert \nabla L(\btheta,\bgamma^1) - \nabla L(\btheta,\bgamma^2)\Vert &= \bigg\Vert \sum\limits_{s \in S}\mu_{\btheta}(s)\sum\limits_{a \in A(s)}\pi_{\btheta}(a|s)(M^{\pi_{\btheta},\bgamma^1}(s,a) -M^{\pi_{\btheta},\bgamma^2}(s,a))\nabla \log \pi_{\btheta}(a|s)  \bigg\Vert\\
         & \leq  \sum\limits_{s \in S}\mu_{\btheta}(s)\sum\limits_{a \in A(s)}\pi_{\btheta}(a|s)\big\Vert(M^{\pi_{\btheta},\bgamma^1}(s,a) -M^{\pi_{\btheta},\bgamma^2}(s,a))\nabla \log \pi_{\btheta}(a|s)  \big\Vert\\
         &\leq D\sum\limits_{s \in S}\mu_{\btheta}(s)\sum\limits_{a \in A(s)}\pi_{\btheta}(a|s)\big\vert M^{\pi_{\btheta},\bgamma^1}(s,a) -M^{\pi_{\btheta},\bgamma^2}(s,a) \big\vert,
\end{align*}
where the last inequality follows from Assumption 4.
Now,
\begin{align*}
    \big\vert M^{\pi_{\btheta},\bgamma^1}(s,a) -M^{\pi_{\btheta},\bgamma^2}(s,a)) \big\vert &= \big\vert \sum\limits_{t=1}^{\infty}E\bigg[ \sum\limits_{i=1}^{N}(\bgamma^1_i -\bgamma^2_i ) (h_i(t) - G_i(\btheta))
     | s_0 = s,a_0 = a,\pi_{\btheta}\bigg]\big\vert\\
     &= \bigg\vert\sum\limits_{i=1}^{N}(\bgamma^1_i -\bgamma^2_i )\sum\limits_{t=1}^{\infty}E\bigg[(h_i(t) - G_i(\btheta))| s_0 = s,a_0 = a,\pi_{\theta}\bigg]\bigg\vert\\
     &\leq \sum\limits_{i=1}^{N}\vert(\bgamma^1_i -\bgamma^2_i )\vert\bigg\vert\sum\limits_{t=1}^{\infty}E\bigg[(h_i(t) - G_i(\btheta))| s_0 = s,a_0 = a,\pi_{\theta}\bigg]\bigg\vert\\
     &\leq N \bar{Q}\vert \bgamma^1_m - \bgamma^2_m\vert,
\end{align*}
where,
\begin{align*}
    \bar{Q} &= \sup\limits_{i,\btheta,s,a}\vert Q^{\btheta}_{i}(s,a)\vert,\\
\vert \bgamma^1_m - \bgamma^2_m\vert &= \max_{i=1,2,..,N}\vert \bgamma^1_i - \bgamma^2_i\vert,
\end{align*}
and where 
$Q^{\btheta}_{i}(s,a)$ is the expected differential cost for the state action pair $(s,a)$ with single-stage cost as $h_i(t)$ at time instant $t$ when actions are picked according to policy $\pi_{\btheta}$.
Hence,
\begin{align*}
    \Vert \nabla L(\btheta,\bgamma^1) - \nabla L(\btheta,\bgamma^2)\Vert \leq ND \bar{Q}\vert \bgamma^1_m - \bgamma^2_m\vert.
\end{align*}
The claim follows by letting $C=ND\bar{Q}$.
\end{proof}

\begin{lemma}\label{lemma:del_h}
For any $t > 0$,
\begin{align*}
    \Vert \Delta H(O_t,L_t,v_t,\theta_t,\gamma(t))\Vert^2 \leq D^2(2(L_t - L(\theta_t,\gamma(t)))^2 + 8\Vert v_{t} - v(t)^*\Vert^2).
\end{align*}
\end{lemma}

\begin{proof}
    Applying the definition of $\Delta H(\cdot)$  immediately yields the result.
\end{proof}

\begin{lemma} \label{lemma:gamma}
For any $t \geq 0$,
\begin{align*}
    E[\check{\Gamma}(O_t,\theta_t,\gamma(t),q(t),h(t))] \geq-(D_1(\tau + 1)\sum_{k=t-\tau + 1}^{t} E\Vert \theta_k - \theta_{k-1} \Vert + D_2bk^{\tau - 1} + T_1\sum_{i=t-\tau +1 }^{ t} E\vert \gamma_m(i) - \gamma_m(i-1)\vert),
\end{align*}
where $D_1,D_2$ and $T_1$ are positive constants and  $t \geq \tau \geq 0$.
\end{lemma}
\begin{proof}
    We have 
\begin{align*}
   & E[\check{\Gamma}(O_t,\theta_t,\gamma(t),q(t),h(t))]\\ &=
E_{s_{t} \sim p,a_t \sim \pi_{\theta_t},s_{t+1} \sim p}[E[\langle \nabla L(\theta_t,\gamma(t)), H(O_t,\theta_t,\gamma(t),q(t),h(t)) - E_{O^{'},q,h}[H(O^{'},\theta_t,\gamma(t),q,h) ]\rangle|s_t,a_t,s_{t+1}]]\\
&= E[\langle \nabla L(\theta_t,\gamma(t)) , \bar{H}(O_t,\theta_t,\gamma(t)) - E_{O^{'},q,h}[H(O^{'},\theta_t,\gamma(t),q,h) ]\rangle]\\
&= E[\langle \nabla L(\theta_t,\gamma(t)) , \bar{H}(O_t,\theta_t,\gamma(t)) - E_{O^{'}}[E_{q,h}[H(O^{'},\theta_t,\gamma(t),q,h)|s,a,s^{'} ]]\rangle]\\
&= E[\langle \nabla L(\theta_t,\gamma(t)) , \bar{H}(O_t,\theta_t,\gamma(t)) - E_{O^{'}}[\bar{H}(O^{'},\theta_t,\gamma(t))]\rangle]\\
&= E[Q(O_t,\theta_t,\gamma(t))],
\end{align*}
where, 
\begin{align*}
    \bar{H}(O,\theta,\gamma) &= (c(s,a,s^{'},\gamma) -L(\theta,\gamma)  +({f_{s^{'}}}^{T} - {f_{s}}^{T})v^{*}(\theta,\gamma))\nabla \log\pi_{\theta}(a|s)\\
    c(s,a,s^{'},\gamma) &= \sum\limits_{q}(q\cdot \bar{p}(q|s,a,s^{'})) + \sum\limits_{k=1}^{k=N}\gamma_{k}(\sum\limits_{h}(h\cdot p_{k}(h|s,a,s^{'}) )- \alpha_k).
\end{align*}
The second equality is satisfied because $\theta_t$ and $\gamma(t)$ do not depend on $q(t)$ and $h(t)$. The remaining 
proof of Lemma~\ref{lemma:gamma} in turn requires Lemmas~\ref{Lemma:Gamma_gamma} -- \ref{Lemma:Gamma_second} below that we now state and prove.

\begin{sublemma}\label{Lemma:Gamma_gamma}
For any $t \geq 0$,

\begin{align*}
    \vert Q(O_t,\theta_t,\gamma(t)) - Q(O_t,\theta_t,\gamma(t-\tau))\vert \leq T_1\vert\gamma_{m}(t) - \gamma_{m}(t-\tau)\vert,
\end{align*}
where $T_1 > 0$ is a constant.

\end{sublemma}
\textit{Proof}
Denoting $O = (s,a,s^{'})$,
we have for any $\theta ,\gamma_1,\gamma_{2}$, that
\begin{align*}
    &Q(O,\theta,\gamma^1) - Q(O,\theta,\gamma^2) \\
&= \langle \nabla L(\theta,\gamma^1) , \bar{H}(O,\theta,\gamma^1) - E_{O^{'}}[\bar{H}(O^{'},\theta,\gamma^1) ]\rangle - \langle \nabla L(\theta,\gamma^2) , \bar{H}(O,\theta,\gamma^2) - E_{O^{'}}[\bar{H}(O^{'},\theta,\gamma^2) ]\rangle\\
&=\langle \nabla L(\theta,\gamma^1) , \bar{H}(O,\theta,\gamma^1) - E_{O^{'}}[\bar{H}(O^{'},\theta,\gamma^1) ]\rangle -\langle \nabla L(\theta,\gamma^1) , \bar{H}(O,\theta,\gamma^2) - E_{O^{'}}[\bar{H}(O^{'},\theta,\gamma^2) ]\rangle \\
&\qquad +\langle \nabla L(\theta,\gamma^1) , \bar{H}(O,\theta,\gamma^2) - E_{O^{'}}[\bar{H}(O^{'},\theta,\gamma^2) ]\rangle - \langle \nabla L(\theta,\gamma^2) , \bar{H}(O,\theta,\gamma^2) - E_{O^{'}}[\bar{H}(O^{'},\theta,\gamma^2) ]\rangle\\
&= \underbrace{\langle \nabla L(\theta,\gamma^1) , \bar{H}(O,\theta,\gamma^1)  -  \bar{H}(O,\theta,\gamma^2) - E_{O^{'}}[\bar{H}(O^{'},\theta,\gamma^1)] + E_{O^{'}}[\bar{H}(O^{'},\theta,\gamma^2)] \rangle}_{I_1}\\
&\qquad + \underbrace{\langle \nabla L(\theta,\gamma^1) - \nabla  L(\theta,\gamma^2),\bar{H}(O,\theta,\gamma^2) - E_{O^{'}}[\bar{H}(O^{'},\theta,\gamma^2) ]\rangle}_{I_2}.
\end{align*}
Now,
\begin{align*}
    &\Vert \bar{H}(O,\theta,\gamma^1)  -  \bar{H}(O,\theta,\gamma^2) \Vert\\
    &= \Vert (c(s,a,s^{'},\gamma^1)- c(s,a,s^{'},\gamma^2) - L(\theta,\gamma^1) + L(\theta,\gamma^2)  +({f_{s^{'}}}^{T} - {f_{s}}^{T})(v^{*}(\theta,\gamma^1)-v^{*}(\theta,\gamma^2)))\nabla \log\pi_{\theta}(a|s) \Vert\\
    &\leq D(\vert c(s,a,s^{'},\gamma^1)- c(s,a,s^{'},\gamma^2) \vert + \vert L(\theta,\gamma^1) - L(\theta,\gamma^2) \vert + 2\Vert v^{*}(\theta,\gamma^1)-v^{*}(\theta,\gamma^2) \Vert )\\
    &\leq D(2N(U_c + U_{\alpha})\vert \gamma^1_m - \gamma^2_m \vert + 2L_2\vert \gamma^1_m - \gamma^2_m \vert),
\end{align*}
where  $\vert\gamma_m^1-\gamma_m^2\vert = \max\limits_{i = 1,2,3.... , N}|\gamma_i^1-\gamma_i^2|$. In the above, $U_c$ and $U_\alpha$ are as before, see Section~\ref{prel}. Further, 
from Assumption 1, $\vert f_s\vert \leq 1$ and hence $\vert f_{s'}-f_s\vert \leq 2$.
Thus, for term $I_1$, note that 
\begin{align*}
    I_1 &\leq 4D(N(U_c + U_\alpha) + L_2)\Vert \nabla L(\theta,\gamma^1)\Vert \vert \gamma^1_m - \gamma^2_m \vert.
\end{align*}
Further, for term $I_{2}$, we have
\begin{align*}
    I_2 &\leq \Vert \nabla L(\theta,\gamma^1) - \nabla  L(\theta,\gamma^2) \Vert \Vert \bar{H}(O,\theta,\gamma^2) - E_{O^{'}}[\bar{H}(O^{'},\theta,\gamma^2) ]\Vert\\
    &\leq 4D(U_r + U_v)\Vert \nabla L(\theta,\gamma^1) - \nabla  L(\theta,\gamma^2) \Vert\\
    &\leq 4DC(U_r + U_v)\vert \gamma^1_m - \gamma^2_m\vert.
\end{align*} 
The last inequality above is because of Lemma \ref{Lemma:L-smooth2}.
After collecting the two parts, we now have,
\begin{align*}
    \vert Q(O,\theta,\gamma^1) - Q(O,\theta,\gamma^2)\vert &\leq T_1\vert \gamma^1_m - \gamma^2_m\vert,
\end{align*}
where $T_1 =  4D(N(U_c + U_\alpha) + L_2)U_{L}+  4DC(U_r + U_v)$.

\begin{sublemma}\label{Lemma:Gamma_theta}
For any $t \geq 0$,$\theta_1,\theta_2 \in R^d$,$\gamma = (\gamma_1,\gamma_2,...,\gamma_N)^T $ with $0 \leq \gamma_i \leq M $ where $i \in \{1,2,..,N\}$
\begin{align*}
    \vert Q(O,\theta_1,\gamma) - Q(O,\theta_2,\gamma)\vert &\leq T_2\Vert \theta_1 - \theta_2 \Vert,
\end{align*}
for some $T_2 > 0$.
\end{sublemma}
\textit{Proof} Recall that $O_t = (s_t,a_t,s_{t+1})$, hence 
for any $\theta_{1} ,\theta_{2},\gamma$,
\begin{align*}
    &Q(O,\theta_1,\gamma) - Q(O,\theta_2,\gamma) \\
&= \langle \nabla L(\theta_1,\gamma) , \bar{H}(O,\theta_1,\gamma) - E_{O^{'}}[\bar{H}(O^{'},\theta_1,\gamma) ]\rangle - \langle \nabla L(\theta_2,\gamma) , \bar{H}(O,\theta_2,\gamma) - E_{O^{'}}[\bar{H}(O^{'},\theta_2,\gamma) ]\rangle\\
&=\langle \nabla L(\theta_1,\gamma) , \bar{H}(O,\theta_1,\gamma) - E_{O^{'}}[\bar{H}(O^{'},\theta_1,\gamma) ]\rangle -\langle \nabla L(\theta_1,\gamma) , \bar{H}(O,\theta_2,\gamma) - E_{O^{'}}[\bar{H}(O^{'},\theta_2,\gamma) ]\rangle \\
&\qquad +\langle \nabla L(\theta_1,\gamma) , \bar{H}(O,\theta_2,\gamma) - E_{O^{'}}[\bar{H}(O^{'},\theta_2,\gamma) ]\rangle - \langle \nabla L(\theta_2,\gamma) , \bar{H}(O,\theta_2,\gamma) - E_{O^{'}}[\bar{H}(O^{'},\theta_2,\gamma) ]\rangle\\
&= \underbrace{\langle \nabla L(\theta_1,\gamma) , \bar{H}(O,\theta_1,\gamma)  -  \bar{H}(O,\theta_2,\gamma) - E_{O^{'}}[\bar{H}(O^{'},\theta_1,\gamma)] + E_{O^{'}}[\bar{H}(O^{'},\theta_2,\gamma)] \rangle}_{I_1}\\
&\qquad + \underbrace{\langle \nabla L(\theta_1,\gamma) - \nabla  L(\theta_2,\gamma),\bar{H}(O,\theta_2,\gamma) - E_{O^{'}}[\bar{H}(O^{'},\theta_2,\gamma) ]\rangle}_{I_2}.
\end{align*}
     
Now,
\begin{align*}
    &\Vert \bar{H}(O,\theta_1,\gamma) - \bar{H}(O,\theta_2,\gamma)\Vert  \\
    &=\Vert (c(s,a,s^{'},\gamma) -L(\theta_1,\gamma)  +({f_{s^{'}}}^{T} - {f_{s}}^{T})v^{*}(\theta_1,\gamma))\nabla \log\pi_{\theta_1}(a|s)\\
    &\qquad- (c(s,a,s^{'},\gamma) -L(\theta_2,\gamma)  +({f_{s^{'}}}^{T} - {f_{s}}^{T})v^{*}(\theta_2,\gamma))\nabla \log\pi_{\theta_2}(a|s) \Vert\\
    &\leq \Vert (c(s,a,s^{'},\gamma) -L(\theta_1,\gamma)  +({f_{s^{'}}}^{T} - {f_{s}}^{T})v^{*}(\theta_1,\gamma))(\nabla \log\pi_{\theta_1}(a|s) - \nabla \log\pi_{\theta_2}(a|s)) \Vert\\
    &\qquad+ \Vert (L(\theta_2,\gamma) - L(\theta_1,\gamma) + ({f_{s^{'}}}^{T} - {f_{s}}^{T})(v^{*}(\theta_1,\gamma) - v^{*}(\theta_1,\gamma)))\nabla \log\pi_{\theta_2}(a|s)\Vert\\
    &\leq 2(U_r + U_v)M_m\Vert \theta_1 - \theta_2\Vert + D(\vert L(\theta_2,\gamma) - L(\theta_1,\gamma) \vert + 2L_1\Vert \theta_1 - \theta_2 \Vert)
\end{align*}
Clearly,
\begin{align*}
    &\vert L(\theta_1,\gamma) - L(\theta_2,\gamma) \vert\\
    &= \vert \sum\limits_{s \in S}\mu_{\theta_1}(s)\sum\limits_{a \in A(s)}\pi_{\theta_1}(s,a) (d(s,a) + \sum\limits_{k=1}^{N}\gamma(k)(h_k(s,a) -\alpha_k)) - \sum\limits_{s \in S}\mu_{\theta_2}(s)\sum\limits_{a \in A(s)}\pi_{\theta_2}(s,a) (d(s,a) + \sum\limits_{k=1}^{N}\gamma(k)(h_k(s,a) -\alpha_k)) \vert\\
    &\leq 2U_r d_{TV}(\mu_{\theta_1} \otimes \pi_{\theta_1},\mu_{\theta_2} \otimes \pi_{\theta_2})\\
    & \leq 2U_r\vert A\vert L\bigg( 1  + \lceil\log_{k}b^{-1}\rceil + 1/(1-k)\bigg)\Vert \theta_1 - \theta_2 \Vert\\
    &=C_{L}\Vert \theta_1 - \theta_2 \Vert.
\end{align*}
The second inequality is from lemma B.1 of \cite{fta_2_timescale}.
Hence,
\begin{align*}
\Vert \bar{H}(O,\theta_1,\gamma) - \bar{H}(O,\theta_2,\gamma)\Vert &\leq  2(U_r + U_v)M_m\Vert \theta_1 - \theta_2\Vert + DC_{L}\Vert \theta_1 - \theta_2 \Vert +2L_1D\Vert \theta_1 - \theta_2\Vert\\
&=A_1\norm{\theta_1-\theta_2}
\end{align*}

Now,
\begin{align*}
    &\Vert E_{O^{'}}[\bar{H}(O^{'},\theta_1,\gamma)] - E_{O^{'}}[\bar{H}(O^{'},\theta_2,\gamma)] \Vert\\ &= \Vert E_{\theta_1}[\bar{H}(O^{'},\theta_1,\gamma)] - E_{\theta_2}[\bar{H}(O^{'},\theta_2,\gamma)] \Vert\\
    &\leq \Vert E_{\theta_1}[\bar{H}(O^{'},\theta_1,\gamma)] - E_{\theta_1}[\bar{H}(O^{'},\theta_2,\gamma)] \Vert + \Vert E_{\theta_1}[\bar{H}(O^{'},\theta_2,\gamma)] - E_{\theta_2}[\bar{H}(O^{'},\theta_2,\gamma)] \Vert\\
    &\leq E_{\theta_1}\Vert \bar{H}(O^{'},\theta_1,\gamma)-\bar{H}(O^{'},\theta_2,\gamma)\Vert + 4D(U_r + U_v)d_{TV}(\mu_{\theta_1} \otimes \pi_{\theta_1} , \mu_{\theta_2} \otimes \pi_{\theta_2} )\\
    &\leq \bigg[2(U_r + U_v)M_m + DC_L + 2L_1D + 4D(U_r + U_v)|A|U_rL \bigg(1 + \lceil \log_{k}b^{-1} \rceil + \frac{1}{1-k} \bigg) \bigg]\norm{\theta_1 - \theta_2}\\
    &=A_2\norm{\theta_1-\theta_2}
\end{align*}
Thus, we have,
\begin{align*}
    I_1 \leq U_L(A_1 + A_2)\norm{\theta_1 - \theta_2}
\end{align*}
For term $I_2$, we have,
\begin{align*}
    I_2 &\leq \Vert \nabla L(\theta_1,\gamma) - \nabla  L(\theta_2,\gamma) \Vert \Vert \bar{H}(O,\theta_2,\gamma) - E_{O^{'}}[\bar{H}(O^{'},\theta_2,\gamma) ] \Vert\\
    &\leq 4D(U_r + U_v)M_{L}\Vert \theta_1 - \theta_2\Vert.
\end{align*}
The last inequality follows from Lemma \ref{Lemma:L-smooth}.
Thus, 
\begin{align*}
    Q(O,\theta_1,\gamma) - Q(O,\theta_2,\gamma) &\leq T_2\Vert \theta_1 - \theta_2 \Vert,
\end{align*}
where 
\begin{align*}
    T_2 &= (A_1 + A_2)U_L +4D(U_r + U_v)M_{L}\\
    A_1 &=2(U_r + U_v)M_m + DC_L + 2L_1D\\
    A_2 &=2(U_r + U_v)M_m + DC_L + 2L_1D + 4D(U_r + U_v)|A|U_rL \bigg(1 + \lceil \log_{k}b^{-1} \rceil + \frac{1}{1-k} \bigg) 
\end{align*}
Here $E_{\theta}$ denotes that $O^{'} = (s,a,s^{'})$ has been sampled as $s \sim \mu_{\theta} , a \sim \pi_{\theta},s^{'} \sim p(s,.,a)$.

\begin{sublemma}\label{Lemma:Gamma_first}
For any $t \geq 0$,%conditioned on $s_{t-\tau + 1},\theta_{t-\tau}$ and $\gamma(t-\tau)$, 
we have
\begin{align*}
     \vert E[(Q(O_t,\theta_{t-\tau},\gamma(t-\tau)) - Q(\tilde{O_{t}},\theta_{t-\tau},\gamma(t-\tau)))|\theta_{t-\tau},\gamma(t-\tau),s_{t-\tau + 1}]\vert \leq 2D(U_{r} + U_{v})U_L\vert A \vert L\sum_{i=t-\tau}^{t} E\Vert\theta_i - \theta_{t-\tau} \Vert. 
\end{align*}
\end{sublemma}

\textit{Proof}
From the definition of $Q(O,\theta,\gamma)$, we have that
\begin{align*}
    &E[(Q(O_t,\theta_{t-\tau},\gamma(t-\tau)) - Q(\tilde{O_{t}},\theta_{t-\tau},\gamma(t-\tau))|\theta_{t-\tau},\gamma(t-\tau),s_{t-\tau + 1}]\\
&= E[\langle \nabla L(\theta_{t-\tau},\gamma(t-\tau)),\bar{H}(O_t,\theta_{t-\tau},\gamma(t-\tau)- \bar{H}(\tilde{O_{t}},\theta_{t-\tau},\gamma(t-\tau)\rangle|\theta_{t-\tau},\gamma(t-\tau),s_{t-\tau + 1}]\\
&=  E[(\langle \nabla L(\theta_{t-\tau},\gamma(t-\tau)),\bar{H}(O_t,\theta_{t-\tau},\gamma(t-\tau)\rangle - \langle \nabla L(\theta_{t-\tau},\gamma(t-\tau)) , \bar{H}(\tilde{O_{t}},\theta_{t-\tau},\gamma(t-\tau))\rangle)|\theta_{t-\tau},\gamma(t-\tau),s_{t-\tau + 1}]\\
&\leq 4D(U_{r} + U_{v})U_{L}d_{TV}(P(O_t = .|s_{t-\tau +1},\theta_{t-\tau}),(P(\tilde{O_{t}} = .|s_{t-\tau +1},\theta_{t-\tau})).
\end{align*}
Now,
\begin{align*}
    d_{TV}(P(O_t = .|s_{t-\tau +1},\theta_{t-\tau}), P(\tilde{O_{t}} = .|s_{t-\tau +1},\theta_{t-\tau})) \leq \frac{1}{2}\vert A \vert L\sum_{i=t-\tau}^{t} E\Vert\theta_i - \theta_{t-\tau} \Vert.
\end{align*}
The inequality above is similar to the one shown in the proof of Lemma D.2 of \cite{fta_2_timescale}. Hence,
\begin{align*}
    E[Q(O_t,\theta_{t-\tau},\gamma(t-\tau)) - Q(\tilde{O_{t}},\theta_{t-\tau},\gamma(t-\tau))] \leq 2D(U_{r} + U_{v})U_{L}\vert A \vert L\sum_{i=t-\tau}^{t} E\Vert\theta_i - \theta_{t-\tau} \Vert. 
\end{align*}

\begin{sublemma}\label{Lemma:Gamma_second}
For any $t \geq 0$, %conditioned on $s_{t-\tau + 1},\theta_{t-\tau}$ and $\gamma(t-\tau)$, 
we have
\begin{align*}
    \vert E[Q(\tilde{O_{t}},\theta_{t-\tau},\gamma(t-\tau)) - Q(O_{t}^{'},\theta_{t-\tau},\gamma(t-\tau))|\theta_{t-\tau},\gamma(t-\tau),s_{t-\tau + 1}]\vert \leq 4D(U_{r} + U_{v})U_L bk^{\tau - 1}.
\end{align*} 
\end{sublemma}
\textit{Proof} Note that
\begin{align*}
    &E[Q(\tilde{O_{t}},\theta_{t-\tau},\gamma(t-\tau)) - Q(O_{t}^{'},\theta_{t-\tau},\gamma(t-\tau))|\theta_{t-\tau},\gamma(t-\tau),s_{t-\tau + 1}]\\
&= E[\langle \nabla L(\theta_{t-\tau},\gamma(t-\tau)),\bar{H}(\tilde{O_{t}},\theta_{t-\tau},\gamma(t-\tau))- \bar{H}(O_{t}^{'},\theta_{t-\tau},\gamma(t-\tau))\rangle|\theta_{t-\tau},\gamma(t-\tau),s_{t-\tau + 1}]\\
&= E[(\langle \nabla L(\theta_{t-\tau},\gamma(t-\tau)),\bar{H}(\tilde{O_{t}},\theta_{t-\tau},\gamma(t-\tau))\rangle - \langle \nabla L(\theta_{t-\tau},\gamma(t-\tau)) , \bar{H}(O_{t}^{'},\theta_{t-\tau},\gamma(t-\tau))\rangle)|\theta_{t-\tau},\gamma(t-\tau),s_{t-\tau + 1}]\\
&\leq 4D(U_{r} + U_{v})U_{L}d_{TV}(P(\tilde{O_{t}} = .|s_{t-\tau +1},\theta_{t-\tau}),\mu_{\theta_{t-\tau}} \otimes \pi_{\theta_{t-\tau}} \otimes P)\\
&\leq 4D(U_{r} + U_{v})U_{L} bk^{\tau - 1} 
\end{align*}
The last inequality comes from the proof of Lemma D.3  of \cite{fta_2_timescale}.

\vspace*{6pt}
We now continue with the remaining proof of Lemma~\ref{lemma:gamma}.

\vspace*{10pt}
\textit{Proof of Lemma \ref{lemma:gamma} (Contd.)}

We decompose $E[Q(O_t,\theta_t,\gamma(t))]$ as follows:
\begin{align*}
    E[Q(O_t,\theta_t,\gamma(t))] 
    &= E[Q(O_t,\theta_t,\gamma(t)) - Q(O_t,\theta_t,\gamma(t-\tau))] + E[Q(O_t,\theta_t,\gamma(t-\tau)) - Q(O_t,\theta_{t-\tau},\gamma(t-\tau))]\\ 
    &\qquad + E[Q(O_t,\theta_{t-\tau},\gamma(t-\tau)) - Q(\tilde{O_t},\theta_{t-\tau},\gamma(t-\tau))] + E[Q(\tilde{O_t},\theta_{t-\tau},\gamma(t-\tau)) - Q(O_t^{'},\theta_{t-\tau},\gamma(t-\tau))] \\
    &\qquad+ E[Q(O_t^{'},\theta_{t-\tau},\gamma(t-\tau))]
\end{align*}
where $\tilde{O_t} = (\tilde{s}_{t} , \tilde{a}_{t},\tilde{s}_{t+1})$  is from the auxiliary  Markov chain  and $O_t^{'} = (s_t,a_t,s_{t+1})$  is from the stationary distribution with $ s_t \sim \mu_{\theta_{t-\tau}},a_t \sim \pi_{\theta_{t-\tau}},s_{t+1} \sim p(s_t,.,a_t)$ which
actually satisfies $E[Q(O_t^{'},\theta_{t-\tau},\gamma(t-\tau))] = 0$.
By collecting the corresponding bounds from Lemmas \ref{Lemma:Gamma_gamma}--\ref{Lemma:Gamma_second}, we have
\begin{align*}
    E[Q(O_t,\theta_t,\gamma(t))] &\geq -T_1E|\gamma_m(t) - \gamma_m(t-\tau)| - T_2E\Vert \theta_t - \theta_{t-\tau} \Vert - 2D(U_{r} + U_{v})U_L\vert A \vert L\sum_{i=t-\tau}^{t} E\Vert\theta_i - \theta_{t-\tau} \Vert \\
    &\qquad- 4D(U_{r} + U_{v})U_{L} bk^{\tau - 1}\\ 
   &\geq -T_1\sum_{i=t-\tau +1 }^{ t} E\vert \gamma_m(i) - \gamma_m(i-1)\vert - T_2\sum_{i=t-\tau +1 }^{ t}E\Vert \theta_i - \theta_{i-1} \Vert\\ 
   &\qquad- 2D(U_{r} + U_{v})U_{L}\vert A \vert L\sum_{i=t-\tau +1 }^{t}\sum_{j=t-\tau +1 }^{ i}E\Vert\theta_j - \theta_{j-1} \Vert - 4D(U_{r} + U_{v})U_{L} bk^{\tau - 1} \\
   &\geq -T_1\sum_{i=t-\tau +1 }^{ t} E\vert \gamma_m(i) - \gamma_m(i-1)\vert - T_2\sum_{i=t-\tau +1 }^{ t}E\Vert \theta_i - \theta_{i-1} \Vert\\ 
   &\qquad- 2D(U_{r} + U_{v})U_{L}\vert A \vert L\tau\sum_{j=t-\tau +1 }^{ t}E\Vert\theta_j - \theta_{j-1} \Vert - 4D(U_{r} + U_{v})U_{L} bk^{\tau - 1} \\
   &\geq -(D_1(\tau + 1)\sum_{k=t-\tau + 1}^{t} E\Vert \theta_k - \theta_{k-1} \Vert + D_2bk^{\tau - 1} + T_1\sum_{i=t-\tau +1 }^{ t} E\vert \gamma_m(i) - \gamma_m(i-1)\vert),
\end{align*} 
where $D_1 := \max \{T_2 , 2D(U_{r} + U_{v})U_{L}\vert A \vert L\}$ and $D_2 := 4D(U_{r} + U_{v})U_{L}$. This 
completes the proof.
\end{proof}

\textit{Proof of Theorem 1}

Under the update rule of Algorithm 1 for the actor recursion, we have:
\begin{align*}
    \theta_{t+1} = \theta_t + b(t)\delta_{t}\Psi_{s_{t}a_{t}}.
\end{align*}
So, using lemma \ref{Lemma:L-smooth}, we have 
\begin{align*}
    L(\theta_{t+1},\gamma(t)) \geq L(\theta_t,\gamma(t)) + b(t) \langle \nabla L(\theta_t,\gamma(t)) ,\delta_{t}\nabla \log\pi_{\theta_t}(a_t|s_t) \rangle 
- M_{L}b(t)^2\Vert\delta_{t}\nabla \log\pi_{\theta_t}(a_t|s_t)\Vert^2.
\end{align*}
Now
\begin{align*}
 &\delta_{t}\nabla \log\pi_{\theta_t}(a_t|s_t) \\
 &= 
 ( q(t) + \sum_{k=1}^{N}\gamma_{k}(t)(h_k(t)-\alpha_k) - L_t + v_{t}^T(f(s_{t+1}) - f(s_t))\nabla \log\pi_{\theta_t}(a_t|s_t) \\
 &= (q(t) + \sum_{k=1}^{k=N}\gamma_{k}(t)(h_k(t)-\alpha_k) - L(\theta_t,\gamma(t)) + L(\theta_t,\gamma(t)) - L_t + (f(s_{t+1})^T - f(s_t)^T)(v_{t} - v_{t}^*) \\
 &\qquad + (f(s_{t+1})^T - f(s_t)^T)v_{t}^*)\nabla \log\pi_{\theta_t}(a_t|s_t) \\
 &= \Delta H(O_t,L_t,v_{t},\theta_t,\gamma(t)) + H(O_t,\theta_t,\gamma(t),q(t),h(t)),
\end{align*}
where $h(t) = (h_1(t),h_2(t),h_3(t),.....,h_N(t))$.
Hence,
\begin{align*}
    L(\theta_{t+1},\gamma(t)) 
    &\geq L(\theta_t,\gamma(t)) + b(t) \langle \nabla L(\theta_t,\gamma(t)) ,  \Delta H(O_t,L_t,v_{t},\theta_t,\gamma(t)) + H(O_t,\theta_t,\gamma(t),q(t),h(t))\rangle\\ 
    &\qquad  - M_{L}b(t)^2\Vert\delta_{t}\nabla \log\pi_{\theta_t}(a_t|s_t)\Vert^2\\
    &\geq L(\theta_t,\gamma(t)) + b(t) \langle \nabla L(\theta_t,\gamma(t)) ,  \Delta H(O_t,L_t,v_{t},\theta_t,\gamma(t)) \rangle \\
    &\qquad +b(t)\langle \nabla L(\theta_t,\gamma(t)) , H(O_t,\theta_t,\gamma(t),q(t),h(t)) - E_{O^{'},q(t),h(t)}[H(O^{'},\theta_t,\gamma(t),q(t),h(t))]\rangle\\ &\qquad +b(t)\langle \nabla L(\theta_t,\gamma(t)) ,  E_{O^{'},q(t),h(t)}[H(O^{'},\theta_t,\gamma(t),q(t),h(t))]\rangle - M_{L}b(t)^2\Vert\delta_{t}\nabla \log\pi_{\theta_t}(a_t|s_t)\Vert^2\\
   & \geq L(\theta_t,\gamma(t)) + b(t) \langle \nabla L(\theta_t,\gamma(t)) ,  \Delta H(O_t,L_t,v_{t},\theta_t,\gamma(t)) \rangle \\
   &\qquad+ b(t)\check{\Gamma}(O_t,\theta_t,\gamma(t),q(t),h(t)) + b(t)\Vert\nabla L(\theta_t,\gamma(t))\Vert^2 +  b(t)\langle \nabla L(\theta_t,\gamma(t)) ,  E_{O^{'}}[\Delta H^{'}(O^{'} , \theta_t,\gamma(t))]\rangle\\
    &\qquad- M_{L}b(t)^2\Vert\delta_{t}\nabla \log\pi_{\theta_t}(a_t|s_t)\Vert^2.
\end{align*}
In the first inequality, we discard the $1/2$ in front of the square norm term. After rearranging the terms, we obtain
\begin{align*}
     \Vert \nabla L(\theta_t,\gamma(t))\Vert^2 &\leq \frac{1}{b(t)} (L(\theta_{t+1},\gamma(t)) - L(\theta_t,\gamma(t))) -
\langle \nabla L(\theta_t,\gamma(t)) ,  \Delta H(O_t,L_t,v_t,\theta_t,\gamma(t)) \rangle \\
&\qquad - \check{\Gamma}(O_t,\theta_t,\gamma(t),q(t),h(t)) - \langle \nabla L(\theta_t,\gamma(t)) ,  E_{O^{'}}[\Delta H^{'}(O^{'} , \theta_t,\gamma(t)]\rangle \\ 
&\qquad+ M_{L}b(t)\Vert\delta_{t}\nabla \log\pi_{\theta_t}(a_t|s_t)\Vert^2.
\end{align*}
After taking expectations we have,
\begin{align}
\label{expectation:performance}
    E[\Vert\nabla L(\theta_t,\gamma(t))\Vert^2 ] &\leq \frac{1}{b(t)} E[(L(\theta_{t+1},\gamma(t)) - L(\theta_t,\gamma(t)))] -
E[\langle \nabla L(\theta_t,\gamma(t)) ,  \Delta H(O_t,L_t,v_t,\theta_t,\gamma(t)) \rangle]\notag\\\notag
&\qquad - E[\check{\Gamma}(O_t,\theta_t,\gamma(t),q(t),h(t))] - E[\langle \nabla L(\theta_t,\gamma(t)) ,  E_{O^{'}}[\Delta H^{'}(O^{'} , \theta_t,\gamma(t))]\rangle] \\\notag
&\qquad+ M_{L}b(t)E[||\delta_{t}\nabla \log\pi_{\theta_t}(a_t|s_t)||^2]. \hspace{4cm} \\
\end{align}
Now, observe that
\begin{align*}
    E[\langle \nabla L(\theta_t,\gamma(t)) ,  \Delta H(O_t,L_t,v_{t},\theta_t,\gamma(t)) \rangle] \geq -D\sqrt{E\Vert \nabla L(\theta_t,\gamma(t))\Vert^2}\sqrt{2E\Vert A_t\Vert^2 + 8E\Vert B_t\Vert^2},
\end{align*}
where
\begin{align*}
    A_t &= L_t - L(\theta_t,\gamma(t)),\\
B_t &=v_t - v^{*}(\theta_t,\gamma(t)), 
\end{align*}
 and the inequality follows from the Cauchy inequality and
Lemma \ref{lemma:del_h}.

Next, we have that
\begin{align*}
E[\check{\Gamma}(O_t,\theta_t,\gamma(t),q(t),h(t))] \geq  &-(D_1(\tau + 1)\sum_{k=t-\tau + 1}^{t} E\Vert \theta_k - \theta_{k-1} \Vert + D_2bk^{\tau - 1} + T_1\sum_{i=t-\tau +1 }^{ t} E\vert \gamma_m(i) - \gamma_m(i-1)\vert)\\
& \geq -(2D_1D(\tau + 1)(U_r + U_v)\sum_{k=t-\tau + 1}^{t-1} b(k) + D_2bk^{\tau - 1}\\
&\qquad+ T_1(U_\alpha + U_c)\sum_{k=t-\tau +1 }^{ t-1} c(k))
\end{align*}
Also, note that
\begin{align*}
    \langle \nabla L(\theta_t,\gamma(t)) ,  E_{O^{'}}[\Delta H^{'}(O^{'} , \theta_t,\gamma(t))]\rangle
     &\geq -U_{L}\sqrt{\Vert E_{O^{'}}[\Delta H^{'}(O^{'} , \theta_t,\gamma(t))]\Vert^2} \\
    &\geq  -U_{L}\sqrt{ E_{O^{'}}\Vert\Delta H^{'}(O^{'} , \theta_t,\gamma(t))\Vert^2} \\
&\geq -2DU_{L}\epsilon_{app}
\end{align*}
Now we return to the inequality in (\ref{expectation:performance}) and plug the above terms back in it to obtain
\begin{align*}
   E[\Vert\nabla L(\theta_t,\gamma(t))\Vert^2 ]& \leq \frac{1}{b(t)} E[(L(\theta_{t+1},\gamma(t)) - L(\theta_t,\gamma(t)))] +
D\sqrt{E\Vert \nabla L(\theta_t,\gamma(t))\Vert^2}\sqrt{2E\Vert A_t\Vert^2 + 8E\Vert B_t\Vert^2}\\
&\qquad +  2D_1D(\tau + 1)(U_r + U_v)\sum_{k=t-\tau }^{t-1} b(k) + D_2bk^{\tau - 1}\\
&\qquad+ T_1(U_\alpha + U_c)\sum_{k=t-\tau }^{ t-1} c(k) +  2DU_{L}\epsilon_{app}\\
&\qquad+  M_Lb(t)E[\Vert\delta_{t}\nabla \log\pi_{\theta_t}(a_t|s_t)\Vert^2].
\end{align*}
By setting $\tau = \tau_t$, we have,
\begin{align*}
    E[\Vert\nabla L(\theta_t,\gamma(t))\Vert^2 ] &\leq \frac{1}{b(t)} E[(L(\theta_{t+1},\gamma(t)) - L(\theta_t,\gamma(t)))] +
D\sqrt{E\Vert \nabla L(\theta_t,\gamma(t))\Vert^2}\sqrt{2E\Vert A_t\Vert^2 + 8E\Vert B_t\Vert^2}\\
&\qquad +  2D_1D(\tau_t + 1)^2(U_r + U_v)b(t-\tau_t) + D_2b(t) + 4M_LD^2(U_r + U_v)^2b(t)\\
&\qquad+ T_1(U_\alpha + U_c)(\tau_t + 1)c(t-\tau_t)+  2DU_{L}\epsilon_{app}\\
&\leq \frac{1}{b(t)} E[(L(\theta_{t+1},\gamma(t)) - L(\theta_t,\gamma(t)))] +
D\sqrt{E\Vert \nabla L(\theta_t,\gamma(t))\Vert^2}\sqrt{2E\Vert A_t\Vert^2 + 8E\Vert B_t\Vert^2}\\
&\qquad +  M_1(\tau_t + 1)^2b(t-\tau_t) + M_2b(t) + M_3(\tau_t + 1)c(t-\tau_t)+  2DU_{L}\epsilon_{\text{app}}.
\end{align*}
Summing the expectation from $\tau_t$ to $t$ we have,
\begin{align*}
    \sum_{k=\tau_t}^{t}E[\Vert\nabla L(\theta_k,\gamma(k))\Vert^2 ]  &\leq \underbrace{\sum_{k=\tau_t}^{t}\frac{1}{b(k)} E[(L(\theta_{k+1},\gamma(k)) - L(\theta_k,\gamma(k)))] }_{I_1}\\
&\qquad + \sum_{k=\tau_t}^{t}D\sqrt{E\Vert \nabla L(\theta_k,\gamma(k))\Vert^2}\sqrt{2E\Vert A_k\Vert^2 + 8E\Vert B_k\Vert^2}\\
&\qquad +  \underbrace{\sum_{k=\tau_t}^{t}( M_1(\tau_t + 1)^2b(k-\tau_t) + M_2b(k) + M_3(\tau_t + 1)c(k-\tau_t))}_{ I_2}\\
&\qquad +  2DU_{L}\epsilon_{\text{app}}(t-\tau_t + 1)
\end{align*} 
For the term $I_1$ above,
\begin{align*}
    &\sum_{k=\tau_t}^{t}\frac{1}{b(k)} E[(L(\theta_{k+1},\gamma(k)) - L(\theta_k,\gamma(k)))]\\ 
&= \sum_{k=\tau_t}^{t}\frac{1}{b(k)} E[(L(\theta_{k+1},\gamma(k)) - L(\theta_{k+1},\gamma(k+1))) + (L(\theta_{k+1},\gamma(k+1)) - L(\theta_k,\gamma(k)))] \\
&\leq \sum_{k=\tau_t}^{t}\frac{1}{b(k)}(U_c + U_\alpha) E[\sum_{m=1}^{N}\vert\gamma_m(k) - \gamma_m(k+1)\vert] \\
&\qquad+ \sum_{k=\tau_t}^{t}\big{(}\frac{1}{b(k-1)} - \frac{1}{b(k)}\big{)}E[L(\theta_k,\gamma(k))] - \frac{1}{b(\tau_t - 1)}E[L(\theta_{\tau_t},\gamma(\tau_t))] + \frac{1}{b(t)}E[L(\theta_{t+1}.\gamma(t+1))] \\
&\leq N(U_c + U_\alpha)^2\sum_{k=\tau_t}^{t}\frac{c(k)}{b(k)} +  \sum_{k=\tau_t}^{t}\big{(}\frac{1}{b(k-1)} - \frac{1}{b(k)}\big{)}U_r + \frac{1}{b(\tau_t - 1)}U_r + \frac{1}{b(t)}U_r\\
&=  N(U_c + U_\alpha)^2\sum_{k=\tau_t}^{t}\frac{c(k)}{b(k)} +U_r\big{[}  \sum_{k=\tau_t}^{t}\big{(}\frac{1}{b(k-1)} - \frac{1}{b(k)}\big{)} + \frac{1}{b(\tau_t - 1)} + \frac{1}{b(t)}\big{]}\\
&= N(U_c + U_\alpha)^2\frac{c_c}{c_b}\sum_{k=\tau_t}^{t}(1+k)^{\sigma - \beta} + 2U_rb(t)^{-1}\\
&\leq \frac{N(U_c + U_\alpha)^2c_c}{c_b(1+\beta - \sigma)}(t-\tau_t + 1)^{1 - \beta + \sigma} + 2\frac{U_r}{c_b}(1+t)^{\sigma}\\
&= B_1(t-\tau_t + 1)^{1 - \beta + \sigma} + B_2(1+t)^{\sigma}.
\end{align*}
The first inequality above holds because
\begin{align*}
    L(\theta_{k+1},\gamma(k)) - L(\theta_{k+1},\gamma(k+1)) &= \sum\limits_{s \in S}\mu_{\theta_{k+1}}(s)\sum\limits_{a \in A(s)}\pi_{\theta_{k+1}}(a|s) (\sum\limits_{m=1}^{N}(\gamma_m(k) -\gamma_m(k+1))(h_k(s,a) -\alpha_k)) \\
    &\leq (U_c + U_\alpha)\sum\limits_{m=1}^{N}\vert \gamma_m(k) -\gamma_m(k+1)\vert.
\end{align*}

Now, for the term $I_2$,
\begin{align*}
  &\sum_{k=\tau_t}^{t}( M_1(\tau_t + 1)^2b(k-\tau_t) + M_2b(k) + M_3(\tau_t + 1)c(k-\tau_t))\\
&\leq (M_1(\tau_t + 1)^2 + M_2)\sum_{k= 0}^{t - \tau_t}b(k) +M_3(\tau_t + 1) \sum_{k= 0}^{t - \tau_t}c(k)\\
&= (M_1(\tau_t + 1)^2 + M_2)c_b\sum_{k= 0}^{t - \tau_t}(1+k)^{-\sigma} + M_3(\tau_t + 1)c_c\sum_{k= 0}^{t - \tau_t}(1+k)^{-\beta}\\
&\leq \frac{(M_1(\tau_t + 1)^2 + M_2)c_b}{1-\sigma}(t -\tau_t + 1)^{1-\sigma} + \frac{M_3(\tau_t + 1)c_c}{1-\beta}(t -\tau_t + 1)^{1-\beta}\\
&\leq \big{(}\frac{(M_1(\tau_t + 1)^2 + M_2)c_b}{1-\sigma} + \frac{M_3(\tau_t + 1)^2c_c}{1-\beta}\big{)}(t -\tau_t + 1)^{1-\sigma}\\
&= B_3(\tau_t + 1)^2(t -\tau_t + 1)^{1-\sigma}.  
\end{align*}
The second inequality holds because
\begin{align*}
    \sum_{k= 0}^{t - \tau_t}(1+k)^{-\sigma} \leq \int_{0}^{t-\tau_t + 1}x^{-\sigma}dx = \frac{(t-\tau_t + 1)^{1-\sigma}}{(1-\sigma)}.
\end{align*}

After combining all the terms, we have
\begin{align*}
    \sum_{k=\tau_t}^{t}E[\Vert\nabla L(\theta_k,\gamma(k))\Vert^2 ]  &\leq B_1(t-\tau_t + 1)^{1 - \beta + \sigma} + B_2(1+t)^{\sigma} + B_3(\tau_t + 1)^2(t -\tau_t + 1)^{1-\sigma}\\
&\qquad + D\sqrt{\sum_{k=\tau_t}^{t}E\Vert \nabla L(\theta_k,\gamma(k))\Vert^2}\sqrt{2\sum_{k=\tau_t}^{t}E\Vert A_k\Vert^2 + 8\sum_{k=\tau_t}^{t}E\Vert B_k\Vert^2}\\ 
&\qquad+ 2DU_{L}\epsilon_{\text{app}}(t-\tau_t + 1).
\end{align*}
Dividing $(1 + t - \tau_t)$ on both sides and assuming $t \geq 2\tau_t - 1$, we can express the result as
\begin{align*}
    1/(1+t-\tau_{t})\sum_{k=\tau_{t}}^{t} E [\Vert\nabla L(\theta_k,\gamma(k))\Vert^2 ]
&\leq B_1(t-\tau_t + 1)^{ \sigma - \beta} + 2B_2(1+t)^{\sigma - 1 } + B_3(\tau_t + 1)^2(t -\tau_t + 1)^{-\sigma}\\
&+ D\sqrt{\frac{1}{ 1 + t - \tau_t}\sum_{k=\tau_t}^{t}E\Vert\nabla L(\theta_k,\gamma(k))\Vert^2}\sqrt{Z(t)} + 2DU_{L}\epsilon_{app},
\end{align*}
where,
\begin{align*}
    Z(t) = (2\sum_{k=\tau_t}^{t}E\Vert A_k\Vert^2 + 8\sum_{k=\tau_t}^{t}E\Vert B_k\Vert^2)/(1 + t -\tau_t).
\end{align*}
Let 
\begin{align*}
    F(t) = 1/(1+t-\tau_{t})\sum_{k=\tau_{t}}^{t} E [\Vert\nabla L(\theta_k,\gamma(k))\Vert^2 ].
\end{align*}
So, we have
\begin{align*}
    F(t) \leq \mathcal{O}(t^{\sigma - \beta}) +  \mathcal{O}((\log t)^2 t^{-\sigma }) + \mathcal{O}(\epsilon_{app}) + 2D\sqrt{F(t)}\sqrt{Z(t)},
\end{align*}
because $\tau_t = \mathcal{O}(\log t)$,
which gives
\begin{align*}
    (\sqrt{F(t)} - D\sqrt{Z(t)})^2 \leq \mathcal{O}(t^{\sigma - \beta}) + \mathcal{O}((\log t)^2 t^{-\sigma }) + \mathcal{O}(\epsilon_{app}) + D^2Z(t).
\end{align*}
Let 
\begin{align*}
    A(t) = \mathcal{O}(t^{\sigma - \beta}) + \mathcal{O}((\log t)^2 t^{-\sigma }) + \mathcal{O}(\epsilon_{app}).
\end{align*} 
\hfill\break
Thus, we have
\begin{align*}
   &(\sqrt{F(t)} - D\sqrt{Z(t)})^2 \leq A(t) + D^2Z(t)\\
&\Rightarrow \sqrt{F(t)} - D\sqrt{Z(t)} \leq  \sqrt{A(t)} + D\sqrt{Z(t)}\\
&\Rightarrow \sqrt{F(t)} \leq  \sqrt{A(t)} + 2D\sqrt{Z(t)}\\
&\Rightarrow F(t) \leq 2A(t) + 8D^2Z(t).
\end{align*}

The first and third implications hold because for a function $M(t) \leq Q(t) + R(t)$(with each positive), we have,
\begin{align*}
    \sqrt{M(t)} &\leq \sqrt{Q(t)} + \sqrt{R(t)}\\
    M(t)^2 &\leq 2Q(t)^2 + 2R(t)^2
\end{align*}
So finally we have the following:
\begin{align*}
    \min\limits_{0\leq k\leq t}E[\Vert\nabla L(\theta_k,\gamma(k))\Vert^2 ]  
    &\leq 1/(1+t-\tau_{t})\sum_{k=\tau_{t}}^{t} E [\Vert\nabla L(\theta_k,\gamma(k))\Vert^2 ]\\
    &= \mathcal{O}(t^{\sigma - \beta}) ) + \mathcal{O}((\log t)^2 t^{-\sigma }) + \mathcal{O}(\epsilon_\text{app}) + \mathcal{O}(Z(t)).
\end{align*}

\subsection{Proof of Theorem 2: Estimating the Average Reward for Constrained Actor critic}\label{thm2_avg_reward}

We define several notations to clarify the probabilistic dependency below.
\begin{align*}
\begin{split}
   O_t 
   &:= 
   (s_t,a_t,s_{t+1}),\\
   O 
   &:= 
   (s,a,s^{'}),\\
   L_t^{*}
   &:= 
   L(\theta_t,\gamma(t)),\\
   y_t
   &:= 
   L_t - L_t^{*},\\
  \hat{\Xi}(L_t,\theta_t,\gamma(t),q(t),h(t))
  &:= 
  y_t\left(q(t) + \sum_{k=1}^{N}\gamma_{k}(t)(h_k(t)-\alpha_k) - L_t^{*}\right).
\end{split}
\end{align*}
Before we proceed further, we first state and prove Lemmas 5 and 6 below that will be used in the proof of Theorem 2 . Moreover, the proof of Lemma 6 shall rely on Lemmas \ref{Lemma:Xi_1}--\ref{Lemma:Xi_5} that we also state and prove in the following. Finally, collecting all these results together, we shall obtain the claim for Theorem 2.

\begin{lemma}
\label{lemma: L_bound_theta_gamma}
For any $\theta_1 ,\theta_2 ,\gamma^1 = (\gamma^1_1,\gamma^1_2,....,\gamma^1_N)^T,\gamma^2 =(\gamma^2_1,\gamma^2_2,....,\gamma^2_N)^T$ with $0 \leq \gamma^i_j \leq M$, we have

\begin{align*}
    \vert L(\theta_1,\gamma_1) - L(\theta_2,\gamma_2)\vert \leq C_1 \Vert \theta_1 - \theta_2 \Vert + C_2\vert \gamma_1^p - \gamma_2^p \vert,
\end{align*}
where $C_1 = N(U_c + U_\alpha) , C_2 = 2U_r\vert A\vert L(1 + \lceil \log_k b^{-1} \rceil + \frac{1}{1 - k} )$
and $\vert \gamma_p^1 - \gamma_p^2 \vert  = \max\limits_{i = 1,2,..., N}|\gamma_i^1-\gamma_i^2|$.
\end{lemma}

\begin{proof}
    Note that
\begin{align*}
    \vert L(\theta_1,\gamma^1) - L(\theta_2,\gamma^2)\vert &\leq \vert L(\theta_1,\gamma^1) - L(\theta_1,\gamma^2)\vert + \vert L(\theta_1,\gamma^2) - L(\theta_2,\gamma^2)\vert\\
    &\leq\sum\limits_{s \in S}\mu_{\theta_1}(s)\sum\limits_{a \in A(s)}\pi_{\theta_1}(s,a)\sum_{k=1}^{N}\vert\gamma(k)^1 - \gamma(k)^2\vert\vert(g_k(s,a) -\alpha_k) \vert \\
    &\qquad + \vert \sum\limits_{s \in S}\mu_{\theta_1}(s)\sum\limits_{a \in A(s)}\pi_{\theta_1}(s,a)(d(s,a) + \sum_{k=1}^{N}\gamma(k)^2(h_k(s,a) - \alpha_k)) \\
    &\qquad-\sum\limits_{s \in S}\mu_{\theta_2}(s)\sum\limits_{a \in A(s)}\pi_{\theta_2}(s,a)(d(s,a) + \sum_{k=1}^{N}\gamma(k)^2(h_k(s,a) - \alpha_k))\vert\\
   & \leq N(U_c + U_\alpha)\vert \gamma_p^1 - \gamma_p^2 \vert + 2U_r\vert A\vert L(1 + \lceil \log_k b^{-1} \rceil + \frac{1}{1 - k} )\Vert \theta_1 - \theta_2 \Vert \\
    & \leq C_1\vert \gamma_p^1 - \gamma_p^2 \vert  + C_2\Vert \theta_1 - \theta_2 \Vert,
\end{align*}
where $C_1 = N(U_c + U_\alpha) , C_2 = 2U_r\vert A\vert L(1 + \lceil \log_k b^{-1} \rceil + \frac{1}{1 - k} ) $
and $\vert \gamma_p^1 - \gamma_p^2 \vert  = \max\limits_{i = 1,2,..., N}|\gamma_i^1-\gamma_i^2|$.

The third inequality is because of Lemma B.1 of \cite{fta_2_timescale}.
\end{proof}

\begin{lemma}\label{bound_Xi}
Given the definition of $\hat{\Xi}(L_t,\theta_t,\gamma(t),q(t),h(t))$, for any $t > 0$, we have
 \begin{align*}
     E[\hat{\Xi}(L_t,\theta_t,\gamma(t),q(t),h(t))] &\leq 6U_rN(U_c + U_\alpha)E\vert \gamma_p(t) - \gamma_p(t-\tau) \vert +8U_r\overline{C}E\Vert \theta_t - \theta_{t-\tau} \Vert +2U_rE\vert L_t - L_{t-\tau} \vert\\
     &\qquad+2U_r^2\vert A \vert L\sum_{i=t-\tau}^{t} E\Vert\theta_i - \theta_{t-\tau} \Vert+4U_r^2bk^{\tau - 1},
\end{align*}
where 
\begin{align*}
    \vert \gamma_p(t) - \gamma_p(t-\tau) \vert &= \max\limits_{i=1,2,...,N}\vert \gamma_i(t) - \gamma_i(t-\tau) \vert,\\
    \overline{C} &= U_r\vert A\vert L(1 + \lceil \log_k b^{-1} \rceil + \frac{1}{1 - k} ),\\
    t &\geq \tau \geq 0.
\end{align*}
\end{lemma}

\begin{proof}
We have
\begin{align*}
    E[\hat{\Xi}(L_t,\theta_t,\gamma(t),q(t),h(t))] &=E_{s_{t} \sim p,a_t \sim \pi_{\theta_t},s_{t+1} \sim p}[ E[ y_t(q(t) + \sum_{k=1}^{k=N}\gamma_{k}(t)(h_k(t)-\alpha_k) - L_t^{*})|s_t,a_t,s_{t+1}]]\\
    &=E[y_t(c(s_t,a_t,s_{t+1},\gamma(t)) - L_t^{*})]\\
    &=E[\Xi(O_t,L_t,\theta_t,\gamma(t))]
\end{align*}
where 
\begin{align*}
    c(s,a,s^{'},\gamma) &= \sum\limits_{q}(q\cdot \bar{p}(q|s,a,s^{'})) + \sum\limits_{k=1}^{k=N}\gamma_{k}(\sum\limits_{h}(h\cdot p_{k}(h|s,a,s^{'}) )- \alpha_k).
\end{align*}
The proof will be built on supporting lemmas \ref{Lemma:Xi_1}--\ref{Lemma:Xi_5}
that we first state and prove below.

\begin{sublemma} \label{Lemma:Xi_1}
For any $\theta,L,\gamma^1 = (\gamma^1_1,\gamma^1_2,....,\gamma^1_N)^T,\gamma^2 =(\gamma^2_1,\gamma^2_2,....,\gamma^2_N)^T,O = (s,a,s^{'})$ with $0 \leq \gamma^i_j \leq M$ for $i \in \{1,2\}$ and $j \in \{1,2,...,N\} $, we have
\begin{align*}
    \vert \Xi(O,L,\theta,\gamma^1) - \Xi(O,L,\theta,\gamma^2)\vert  \leq  6U_rN(U_c + U_\alpha)\vert \gamma_p^1 - \gamma_p^2 \vert,
\end{align*}
where 
\begin{align*}
    \vert \gamma_p^1 - \gamma_p^2 \vert = \max_{i = 1,2,..,N}\vert \gamma_i^1 - \gamma_i^2\vert.
\end{align*}
\end{sublemma}

\textit{Proof }
We have,
\begin{align*}
    \vert\Xi(O,L,\theta,\gamma^1) - \Xi(O,L,\theta,\gamma^2)\vert &= \vert(L-L(\theta,\gamma^1))(c(s,a,s^{'},\gamma^1) -L(\theta,\gamma^1)) - (L-L(\theta,\gamma^2))(c(s,a,s^{'},\gamma^2) -L(\theta,\gamma^2))\vert\\
    &\leq \vert(L-L(\theta,\gamma^1))(c(s,a,s^{'},\gamma^1) -c(s,a,s^{'},\gamma^2) + L(\theta,\gamma^2) -L(\theta,\gamma^1))\vert\\
    &\qquad + \vert(L(\theta,\gamma^2) -L(\theta,\gamma^1) )(c(s,a,s^{'},\gamma^2) -L(\theta,\gamma^2))\vert\\
    &\leq 2U_r(\vert c(s,a,s^{'},\gamma^1) -c(s,a,s^{'},\gamma^2) \vert + 2\vert  L(\theta,\gamma^2) -L(\theta,\gamma^1)\vert )\\
    &\leq 6U_rN(U_c + U_\alpha)\vert \gamma_p^1 - \gamma_p^2 \vert,
\end{align*}
where
\begin{align*}
    \vert \gamma_p^1 - \gamma_p^2 \vert = \max\limits_{i = 1,2,..., N}\vert \gamma_i^1 - \gamma_i^2 \vert.
\end{align*}

\begin{sublemma} \label{Lemma:Xi_2}
For any $L$,$\theta_1 ,\theta_2,O = (s,a,s^{'})$,$\gamma = (\gamma_1,\gamma_2,...,\gamma_N)^T$  with $0 \leq \gamma_i \leq M$ for $i \in {1,2,..,N}$, we have
\begin{align*}
    \vert \Xi(O,L,\theta_1,\gamma) - \Xi(O,L,\theta_2,\gamma)\vert  \leq  8U_r\overline{C}\Vert \theta_1 - \theta_2 \Vert,
\end{align*} 
where 
\begin{align*}
    \overline{C} = U_r\vert A\vert L(1 + \lceil \log_k b^{-1} \rceil + \frac{1}{1 - k} ). 
\end{align*}
\end{sublemma}

\textit{Proof}
By definition of $\Xi(O,L,\theta,\gamma)$, we have
\begin{align*}
    \vert \Xi(O,L,\theta_1,\gamma) - \Xi(O,L,\theta_2,\gamma) \vert &= \vert (L - L(\theta_1,\gamma))(C(s,a,s^{'},\gamma) 
- L(\theta_1,\gamma)) - (L - L(\theta_2,\gamma))(C(s,a,s^{'},\gamma) - L(\theta_2,\gamma))\vert\\
&\leq \vert (L - L(\theta_1,\gamma))(C(s,a,s^{'},\gamma) 
- L(\theta_1,\gamma))  - (L - L(\theta_1,\gamma))(C(s,a,s^{'},\gamma)
- L(\theta_2,\gamma)) \vert \\ 
&\qquad + \vert (L - L(\theta_1,\gamma))(C(s,a,s^{'},\gamma)
- L(\theta_2,\gamma)) - (L - L(\theta_2,\gamma))(C(s,a,s^{'},\gamma) - L(\theta_2,\gamma)) \vert\\
& \leq 4U_r\vert L(\theta_1,\gamma) - L(\theta_2,\gamma) \vert \\
& \leq 8U_r\overline{C}\Vert \theta_1 - \theta_2 \Vert,
\end{align*}
where 
$\overline{C} = U_r\vert A\vert L(1 + \lceil \log_k b^{-1} \rceil + \frac{1}{1 - k} )$.

\begin{sublemma} \label{Lemma:Xi_3}
For any $L_1,L_2,\theta,O = (s,a,s^{'}),\gamma = (\gamma_1,\gamma_2,...,\gamma_N)^T$  with $0 \leq \gamma_i \leq M$ for $i \in {1,2,..,N}$, we have
\begin{align*}
    \vert \Xi(O,L_1,\theta,\gamma) - \Xi(O,L_2,\theta,\gamma)\vert  \leq 2U_r\vert L_1 - L_2 \vert.
\end{align*} 
\end{sublemma}
\textit{Proof}
By definition,
\begin{align*}
    \vert \Xi(O,L_1,\theta,\gamma) - \Xi(O,L_2,\theta_,\gamma) \vert &= 
\vert (L_1 - L(\theta,\gamma))(C(s,a,s^{'},\gamma)  - L(\theta,\gamma)) - (L_2 - L(\theta,\gamma))(C(s,a,s^{'},\gamma)  - L(\theta,\gamma))\vert\\
&\leq 2U_r\vert L_1 - L_2 \vert.
\end{align*}
The claim follows.

\begin{sublemma} \label{Lemma:Xi_4}
Consider the tuples $O_t = (s_t,a_t,s_{t+1})$ and $\tilde{O_t} = (\tilde{s_t},\tilde{a_t},\tilde{s}_{t+1})$ of the original and auxiliary Markov chains respectively. Then the following holds:
\begin{align*}
    &\vert E[(\Xi(O_t,L_{t-\tau},\theta_{t-\tau},\gamma(t-\tau)) - \Xi(\tilde{O_t},L_{t-\tau},\theta_{t-\tau},\gamma(t-\tau)))|L_{t-\tau},\theta_{t-\tau},\gamma(t-\tau),s_{t-\tau + 1}] \vert\\
    &\qquad\leq 2U_r^2\vert A \vert L\sum_{i=t-\tau}^{t} E\Vert\theta_i - \theta_{t-\tau} \Vert
\end{align*}
\end{sublemma}

\textit{Proof}
By the Cauchy-Schwartz inequality and the definition of the total variation norm, we have
\begin{align*}
    &E[(\Xi(O_t,L_{t-\tau},\theta_{t-\tau},\gamma(t-\tau)) - \Xi(\tilde{O_t},L_{t-\tau},\theta_{t-\tau},\gamma(t-\tau)))|L_{t-\tau},\theta_{t-\tau},\gamma(t-\tau),s_{t-\tau + 1}]\\ 
    &= (L_{t-\tau} - L_{t-\tau}^*)E[(C(s_t,a_t,s_{t+1}\gamma(t-\tau)) - C(\tilde{s_t} , \tilde{a_t} ,\tilde{s}_{t+1}, \gamma(t-\tau)))|L_{t-\tau},\theta_{t-\tau},\gamma(t-\tau),s_{t-\tau + 1}].
\end{align*}
Now,
\begin{align*}
   &E[(C(s_t,a_t,s_{t+1}\gamma(t-\tau)) - C(\tilde{s_t} , \tilde{a_t} ,\tilde{s}_{t+1}, \gamma(t-\tau)))|L_{t-\tau},\theta_{t-\tau},\gamma(t-\tau),s_{t-\tau + 1}]\\
   &\qquad\leq 2U_rd_{TV}(P(O_t = \cdot |s_{t-\tau +1},\theta_{t-\tau}),(P(\tilde{O_{t}} = \cdot|s_{t-\tau +1},\theta_{t-\tau})).
\end{align*}

The following bound on the total variation norm has been shown in the proof of lemma D.2 of \cite{fta_2_timescale}: 
\begin{align*}
    d_{TV}(P(O_t = \cdot|s_{t-\tau +1},\theta_{t-\tau}),(P(\tilde{O_{t}} = \cdot|s_{t-\tau +1},\theta_{t-\tau})) \leq \frac{1}{2}\vert A \vert L\sum_{i=t-\tau}^{t} E\Vert\theta_i - \theta_{t-\tau} \Vert. 
\end{align*}
Plugging this bound above we have,
\begin{align*}
    \vert E[(\Xi(O_t,L_{t-\tau},\theta_{t-\tau},\gamma(t-\tau)) - \Xi(\tilde{O_t},L_{t-\tau},\theta_{t-\tau},\gamma(t-\tau)))|L_{t-\tau},\theta_{t-\tau},\gamma(t-\tau),s_{t-\tau + 1}] \vert \leq 2U_r^2\vert A \vert L\sum_{i=t-\tau}^{t} E\Vert\theta_i - \theta_{t-\tau} \Vert.
\end{align*}
The claim follows.

\begin{sublemma} \label{Lemma:Xi_5}
Conditioned on $s_{t-\tau +1 },\theta_{t-\tau},L_{t-\tau}, \gamma(t-\tau)$, we have 
\begin{align*}
    E[\Xi(\tilde{O_t},L_{t-\tau},\theta_{t-\tau},\gamma(t-\tau))|L_{t-\tau},\theta_{t-\tau},\gamma(t-\tau),s_{t-\tau + 1}] \leq 4U_r^2bk^{\tau - 1}
\end{align*}
\end{sublemma}
\textit{Proof}
The proof follows in a similar manner as Lemma D.7 of \cite{fta_2_timescale}.

After collecting the corresponding results from lemmas \ref{Lemma:Xi_1}--\ref{Lemma:Xi_5}, we have
\begin{align*}
      E[\hat{\Xi}(L_t,\theta_t,\gamma(t),q(t),h(t))] &\leq 6U_rN(U_c + U_\alpha)E\vert \gamma_p(t) - \gamma_p(t-\tau) \vert +8U_r\overline{C}E\Vert \theta_t - \theta_{t-\tau} \Vert +2U_rE\vert L_t - L_{t-\tau} \vert\\
     &\qquad+2U_r^2\vert A \vert L\sum_{i=t-\tau}^{t} E\Vert\theta_i - \theta_{t-\tau} \Vert+4U_r^2bk^{\tau - 1}.
\end{align*}
\end{proof}

\textit{Proof of Theorem 2: Estimating the Average Reward for Constrained Actor critic}

We have the following update rule in the algorithm that we now analyze:
\begin{align*}
    L_{t+1} = L_t + a(t)(q(t) + \sum_{k=1}^{N}\gamma_{k}(t)(h_k(t)-\alpha_k)) - L_t).
\end{align*}

\hfill\break
Unrolling the above, we obtain
\begin{align*}
    y_{t+1}^2 &= (L_{t+1} - L_{t+1}^{*})^2\\
& =\left(L_t + a(t)\left(q(t) + \sum_{k=1}^{k=N}\gamma_{k}(t)(h_k(t)-\alpha_k) - L_t\right) - L_{t+1}^{*}\right)^2 \\
&= \left(y_t + L_t^* - L_{t+1}^* + a(t)\left(q(t) + \sum_{k=1}^{k=N}\gamma_{k}(t)(h_k(t)-\alpha_k) - L_t\right)\right)^2\\
&= y_t^2 + 2a(t)y_t(C_t - L_t) + 2y_t(L_t^* - L_{t+1}^*) + (L_t^* - L_{t+1}^* + a(t)(C_t - L_t))^2\\
&\leq y_t^2 + 2a(t)y_t(C_t - L_t) + 2y_t(L_t^* - L_{t+1}^*) + 2(L_t^* - L_{t+1}^*)^2 + 2a(t)^2(C_t - L_t)^2\\
&= y_t^2 - 2a(t)y_t^2 + 2a(t)y_t^2+ 2a(t)y_t(C_t - L_t) + 2y_t(L_t^* - L_{t+1}^*) + 2(L_t^* - L_{t+1}^*)^2 + 2a(t)^2(C_t - L_t)^2\\
&= y_t^2 - 2a(t)y_t^2 +  2a(t)y_t(C_t - L_t + y_t) + 2y_t(L_t^* - L_{t+1}^*) + 2(L_t^* - L_{t+1}^*)^2 + 2a(t)^2(C_t - L_t)^2\\
&= (1-2a(t))y_t^2 + 2a(t)y_t(C_t - L_t^*) +  2y_t(L_t^* - L_{t+1}^*) + 2(L_t^* - L_{t+1}^*)^2 + 2a(t)^2(C_t - L_t)^2,
\end{align*}
where $C_t = q(t) + \sum_{k=1}^{N}\gamma_{k}(t)(h_k(t)-\alpha_k)$.The first inequality is due to $( x + y )^2 \leq 2x^2 + 2y^2$. Rearranging and summing from $\tau_t$ to $t$, we have
\begin{align*}
  \sum_{k = \tau_t}^{t}E[y_k^2] &\leq \underbrace{\sum_{k = \tau_t}^{t} \frac{1}{2a(k)}E(y_k^2 - y_{k+1}^2)}_{I_1} + \underbrace{\sum_{k=\tau_t}^{t}E[\hat{\Xi}(L_k,\theta_k,\gamma(k),q(k),h(k))] }_{I_2}\\
&\qquad+ \underbrace{\sum_{k = \tau_t}^{t} \frac{1}{a(k)}E[y_k(L_k^* - L_{k+1}^*)]}_{I_3} + \underbrace{\sum_{k = \tau_t}^{t} \frac{1}{a(k)}E[(L_k^* - L_{k+1}^*)^2]}_{I_4}\\
&\qquad+ \underbrace{\sum_{k = \tau_t}^{t}a(k)E[(C_k - L_k)^2]}_{I_5}.
\end{align*}
We now consider $I_1,\ldots,I_5$ term by term.
For $I_1$, we have
\begin{align*}
    I_1 &= \sum_{k = \tau_t}^{t} \frac{1}{2a(k)}E(y_k^2 - y_{k+1}^2)\\
&= \sum_{k = \tau_t}^{t}(\frac{1}{2a(k)} - \frac{1}{2a(k-1)})E[y_k^2] + \frac{1}{2a(\tau_t - 1)}E[y_{\tau_t}^2] - \frac{1}{2a(t)}E[y_{t+1}^2]\\
&\leq \frac{2U_r^2}{a(t)}.
\end{align*}

For $I_2$, from Lemma \ref{bound_Xi}, we have
\begin{align*}
E[\hat{\Xi}(L_t,\theta_t,\gamma(t),q(t),h(t))] &\leq 6U_rN(U_c + U_\alpha)E\vert \gamma_p(t) - \gamma_p(t-\tau) \vert +8U_r\overline{C}E\Vert \theta_t - \theta_{t-\tau} \Vert +2U_rE\vert L_t - L_{t-\tau} \vert\\
     &\qquad+2U_r^2\vert A \vert L\sum_{i=t-\tau}^{t} E\Vert\theta_i - \theta_{t-\tau} \Vert+4U_r^2bk^{\tau - 1}\\
&\leq 16U_r\overline{C}D(U_r + U_v)\tau b(t-\tau) + 6U_rN(U_c + U_\alpha)\tau(U_c + U_\alpha)c(t-\tau)+ 4U_r^2\tau a(t-\tau) \\
&\qquad+ 4DU_r^2(U_r + U_v)\vert A \vert L\tau(\tau +1 )b(t-\tau) + 4U_r^2bk^{\tau - 1}\\
&\leq B_1\tau^2b(t-\tau) + B_2\tau a(t-\tau) + B_3bk^{\tau - 1} + B_4\tau c(t-\tau). 
\end{align*}
By the choice of $\tau_t$,  we then have
\begin{align*}
   I_2 &= \sum_{k=\tau_t}^{t}E[\hat{\Xi}(L_t,\theta_t,\gamma(t),q(t),h(t))]\\
&\leq (B_1\tau_t^2 + B_3)\sum_{k=\tau_t}^{t}b(k-\tau_t)+ B_2\tau_t\sum_{k=\tau_t}^{t}a(t-\tau_t) + B_4\tau_t\sum_{k=\tau_t}^{t}c(t-\tau_t).
\end{align*}
For $I_3$, we have
\begin{align*}
    I_3 &\leq \left(\sum_{k = \tau_t}^{t}E[y_k^2]\right)^{1/2}\left(\sum_{k = \tau_t}^{t}E\left[\frac{(L_k^* - L_{k+1}^*)^2}{a(k)^2}\right]\right)^{1/2}\\
&\qquad \leq \left(\sum_{k = \tau_t}^{t}E[y_k^2]\right)^{1/2}\left(\sum_{k = \tau_t}^{t}E\left[\frac{(C_1 \Vert \theta_k - \theta_{k+1} \Vert + C_2\Vert \gamma(k)^p - \gamma(k+1)^p \Vert)^2}{a(k)^2}\right]\right)^{1/2}\\
&\qquad \leq \left(\sum_{k = \tau_t}^{t}E[y_k^2]\right)^{1/2}\left(\sum_{k = \tau_t}^{t}E\left[\frac{(2C_1D(U_r + U_v)b(k) + C_2(U_c + U_\alpha)c(k))^2}{a(k)^2}\right]\right)^{1/2}\\
&\qquad \leq \left(\sum_{k = \tau_t}^{t}E[y_k^2]\right)^{1/2}\left(\overline{K}^2\sum_{k = \tau_t}^{t}\frac{b(k)^2}{a(k)^2}\right)^{1/2},
\end{align*} 
where $\overline{K} = ((2C_1D(U_r + U_v) + (U_c + U_\alpha))$.

For $I_4$, we have
\begin{align*}
    I_4 &= \sum_{k = \tau_t}^{t} \frac{1}{a(k)}E[(L_k^* - L_{k+1}^*)^2]\\
&= \sum_{k = \tau_t}^{t} \frac{1}{a(k)}E[(L(\theta_k,\gamma(k)) - L(\theta_{k+1},\gamma(k+1)))^2]\\
&\leq \sum_{k = \tau_t}^{t} \frac{\overline{K}^2b(k)^2}{a(k)}\\
& = \mathcal{O}\left(\sum_{k = \tau_t}^{t}\frac{b(k)^2}{a(k)}\right).
\end{align*}
For $I_5$, we have
\begin{align*}
I_5 &= \sum_{k = \tau_t}^{t}a(k)E[(C_k - L_k)^2]\\
&\leq \sum_{k = \tau_t}^{t}4U_r^2a(k)\\
&= \mathcal{O}\left(\sum_{k = \tau_t}^{t}a(k)\right).
\end{align*}
Next, after combining $I_1,\ldots,I_5$, using the uniform ergodicity requirement (Assumption 3), the definition of $\tau_t$ and the relation between step-size and mixing time in Equation (4) of the main paper, we obtain the following:

\begin{align*}
    \sum_{k=\tau_t}^{t}E[y_k^2] &\leq \frac{2U_r^2}{c_a}(1+t)^{\omega} +  (B_1\tau_t^2 + B_3)c_b\sum_{k=\tau_t}^{t}(1+k-\tau_t)^{-\sigma}+ B_2c_a\tau_t\sum_{k=\tau_t}^{t}(1+k-\tau_t)^{-\omega}\\
&\qquad +B_4\tau_tc_c\sum_{k=\tau_t}^{t}(1+k-\tau_t)^{-\beta} + \overline{K}\frac{c_b}{c_a}(\sum_{k = \tau_t}^{t}E[y_k^2])^{1/2}(\sum_{k = \tau_t}^{t}(1+k)^{-2(\sigma - \omega)})^{1/2}\\
&\qquad + \overline{K}\frac{c_b^2}{c_a}\sum_{k = \tau_t}^{t} (1+k)^{\omega-2\sigma} + 4U_r^2c_a\sum_{k = \tau_t}^{t}(1+k)^{-\omega}\\
& \leq \frac{2U_r^2}{c_a}(1+t)^{\omega} + ((B_1\tau_t^2 + B_3)c_b + B_2c_a\tau_t + \overline{K}c_b^2 + 4U_r^2c_a + B_4\tau_tc_c)\sum_{k=\tau_t}^{t}(1+k-\tau_t)^{-\omega} \\
&\qquad+ \overline{K}\frac{c_b}{c_a}(\sum_{k = \tau_t}^{t}E[y_k^2])^{1/2}(\sum_{k = \tau_t}^{t}(1+k)^{-2(\sigma - \omega)})^{1/2}\\
& \leq \frac{2U_r^2}{c_a}(1+t)^{\omega} + ((B_1\tau_t^2 + B_3)c_b + B_2c_a\tau_t + \overline{K}c_b^2 + 4U_r^2c_a + B_4\tau_tc_c)\sum_{k=0}^{t -\tau_t}(1+k)^{-\omega}\\ 
&\qquad+  \overline{K}\frac{c_b}{c_a}(\sum_{k = \tau_t}^{t}E[y_k^2])^{1/2}(\sum_{k = \tau_t}^{t}(1+k)^{-2(\sigma - \omega)})^{1/2}\\
& \leq \frac{2U_r^2}{c_a}(1+t)^{\omega} + ((B_1\tau_t^2 + B_3)c_b + B_2c_a\tau_t + \overline{K}c_b^2 + 4U_r^2c_a + B_4\tau_tc_c)\frac{(t-\tau_t +1)^{1-\omega}}{1-\omega}\\
&\qquad+ \overline{K}\frac{c_b}{c_a}(\sum_{k = \tau_t}^{t}E[y_k^2])^{1/2}(\frac{(1+t-\tau_t)^{1-2(\sigma-\omega)}}{1-2(\sigma - \omega)})^{1/2}\\
\end{align*}
Note also that we have used above the precise form of the step-sizes as mentioned towards the end of Section 4.1 (main paper).
After applying the squaring technique (as in proof of Theorem \ref{thmm:actor_1}), we have:
\begin{align*}
    \sum_{k=\tau_t}^{t}E[y_k^2] &\leq \frac{4U_r^2}{c_a}(1+t)^{\omega}  + 2((B_1\tau_t^2 + B_3)c_b + B_2c_a\tau_t + \overline{K}c_b^2 + 4U_r^2c_a + B_4\tau_tc_c)\frac{(t-\tau_t +1)^{1-\omega}}{1-\omega}\\
    &\qquad \qquad+ 8\overline{K}^2\frac{c_b^2}{c_a^2}\frac{(1+t-\tau_t)^{1-2(\sigma-\omega)}}{1-2(\sigma - \omega)}\\
    & = \mathcal{O}(t^\omega) + \mathcal{O}(\log^2t \cdot t^{1-\omega}) + \mathcal{O}(t^{1-2(\sigma - \omega)}) .
\end{align*}

Dividing by $(1+t-\tau_t)$  and assuming $t \geq 2\tau_t - 1$, we have
\begin{align*}
    \sum_{k=\tau_t}^{t}E[y_k^2]/(1+t-\tau_t) = \mathcal{O}(t^{\omega-1}) + \mathcal{O}(\log^2t \cdot t^{-\omega}) + \mathcal{O}(t^{-2(\sigma - \omega)}).
\end{align*}

\subsection{Proof of Theorem 2: Estimating the convergence point of Critic for Constrained Actor Critic }\label{thm2_critic_converge}
\hfill\break
We first describe the notations used here.
\begin{align}
    \begin{split}
        O_t &:= (s_t,a_t,s_{t+1}),\\
        O &:= (s,a,s^{'}),\\
        v^*(t) &:= v^*(\theta_t,\gamma(t)),\\
         L_t^* &:= L(\theta_t,\gamma(t)),\\
         m_t &:= v_{t} - v^*(t),\\
         y_t &:= L_t - L_t^*,\\
        g(O,v,\theta,\gamma,q,h) &:= (q + \sum_{k=1}^{k=N}\gamma(k)(h_k-\alpha_k) - L(\theta,\gamma) + v^{T}(f_s^{'} - f_s))f_s,\\
        c(s,a,s^{'},\gamma) &:= \sum\limits_{q}(q\cdot \bar{p}(q|s,a,s^{'})) + \sum\limits_{k=1}^{k=N}\gamma(k)(\sum\limits_{h}(h\cdot p_{k}(h|s,a,s^{'}) )- \alpha_k), \\
        \overline{g}(v,\theta,\gamma) &:= E_{s \sim \mu_\theta,a \sim \pi_\theta,s^{'} \sim p}[(c(s,a,s^{'},\gamma) - L(\theta,\gamma) + (f_{s^{'} }- f_s)^Tv)f_s],\\
        \Lambda(O,v,\theta,\gamma,q,h) &:= \langle v - v^*(\theta,\gamma) , g(O,v,\theta,\gamma,q,h) - \overline{g}(v,\theta,\gamma) \rangle,\\
        \Delta g(O,L,\theta,\gamma) &:= (L(\theta,\gamma) - L)f_s.
   \end{split}
\end{align}
In the above,  $h = (h_1,h_2,h_3,...,h_N)$.
\begin{comment}
\begin{lemma}\label{lemma:four_bound}

For any $\theta,v,O = (s,a,s^{'}),\gamma$ 
\begin{align*}
&\Vert g(O,v,\theta,\gamma,q,h) \Vert \leq 2(U_r + U_v),\\
&\Vert \Delta g(O,L,\theta,\gamma) \Vert \leq 2U_r,\\
&\vert \Lambda(O,v,\theta,\gamma,q,h) \vert \leq 8U_v(U_r + U_v)^2.
\end{align*}
\end{lemma}
\begin{proof}
    \begin{align*}
   \Vert g(O,v,\theta,\gamma,q,h)\Vert  &= \Vert (q + \sum\limits_{k=1}^{k=N}\gamma(k)(h_k - \alpha_k) - L(\theta,\gamma) + v^{T}(f_s^{'} - f_s))f_s\Vert\\ 
   &\leq 2(U_r + U_v).
\end{align*}
Now,
\begin{align*}
    \Vert\Delta g(O,L,\theta,\gamma)\Vert &= \Vert(L(\theta,\gamma) - L)f_s\Vert\\
    &\leq \Vert L(\theta,\gamma) - L \Vert\\
    &\leq 2U_r.
\end{align*}
Further,
\begin{align*}
    \vert\Lambda(O,v,\theta,\gamma,q,h)\vert 
    &= \vert\langle v - v^*(\theta,\gamma) , g(O,v,\theta,\gamma,q,h) - \overline{g}(v,\theta,\gamma) \rangle\vert\\
    & \leq \Vert v - v^*(\theta,\gamma)  \Vert\Vert g(O,v,\theta,\gamma,q,h) - \overline{g}(v,\theta,\gamma) \Vert\\
    & \leq 2U_v\Vert g(O,v,\theta,\gamma,q,h) - \overline{g}(v,\theta,\gamma) \Vert\\
    & \leq 8U_v(U_r + U_v)^2.
\end{align*}
This completes the proof.
\end{proof}
\end{comment}

Before we proceed further, we first state and prove Lemma 7 below that will be used in the proof of Theorem 2(estimating the convergence point of critic) . 

\begin{lemma}
\label{bound_Lambda}
From the definition of $\Lambda(O_t,v_{t},\theta_t,\gamma(t),q(t),h(t))$, for any $0 \leq \tau \leq t$, we have
\begin{align*}
     E[\Lambda(O_t,v_t,\theta_t,\gamma(t),q(t),h(t))] &\leq C_1(\tau +1)E\Vert \theta_t - \theta_{t-\tau}\Vert + C_2 bk^{\tau - 1} + C_3E\Vert v_{t} - v(t-\tau) \Vert \\
&\qquad+ C_4E\vert  \gamma_m(t) - \gamma_m(t-\tau)\vert,
\end{align*}
where $C_1,C_2,C_3,C_4$ are positive constants and $\vert \gamma_m(t) - \gamma_m(t-\tau) \vert = \max\limits_{i=1,2,...,N}\vert \gamma_i(t) - \gamma_i(t-\tau) \vert$.
\end{lemma}
\begin{proof}
    We have,
    \begin{align*}
        E[\Lambda(O_t,v_t,\theta_t,\gamma(t),q(t),h(t))] = E[\langle v_t - v^*(\theta_t,\gamma(t)) , g(O_t,v_t,\theta_t,\gamma(t),q(t),h(t)) - \overline{g}(v_t,\theta_t,\gamma(t)) \rangle].
    \end{align*}
    Note now that $v_t,\theta_t,\gamma(t)$ do not depend on $q(t)$ and $h(t)$. Hence we can write,
    \begin{align*}
        &E[\langle v_t - v^*(\theta_t,\gamma(t)) , g(O_t,v_t,\theta_t,\gamma(t),q(t),h(t)) - \overline{g}(v_t,\theta_t,\gamma(t)) \rangle]\\
        &=E_{s_{t} \sim p,a_t \sim \pi_{\theta_t},s_{t+1} \sim p}[E[\langle v_t - v^*(\theta_t,\gamma(t)) , g(O_t,v_t,\theta_t,\gamma(t),q(t),h(t)) - \overline{g}(v_t,\theta_t,\gamma(t)) \rangle|s_t,a_t,s_{t+1}]]\\
        &= E[\langle v_t - v^*(\theta_t,\gamma(t)) , \check{g}(O_t,v_t,\theta_t,\gamma(t)) - \overline{g}(v_t,\theta_t,\gamma(t)) \rangle],
    \end{align*}
    where, 
    \begin{align*}
         \check{g}(O_t,v_t,\theta_t,\gamma(t)) &= (c(s_t,a_t,s_{t+1},\gamma(t)) - L(\theta,\gamma) + v^{T}(f_s^{'} - f_s))f_s,\\
         c(s_t,a_t,s_{t+1},\gamma(t)) &= \sum\limits_{q}(q\cdot \bar{p}(q|s_t,a_t,s_{t+1})) + \sum\limits_{k=1}^{k=N}\gamma_{k}(t)(\sum\limits_{h}(h\cdot p_{k}(h|s_t,a_t,s_{t+1}) )- \alpha_k).
    \end{align*}
    Let $\langle v_t - v^*(\theta_t,\gamma(t)) , \check{g}(O_t,v_t,\theta_t,\gamma(t)) - \overline{g}(v_t,\theta_t,\gamma(t)) \rangle = \bar{\Lambda}(O_t,v_t,\theta_t,\gamma(t))$. 
    Note that we can decompose $E[\bar{\Lambda}(O_t,v_t,\theta_t,\gamma(t))]$ as follows:
    \begin{align*}
        E[\bar{\Lambda}(O_t,v_t,\theta_t,\gamma(t))] &= \underbrace{E[\bar{\Lambda}(O_t,v_t,\theta_t,\gamma(t)) -\bar{\Lambda}(O_t,v_t,\theta_t,\gamma(t-\tau)) ]}_{I_1} + \underbrace{E[\bar{\Lambda}(O_t,v_t,\theta_t,\gamma(t-\tau)) -\bar{\Lambda}(O_t,v_t,\theta_{t-\tau},\gamma(t-\tau)) ]}_{I_2}\\
        &\qquad+ \underbrace{E[ \bar{\Lambda}(O_t,v_t,\theta_{t-\tau},\gamma(t-\tau)) - \bar{\Lambda}(O_t,v_{t-\tau},\theta_{t-\tau},\gamma(t-\tau)) ]}_{I_3}\\
        &\qquad + \underbrace{E[\bar{\Lambda}(O_t,v_{t-\tau},\theta_{t-\tau},\gamma(t-\tau)) - \bar{\Lambda}(\tilde{O_t},v_{t-\tau},\theta_{t-\tau},\gamma(t-\tau))]}_{I_4}\\
        &\qquad +\underbrace{E[\bar{\Lambda}(\tilde{O_t},v_{t-\tau},\theta_{t-\tau},\gamma(t-\tau))]}_{I_5}.
    \end{align*}

For term $I_1$,
\begin{align*}
    &\bar{\Lambda}(O_t,v_t,\theta_t,\gamma(t)) -\bar{\Lambda}(O_t,v_t,\theta_t,\gamma(t-\tau))\\
    &= \langle v_t - v^*(\theta_t,\gamma(t)) , \check{g}(O_t,v_t,\theta_t,\gamma(t)) - \overline{g}(v_t,\theta_t,\gamma(t)) \rangle\\
    &\qquad- \langle v_t - v^*(\theta_t,\gamma(t-\tau)) , \check{g}(O_t,v_t,\theta_t,\gamma(t-\tau)) - \overline{g}(v_t,\theta_t,\gamma(t-\tau)) \rangle\\
    &=\langle v_t - v^*(\theta_t,\gamma(t)) , \check{g}(O_t,v_t,\theta_t,\gamma(t)) -\check{g}(O_t,v_t,\theta_t,\gamma(t-\tau)) +\overline{g}(v_t,\theta_t,\gamma(t-\tau)) - \overline{g}(v_t,\theta_t,\gamma(t)) \rangle\\
    &\qquad + \langle  v^*(\theta_t,\gamma(t-\tau)) - v^*(\theta_t,\gamma(t))  , \check{g}(O_t,v_t,\theta_t,\gamma(t-\tau)) - \overline{g}(v_t,\theta_t,\gamma(t-\tau)) \rangle\\
    &\leq 8U_vN(U_c + U_\alpha)\vert \gamma_m(t) - \gamma_m(t-\tau) \vert +4L_2(U_r + U_v)\vert \gamma_m(t) - \gamma_m(t-\tau) \vert,
\end{align*}
where $\vert \gamma_m(t) - \gamma_m(t-\tau) \vert = \max\limits_{i=1,2,...,N}\vert \gamma_i(t) - \gamma_i(t-\tau) \vert$.

For the remaining terms $I_2--I_5$, exactly similar analysis as Lemmas D.8--D.11 of \cite{fta_2_timescale} can be carried out to obtain similar claims. For terms $I_4$ and $I_5$ we bound the expectation conditioned on $\theta_{t-\tau},\gamma(t-\tau),v_{t-\tau}$ and $s_{t-\tau +1}$. Hence, after combining all the terms we get
\begin{align*}
     E[\Lambda(O_t,v_t,\theta_t,\gamma(t),q(t),h(t))] &\leq C_1(\tau +1)E\Vert \theta_t - \theta_{t-\tau}\Vert + C_2 bk^{\tau - 1} + C_3E\Vert v_{t} - v(t-\tau) \Vert \\
&\qquad+ C_4E\vert  \gamma_m(t) - \gamma_m(t-\tau)\vert,
\end{align*}
where $C_1,C_2,C_3,C_4$ are positive constants.
\end{proof}

\textit{Proof of Theorem 2: Estimating the convergence point of Critic for Constrained Actor Critic} 

We use here the update rule of $v_{t}$ with projection. We shall assume here that the projection set $C$ is large enough so that $v^*(t+1)$ lies within the set $C$. If this is not the case, then the algorithm will practically converge to a point that is closest in $C$ to $v^*(t+1)$. We avoid such a case by assuming that $v^*(t+1)$ lies within $C$ itself. Recall also that the set $C$ is compact and convex which ensures that the point in $C$ to which the update with increment is projected to is the closest to it and is also unique. Thus, we obtain using the definition of $m_t$ described at the beginning of this section that
\begin{align*}
  \Vert m_{t+1} \Vert^2 &= \Vert \Gamma(v_{t} + a(t)\delta_{t}f_{s_t}) - v^*(t+1)\Vert^2\\
    &\leq \Vert v_{t} + a(t)\delta_{t}f_{s_t} - v^*(t+1)\Vert^2\\
    &= \Vert m_t + a(t)\delta_{t}f_{s_t} + v^*(t) - v^*(t+1) \Vert^2\\
    &= \Vert m_t + a(t)(q(t) + \sum_{k=1}^{N}\gamma_{k}(t)(h_k(t)-\alpha_k) - L_t + v_{t}^{T}(f_{s_{t+1}} - f_{s_t}))f_{s_t} + v^*(t) - v^*(t+1) \Vert^2\\
    &= \Vert m_t + a(t)(g(O_t,v_{t},\theta_t,\gamma(t),q(t),h(t)) + \Delta g(O_t,L_t,\theta_t,\gamma(t))) + v^*(t) - v^*(t+1) \Vert^2\\
    &= \Vert m_t \Vert^2 + 2a(t)\langle m_t,(g(O_t,v_{t},\theta_t,\gamma(t),q(t),h(t)) \rangle + 2a(t) \langle m_t ,\Delta g(O_t,L_t,\theta_t,\gamma(t)) \rangle\\
   &\qquad+ 2 \langle m_t,v^*(t) - v^*(t+1) \rangle \\
   &\qquad+ \Vert a(t)(g(O_t,v_{t},\theta_t,\gamma(t),q(t),h(t)) + \Delta g(O_t,L_t,\theta_t,\gamma(t))) + v^*(t) - v^*(t+1)\Vert^2\\
   & \leq \Vert m_t \Vert^2 + 2a(t)\langle m_t,(g(O_t,v_{t},\theta_t,\gamma(t),q(t),h(t)) \rangle + 2a(t) \langle m_t ,\Delta g(O_t,L_t,\theta_t,\gamma(t)) \rangle\\
   &\qquad + 2 \langle m_t,v^*(t) - v^*(t+1) \rangle \\
   &\qquad+ 2 a(t)^2\Vert (g(O_t,v_{t},\theta_t,\gamma(t),q(t),h(t)) + \Delta g(O_t,L_t,\theta_t,\gamma(t))) \Vert^2 + 2\Vert v^*(t) - v^*(t+1)\Vert^2\\
   & = \Vert m_t \Vert^2 + 2a(t)\langle m_t,\overline{g}(v_{t},\theta_t,\gamma(t)) \rangle + 2a(t)\Lambda(O_t,v_{t},\theta_t,\gamma(t),q(t),h(t))\\ 
   &\qquad + 2 a(t)\langle m_t , \Delta g(O_t,L_t,\theta_t,\gamma(t)) \rangle  + 2 \langle m_t,v^*(t) - v^*(t+1) \rangle \\
  &\qquad+ 2 a(t)^2\Vert (g(O_t,v_{t},\theta_t,\gamma(t),q(t),h(t)) + \Delta g(O_t,L_t,\theta_t,\gamma(t))) \Vert^2 + 2\Vert v^*(t) - v^*(t+1)\Vert^2\\
  & \leq \Vert m_t \Vert^2 + 2a(t)\langle m_t,\overline{g}(v_{t},\theta_t,\gamma(t)) \rangle + 2a(t)\Lambda(O_t,v_{t},\theta_t,\gamma(t),q(t),h(t)) \\
  &\qquad + 2 a(t)\langle m_t , \Delta g(O_t,L_t,\theta_t,\gamma(t)) \rangle  + 2 \langle m_t,v^*(t) - v^*(t+1) \rangle \\
  &\qquad+ 8 a(t)^2(U_r + U_v)^2+ 2\Vert v^*(t) - v^*(t+1)\Vert^2, 
\end{align*}
where the second inequality is due to $\Vert x + y \Vert^2 \leq 2\Vert x \Vert^2 + 2\Vert y \Vert^2$ and the third one is due to $\Vert (g(O_t,v_{t},\theta_t,\gamma(t),q(t),h(t)) + \Delta g(O_t,L_t,\theta_t,\gamma(t))) \Vert \leq 2(U_r + U_v)$. Now, as a consequence of Assumption 2, we have
\begin{align*}
    \langle m_t,\overline{g}(v_{t},\theta_t,\gamma(t)) \rangle &= \langle m_t,\overline{g}(v_{t},\theta_t,\gamma(t)) - \overline{g}(v^*(t),\theta_t,\gamma(t)) \rangle \\
    & = \langle m_t,E[(f_{s^{'} }- f_s)^T(v_{t} - v^*(t))f_s] \rangle \\
    & = m_t^TE[f_s(f_{s^{'} }- f_s)^T]m_t   \\
    &= m_t^TAm_t   \\
    &\leq -\lambda_e\Vert m_t \Vert^2,
\end{align*}
where the first equation is because of the equation in Section 4.1 of the main paper. Taking expectations up to $s_{t+1}$, we have
\begin{align*}
   E\Vert m_{t+1} \Vert^2 &\leq E\Vert m_t \Vert^2 + 2a(t)E\langle m_t,\overline{g}(v_{t},\theta_t,\gamma(t),) \rangle + 2a(t)E\Lambda(O_t,v_{t},\theta_t,\gamma(t),q(t),h(t))\\ 
   &\qquad + 2 a(t)E\langle m_t , \Delta g(O_t,L_t,\theta_t,\gamma(t)) \rangle  + 2 E\langle m_t,v^*(t) - v^*(t+1) \rangle \\
   &\qquad+ 8 a(t)^2(U_r + U_v)^2+ 2E\Vert v^*(t) - v^*(t+1)\Vert^2\\ 
   & \leq (1-2\lambda_ea(t))E\Vert m_t \Vert^2  \\
   &\qquad+ 2a(t)E\Lambda(O_t,v_{t},\theta_t,\gamma(t),q(t),h(t)) 
   + 2 a(t)E\langle m_t , \Delta g(O_t,L_t,\theta_t,\gamma(t)) \rangle  \\
   &\qquad+ 2 E\langle m_t,v^*(t) - v^*(t+1) \rangle\\
   &\qquad+ 8 a(t)^2(U_r + U_v)^2+ 2E\Vert v^*(t) - v^*(t+1)\Vert^2\\ 
    &\leq (1-2\lambda_ea(t))E\Vert m_t \Vert^2\\ 
    &\qquad+ 2a(t)E\Lambda(O_t,v_{t},\theta_t,\gamma(t),q(t),h(t)) 
    + 2 a(t)E \Vert m_t \Vert \vert y_t\vert\\
    &\qquad+ 2 E\Vert m_t\Vert(\Vert v^*(\theta_t,\gamma(t)) - v^*(\theta_t,\gamma(t+1))\Vert + \Vert v^*(\theta_t,\gamma(t+1)) - v^*(\theta_{t+1},\gamma(t+1)) \Vert) \\
    &\qquad+ 8 a(t)^2(U_r + U_v)^2\\
    &\qquad+ 4E[\Vert v^*(\theta_t,\gamma(t)) - v^*(\theta_t,\gamma(t+1))\Vert^2 + \Vert v^*(\theta_t,\gamma(t+1)) - v^*(\theta_{t+1},\gamma(t+1)) \Vert^2]\\
    &\leq (1-2\lambda_ea(t))E\Vert m_t \Vert^2  \\
    &\qquad+ 2a(t)E\Lambda(O_t,v_{t},\theta_t,\gamma(t),q(t),h(t)) 
    + 2 a(t)E \Vert m_t \Vert \vert y_t\vert\\
    &\qquad+ 2 E\Vert m_t\Vert(L_2\Vert \gamma_m(t) - \gamma_m(t+1)\Vert + L_1 \Vert \theta_t - \theta_{t+1} \Vert) \\
    &\qquad+ 8 a(t)^2(U_r + U_v)^2 + 4E[L_2^2\Vert \gamma_m(t) - \gamma_m(t+1)\Vert^2 + L_1^2 \Vert \theta_t - \theta_{t+1} \Vert^2] \\
    &\leq (1-2\lambda_ea(t))E\Vert m_t \Vert^2  \\
    &\qquad+ 2a(t)E\Lambda(O_t,v_{t},\theta_t,\gamma(t),q(t),h(t)) 
    + 2 a(t)E \Vert m_t \Vert \vert y_t\vert\\
    &\qquad+ 2 E\Vert m_t\Vert(L_2(U_c + U_\alpha)c(t) + 2DL_1(U_r + U_v)b(t) ) \\
    &\qquad+ 8 a(t)^2(U_r + U_v)^2 + 4E[L_2^2(U_c + U_\alpha)^2c(t)^2 + 4L_1^2D^2(U_r + U_v)^2b(t)^2 ]\\
   &\leq (1-2\lambda_ea(t))E\Vert m_t \Vert^2  \\
   &\qquad + 2a(t)E[\Lambda(O_t,v_{t},\theta_t,\gamma(t),q(t),h(t))]
  + 2 a(t)E \Vert m_t \Vert \vert y_t\vert\\
   &\qquad+ 2((L_2(U_c + U_\alpha) + 2DL_{1}(U_r + U_v))b(t) )  E\Vert m_t\Vert\\
   &\qquad+ 8 a(t)^2((U_r + U_v)^2 + 4L_2^2(U_c + U_\alpha)^2 + 16L_{1}^2D^2(U_r + U_v)^2) ).
\end{align*}
Rearranging now the terms in the inequality results in the following:
\begin{align*}
   2\lambda_e E\Vert m_t \Vert^2 &\leq \frac{1}{a(t)}\big{(}E\Vert m_t \Vert^2  - E\Vert m_{t+1} \Vert^2 \big{)} + 2E\Lambda(O_t,v_{t},\theta_t,\gamma(t),q(t),h(t)) \\
   &\qquad+  2 E \Vert m_t \Vert \vert y_t\vert+ B_q\frac{b(t)}{a(t)} )  E\Vert m_t\Vert + C_qa(t)\\
   &\leq \frac{1}{a(t)}\big{(}E\Vert m_t \Vert^2  - E\Vert m_{t+1} \Vert^2 \big{)} +2E\Lambda(O_t,v_{t},\theta_t,\gamma(t),q(t),h(t))  \\
   &\qquad + 2\sqrt{E \Vert m_t \Vert^2}\sqrt{ E y_t^2}+ B_q\frac{b(t)}{a(t)}  \sqrt{E \Vert m_t \Vert^2}+ C_qa(t), 
\end{align*}
where
\begin{align*}
    B_q &= 2((L_2(U_c + U_\alpha) + 2DL_{1}(U_r + U_v)), \\
    C_q &= 8 ((U_r + U_v)^2 + 4L_2^2(U_c + U_\alpha)^2 + 16L_{1}^2D^2(U_r + U_v)^2) ).
\end{align*}
Now,
\begin{align}\label{inequality:critic}
    2\lambda_e \sum_{k = \tau_t}^{t}E\Vert m_k \Vert^2 &\leq \underbrace{\sum_{k = \tau_t}^{t}\frac{1}{a(k)}\big{(}E\Vert m_k \Vert^2  - E\Vert m_{k+1} \Vert^2 \big{)} }_{I_1}\notag\\
    &\qquad + \underbrace{2\sum_{k = \tau_t}^{t}E\Lambda(O_k,v_{k},\theta_k,\gamma(k),q(k),h(k))}_{I_2} + \underbrace{2\sum_{k = \tau_t}^{t}\sqrt{E \Vert m_k \Vert^2}\sqrt{ E [y_k^2]}}_{I_3}\notag\\
   &\qquad+ \underbrace{B_q\sum_{k = \tau_t}^{t}\frac{b(k)}{a(k)} ) \sqrt{E \Vert m_k \Vert^2}}_{I_4}+ \underbrace{C_q\sum_{k = \tau_t}^{t}a(k)}_{I_5}.
\end{align}
\hfill\break
Consider now the term $I_1$. We have the following:
\begin{align*}
    I_1 &:= \sum_{k = \tau_t}^{t}\frac{1}{a(k)}\big{(}E\Vert m_k \Vert^2  - E\Vert m_{k+1} \Vert^2 \big{)}\\
   & =\sum_{k = \tau_t}^{t}\big{(}\frac{1}{a(k)} - \frac{1}{a(k-1)}\big{)}E\Vert m_k \Vert^2 + \frac{1}{a(\tau_t - 1)}E\Vert m_{\tau_t} \Vert^2 - \frac{1}{a(t)}E\Vert m_{t+1} \Vert^2\\
   & \leq \sum_{k = \tau_t}^{t}\big{(}\frac{1}{a(k)} - \frac{1}{a(k-1)}\big{)}E\Vert m_k \Vert^2 + \frac{1}{a(\tau_t - 1)}E\Vert m_{\tau_t} \Vert^2\\
   & \leq 4U_v^2\big{(} \sum_{k = \tau_t}^{t}\big{(}\frac{1}{a(k)} - \frac{1}{a(k-1)} \big{)} +  \frac{1}{a(\tau_t - 1)}\big{)}\\
   & = 4U_v^2\frac{1}{a(t)} = 4\frac{U_v^2}{c_a}(1+t)^{\omega} = \mathcal{O}(t^{\omega}).
\end{align*} 
For the term $I_2$, note that  
\begin{align*}
    I_2 &= 2\sum_{k = \tau_t}^{t}E[\Lambda(O_k,v_k,\theta_k,\gamma(k),q(k),h(k))]\\
    &\leq 4DC_1(U_r + U_v)(\tau_t +1)^2\sum_{k=\tau_t}^{t}b(k-\tau_t)+ 2C_2\sum_{k=\tau_t}^{t}b(t) \\
    &\qquad+ 4C_3(U_r + U_v)\tau_t\sum_{k=0}^{t-\tau_t}a(k)
    + 2C_4(U_c + U_\alpha)\tau_t\sum_{k=0}^{t-\tau_t}c(k)\\
    & \leq 4DC_1(U_r + U_v)(\tau_t +1)^2\sum_{k = 0}^{t-\tau_t}b(k)+ 2C_2(t-\tau_t + 1)b(t) \\
    &\qquad + (4C_3(U_r + U_v)\tau_t + 2C_4(U_c + U_\alpha)\tau_t)\sum_{k=0}^{t-\tau_t}a(k)\\  
    & \leq 4DC_1(U_r + U_v)(\tau_t +1)^2c_b\frac{(1+t-\tau_t)^{1-\sigma}}{1-\sigma} + 2C_2(t-\tau_t + 1)c_b(1+t)^{-\sigma}\\
    &\qquad+ (4C_3(U_r + U_v)\tau_t + 2C_4(U_c + U_\alpha)\tau_t)c_a\frac{(1+t-\tau_t)^{1-\omega}}{1-\omega}\\
   &\leq \big{[}\frac{4DC_1(U_r + U_v)(\tau_t +1)^2c_b}{1-\sigma} + 2C_2c_b + \frac{(4C_3(U_r + U_v)\tau_t + 2C_4(U_c + U_\alpha)\tau_t)c_a}{1-\omega}\big{]}(1+t)^{1-\omega}\\
   &= \mathcal{O}((\log t)^2t^{1-\omega}).
\end{align*}
Now, we get the following inequalities for the terms $I_3$, $I_4$ and $I_5$, respectively: 
\begin{align*}
    I_3 &:= 2\sum_{k = \tau_t}^{t}\sqrt{E \Vert m_k \Vert^2}\sqrt{ E y_k^2} \leq 2\big{(} \sum_{k = \tau_t}^{t }E  y_k^2\big{)}^{1/2}\big{(} \sum_{k = \tau_t}^{t}E \Vert m_k \Vert^2\big{)}^{1/2},\\
    I_4 &:= B_q\sum_{k = \tau_t}^{t}\frac{b(k)}{a(k)} ) \sqrt{E \Vert m_k \Vert^2} \leq  \big{(} \sum_{k = 0}^{t - \tau_t}\frac{b(k)^2}{a(k)^2}\big{)}^{1/2}\big{(} \sum_{k = \tau_t}^{t}E \Vert m_k \Vert^2\big{)}^{1/2}, \\
    I_5 &:= C_q\sum_{k = \tau_t}^{t}a(k) \leq C_qc_a(1+t)^{1-\omega}/(1-\omega).
\end{align*} 
Combining all the terms, we obtain
\begin{align*}
    2\lambda_e \sum_{k = \tau_t}^{t}E\Vert m_k \Vert^2
     &\leq \frac{4U_v^2}{c_a}(1+t)^{\omega}\\
     &\qquad+  \big{[}\frac{4DC_1(U_r + U_v)(\tau_t +1)^2c_b}{1-\sigma} + 2C_2c_b + \frac{(4C_3(U_r + U_v)\tau_t + 2C_4(U_c + U_\alpha)\tau_t + C_q)c_a}{1-\omega}\big{]}(1+t)^{1-\omega}\\
     &\qquad + \big{(} \sum_{k = \tau_t}^{t}E  y_k^2\big{)}^{1/2}\big{(} \sum_{k = \tau_t}^{t}E \Vert m_k \Vert^2\big{)}^{1/2} + \big{(} \sum_{k = 0}^{t - \tau_t}\frac{b(k)^2}{a(k)^2}\big{)}^{1/2}\big{(} \sum_{k = \tau_t}^{t}E \Vert m_k \Vert^2\big{)}^{1/2} .
\end{align*}
We assume $t \geq 2\tau_t - 1$. After substituting the value of $y_k$ and applying the squaring technique as in the proof of Theorem 1, we obtain
\begin{align*}
    \bigg(\sum_{k=\tau_t}^{t} E\Vert v_{k} - v^*(k) \Vert^2 \bigg)/(1+t-\tau_t)=\mathcal{O}\bigg(\frac{1}{t^{1-\omega}}\bigg) + \mathcal{O}\bigg(\frac{\log t}{t^{\omega}}\bigg) + \mathcal{O}\bigg(\frac{1}{t^{2(\sigma - \omega)}}\bigg).
\end{align*}

\begin{remark}
It is important to mention here that the requirement that $v^*_{t+1}$ lies within the projection region $C$ has also been made by \cite{fta_2_timescale} except however that they assume that the set $C$ is a ball of some radius $R_w$. We do not assume any such structure on the set $C$ except that it be compact and convex which suffices for our purpose.
\end{remark}

\subsection{Proof of Corollary 1}
Note that we have the following result from Theorem 1:
\begin{align}\label{thm_1_result}
    \min\limits_{0\leq k\leq t}E[\Vert\nabla L(\theta_k,\gamma(k))\Vert^2 ]  = \mathcal{O}(t^{\sigma - \beta}) ) + \mathcal{O}((\log t)^2 t^{-\sigma }) + \mathcal{O}(\epsilon_{app}) + \mathcal{O}(\varepsilon(t)),
\end{align}
where,
\begin{align*}
    \varepsilon(t) &= (2\sum_{k=\tau_t}^{t}E\Vert A_k\Vert^2 + 8\sum_{k=\tau_t}^{t}E\Vert B_k\Vert^2)/(1 + t -\tau_t), \\
     A_k  &=  L_k - L(\theta_k,\gamma(k)),\\
     B_k &= v_k - v(\theta_k,\gamma(k)).
\end{align*}
Now, from the results of Theorem 2, we have
\begin{align*}
    \varepsilon(t) = \mathcal{O}(t^{\omega-1}) + \mathcal{O}(\log t \cdot t^{-\omega}) + \mathcal{O}(t^{-2(\sigma - \omega)}).
\end{align*}
Substituting the above in (\ref{thm_1_result}), we have
\begin{align*}
    \min\limits_{0\leq k\leq t}E[\Vert\nabla L(\theta_k,\gamma(k))\Vert^2 ]   &=  \mathcal{O}(t^{\sigma - \beta})    + \mathcal{O}(\log^2t \cdot t^{-\sigma})+  \mathcal{O}(t^{\omega-1}) + \mathcal{O}(\log t \cdot t^{-\omega}) + \mathcal{O}(t^{-2(\sigma - \omega)}) + \mathcal{O}(\epsilon_{app})\\
    &= \mathcal{O}(t^{\sigma - \beta})    + \mathcal{O}(\log^2t \cdot t^{-\omega})+  \mathcal{O}(t^{\omega-1})  + \mathcal{O}(t^{-2(\sigma - \omega)}) + \mathcal{O}(\epsilon_{app})\\
    &= \mathcal{O}(t^{\sigma - \beta}) + \mathcal{O}(\log^2t \cdot t^{-\omega}) + \mathcal{O}(t^{-2(\sigma - \omega)}) + \mathcal{O}(\epsilon_{app}).
\end{align*}
The second equality holds because $\omega < \sigma$ while the third equality is true because $\sigma - \beta > \omega - 1$.
Optimising over the choice of $\omega, \sigma, \beta$, we obtain $\omega = 0.4$, $\sigma = 0.6$ and $\beta = 1$. Hence,
\begin{align*}
    \min\limits_{0\leq k\leq t}E[\Vert\nabla L(\theta_k,\gamma(k))\Vert^2 ]   = \mathcal{O}(\log^2 t \cdot t^{-0.4}) +\mathcal{O}(\epsilon_{app}).
\end{align*}
Therefore, in order to obtain an $\epsilon$-approximate (ignoring the approximation error as with \cite{fta_2_timescale}) stationary point of the performance function $L(\theta,\gamma)$, namely,
\begin{align*}
    \min\limits_{0\leq k\leq T}E[\Vert\nabla L(\theta_k,\gamma(k))\Vert^2 ]  = \mathcal{O}(\log^2 T \cdot T^{-0.4}) +\mathcal{O}(\epsilon_{app})  \leq \mathcal{O}(\epsilon_{app}) + \epsilon,
\end{align*}
we need to set $T = \tilde{\mathcal{O}}(\epsilon^{-2.5})$.

\subsection{Proof of Theorem 3}

We use the following notation here.
\begin{align*}
  \zeta(O,\theta,\gamma,q,h,G) =  \langle \nabla L(\theta,\gamma) , G^{-1}H(O,\theta,\gamma,q,h) E_{O^{'},q,h}[G^{-1}H(O^{'},\theta,\gamma,q,h)]\rangle,
\end{align*}
where $H(\cdot)$ has been defined in the proof of Theorem 1. Further, $O^{'} = (s,a,s^{'})$ denotes the independent sample $s \sim \mu_{\theta},$ $a \sim \pi_{\theta}$, $s^{'} \sim p(s,.,a)$. Hence, $E_{O^{'},q,h}[\cdot]$ denotes the expectation w.r.t.~the joint distribution of $s \sim \mu_{\theta},$ $a \sim \pi_{\theta}$, $s^{'} \sim p(s,.,a)$, $q \sim \bar{p}(.|s,a,s^{'})$, $h_{i} \sim p_{i}(.|s,a,s^{'})$, $i=1,\ldots,N$. The remaining notations are the same as those used in the proof of Theorem 1.

Now we will state and prove Lemma 8 below that will be used in the proof of Theorem 3 . Moreover, the proof of Lemma 8 shall rely on Lemmas \ref{Lemma:zeta_G}--\ref{Lemma:zeta_second} that we also state and prove in the following. Finally, collecting all these results together, we shall obtain the claim for Theorem 3.

\begin{lemma} \label{lemma:zeta}
For any $t \geq 0$,
\begin{align*}    E[\zeta(O_t,\theta_t,\gamma(t),q(t),h(t),G(t))] &\geq -(D_1(\tau + 1)\sum_{k=t-\tau + 1}^{t} E[\Vert \theta_k - \theta_{k-1} \Vert] + D_2bk^{\tau - 1} + T_1\sum_{i=t-\tau +1 }^{ t} E[\vert \gamma_m(i) - \gamma_m(i-1)\vert]\\
   &\qquad + T_{G}\sum_{i=t-\tau +1 }^{ t} E\Vert G(i)^{-1} - G(i-1)^{-1}\Vert),
\end{align*}
\end{lemma}
where $D_1,D_2,T_1,T_G$ are positive constants and   $t \geq \tau \geq 0$.

\begin{proof}
 We have 
\begin{align*}
&E[\zeta(O_t,\theta_t,\gamma(t),q(t),h(t),G(t))]\\
    &=
E_{s_{t} \sim p,a_t \sim \pi_{\theta_t},s_{t+1} \sim p}[E[\langle \nabla L(\theta_t,\gamma(t)) , G(t)^{-1}H(O_t,\theta_t,\gamma(t),q(t),h(t))\\
&\qquad- E_{O^{'},q,h}[G(t)^{-1}H(O^{'},\theta_t,\gamma(t),q,h)]\rangle|s_t,a_t,s_{t+1}]]\\
&= E[\langle \nabla L(\theta_t,\gamma(t)) , G(t)^{-1}\bar{H}(O_t,\theta_t,\gamma(t)) - E_{O^{'},q,h}[G(t)^{-1}H(O^{'},\theta_t,\gamma(t),q,h)]\rangle]\\
&=  E[\langle \nabla L(\theta_t,\gamma(t)) , G(t)^{-1}\bar{H}(O_t,\theta_t,\gamma(t)) - E_{O^{'}}[E_{q,h}[G(t)^{-1}H(O^{'},\theta_t,\gamma(t),q,h)|s,a,s^{'}]]\rangle]\\
& = E[\langle \nabla L(\theta_t,\gamma(t)) , G(t)^{-1}\bar{H}(O_t,\theta_t,\gamma(t)) - E_{O^{'}}[G(t)^{-1}\bar{H}(O^{'},\theta_t,\gamma(t))]\rangle]\\
&= E[\hat{Q}(O_t,\theta_t,\gamma(t),G(t))],
\end{align*}
where 
\begin{align*}
    \bar{H}(O,\theta,\gamma) &= (c(s,a,s^{'},\gamma) -L(\theta,\gamma)  +({f_{s^{'}}}^{T} - {f_{s}}^{T})v^{*}(\theta,\gamma))\nabla \log\pi_{\theta}(a|s),\\
    c(s,a,s^{'},\gamma) &= \sum\limits_{q}(q\cdot \bar{p}(q|s,a,s^{'})) + \sum\limits_{k=1}^{k=N}\gamma_{k}(\sum\limits_{h}(h\cdot p_{k}(h|s,a,s^{'}) )- \alpha_k).
\end{align*}
The proof makes use of the supporting lemmas \ref{Lemma:zeta_G}--\ref{Lemma:zeta_second} below.
\begin{sublemma}\label{Lemma:zeta_G}
For any $t \geq 0$,
\begin{align*}
    \vert \hat{Q}(O_t,\theta_t,\gamma(t),G(t)) - \hat{Q}(O_t,\theta_t,\gamma(t),G(t-\tau))\vert \leq T_{G}\Vert G(t)^{-1} - G(t-\tau)^{-1}\Vert,
\end{align*}
for some $T_G > 0$.
\end{sublemma}

\textit{Proof}
The following holds:
\begin{align*}
     &\hat{Q}(O_t,\theta_t,\gamma(t),G(t)) - \hat{Q}(O_t,\theta_t,\gamma(t),G(t-\tau))\\
     &= \langle \nabla L(\theta_t,\gamma(t)) , G(t)^{-1}\bar{H}(O_t,\theta_t,\gamma(t)) - E_{O^{'}}[G(t)^{-1}\bar{H}(O^{'},\theta_t,\gamma(t)) ]\rangle\\
     &\qquad - \langle \nabla L(\theta_t,\gamma(t)) , G(t-\tau)^{-1}\bar{H}(O_t,\theta_t,\gamma(t)) - E_{O^{'}}[G(t-\tau)^{-1}\bar{H}(O^{'},\theta_t,\gamma(t)) ]\rangle\\
     &= \langle \nabla L(\theta_t,\gamma(t)) ,(G(t)^{-1} -G(t-\tau)^{-1})\bar{H}(O_t,\theta_t,\gamma(t)) - E_{O^{'}}[(G(t)^{-1} -G(t-\tau)^{-1})\bar{H}(O^{'},\theta_t,\gamma(t)) ]\rangle\\
     &\leq 2D(U_r + U_v)\Vert \nabla L(\theta_t,\gamma(t))\Vert \Vert G(t)^{-1} -G(t-\tau)^{-1}\Vert\\
     &\leq 2D(U_r + U_v)U_{L} \Vert G(t)^{-1} -G(t-\tau)^{-1}\Vert.
\end{align*}
The claim follows by letting $T_G = 2D(U_r + U_v)U_{L} >0$. 
\begin{sublemma}\label{Lemma:zeta_gamma}
For any $t \geq 0$,
\begin{align*}
    \vert \hat{Q}(O_t,\theta_t,\gamma(t),G(t-\tau)) - \hat{Q}(O_t,\theta_t,\gamma(t-\tau),G(t-\tau))\vert \leq T_1\vert\gamma_{m}(t) - \gamma_{m}(t-\tau)\vert
\end{align*}
for some $T_1 > 0$.
\end{sublemma}

\textit{Proof}
Denoting $O = (s,a,s^{'})$,
we have for any $\theta ,\gamma_1,\gamma_{2}$, that
\begin{align*}
    &\hat{Q}(O,\theta,\gamma^1,G) - \hat{Q}(O,\theta,\gamma^2,G) \\
&= \langle \nabla L(\theta,\gamma^1) , G^{-1}\bar{H}(O,\theta,\gamma^1) - E_{O^{'}}[G^{-1}\bar{H}(O^{'},\theta,\gamma^1) ]\rangle - \langle \nabla L(\theta,\gamma^2) , G^{-1}\bar{H}(O,\theta,\gamma^2) - E_{O^{'}}[G^{-1}\bar{H}(O^{'},\theta,\gamma^2) ]\rangle\\
&=\langle \nabla L(\theta,\gamma^1) , G^{-1}\bar{H}(O,\theta,\gamma^1) - E_{O^{'}}[G^{-1}\bar{H}(O^{'},\theta,\gamma^1) ]\rangle -\langle \nabla L(\theta,\gamma^1) , G^{-1}\bar{H}(O,\theta,\gamma^2) - E_{O^{'}}[G^{-1}\bar{H}(O^{'},\theta,\gamma^2) ]\rangle \\
&\qquad +\langle \nabla L(\theta,\gamma^1) , G^{-1}\bar{H}(O,\theta,\gamma^2) - E_{O^{'}}[G^{-1}\bar{H}(O^{'},\theta,\gamma^2) ]\rangle - \langle \nabla L(\theta,\gamma^2) , G^{-1}\bar{H}(O,\theta,\gamma^2) - E_{O^{'}}[G^{-1}\bar{H}(O^{'},\theta,\gamma^2) ]\rangle\\
&= \underbrace{\langle \nabla L(\theta,\gamma^1) , G^{-1}\bar{H}(O,\theta,\gamma^1)  -  G^{-1}\bar{H}(O,\theta,\gamma^2) - E_{O^{'}}[G^{-1}\bar{H}(O^{'},\theta,\gamma^1)] + E_{O^{'}}[G^{-1}\bar{H}(O^{'},\theta,\gamma^2)] \rangle}_{I_1}\\
&\qquad + \underbrace{\langle \nabla L(\theta,\gamma^1) - \nabla  L(\theta,\gamma^2),G^{-1}\bar{H}(O,\theta,\gamma^2) - E_{O^{'}}[G^{-1}\bar{H}(O^{'},\theta,\gamma^2) ]\rangle}_{I_2},
\end{align*}
where $G$ is a non-singular square matrix and $\Vert G \Vert < U_{G}$. We have by lemma \ref{Lemma:L-smooth2} that
\begin{align*}
     \Vert \nabla L(\theta,\gamma^1) - \nabla L(\theta_,\gamma^2)\Vert \leq C| \gamma_m^1 - \gamma_m^2|.
\end{align*}
Now,
\begin{align*}
    &\Vert \bar{H}(O,\theta,\gamma^1)  -  \bar{H}(O,\theta,\gamma^2) \Vert\\
    &= \Vert (c(s,a,s^{'},\gamma^1)- c(s,a,s^{'},\gamma^2) - L(\theta,\gamma^1) + L(\theta,\gamma^2)  +({f_{s^{'}}}^{T} - {f_{s}}^{T})(v^{*}(\theta,\gamma^1)-v^{*}(\theta,\gamma^2)))\nabla \log\pi_{\theta}(a|s) \Vert\\
    &\leq D(\vert c(s,a,s^{'},\gamma^1)- c(s,a,s^{'},\gamma^2) \vert + \vert L(\theta,\gamma^1) - L(\theta,\gamma^2) \vert + 2\Vert v^{*}(\theta,\gamma^1)-v^{*}(\theta,\gamma^2) \Vert )\\
    &\leq D(2N(U_c + U_{\alpha})\vert \gamma^1_m - \gamma^2_m \vert + 2L_2\vert \gamma^1_m - \gamma^2_m \vert),
\end{align*}
where  $\vert\gamma_m^1-\gamma_m^2\vert = \max\limits_{i = 1,2,3.... , N}|\gamma_i^1-\gamma_i^2|$. Hence (for the term $I_1$), we have that 
\begin{align*}
    I_1 &\leq 4D(N(U_c + U_\alpha) + L_2)\Vert \nabla L(\theta,\gamma^1)\Vert U_{G} \vert \gamma^1_m - \gamma^2_m \vert.
\end{align*}
Now observe that (for the term $I_{2}$),
\begin{align*}
    I_2 &\leq \Vert \nabla L(\theta,\gamma^1) - \nabla  L(\theta,\gamma^2) \Vert \Vert \bar{H}(O,\theta,\gamma^2) - E_{O^{'}}[\bar{H}(O^{'},\theta,\gamma^2) ]\Vert\\
    &\leq 4D(U_r + U_v)\Vert \nabla L(\theta,\gamma^1) - \nabla  L(\theta,\gamma^2) \Vert\\
    &\leq 4DC(U_r + U_v)U_{G}\vert \gamma^1_m - \gamma^2_m\vert
\end{align*} 
Combining the RHS of the two terms, we obtain
\begin{align*}
    \vert \hat{Q}(O,\theta,\gamma^1,G) - \hat{Q}(O,\theta,\gamma^2,G)\vert &\leq T_1\vert \gamma^1_m - \gamma^2_m\vert,
\end{align*}
where $T_1 =  4D(N(U_c + U_\alpha) + L_2)U_{L} U_{G}+  4DC(U_r + U_v)U_{G}$.

\begin{sublemma}\label{Lemma:zeta_theta}
For any $t \geq 0, \theta_1,\theta_2,G,\gamma = (\gamma_1,\gamma_2,...,\gamma_N)^T$ with $0 \leq \gamma_i \leq M$ for $i \in \{1,2,,,..,N\}$ and $G$ being a non-singular square matrix with  $\Vert G^{-1} \Vert \leq U_G$,
\begin{align*}
    \vert \hat{Q}(O,\theta_1,\gamma,G) - \hat{Q}(O,\theta_2,\gamma,G)\vert &\leq T_2\Vert \theta_1 - \theta_2 \Vert,
\end{align*}
for some $T_2 > 0$.
\end{sublemma}

\textit{Proof}
Denote $O = (s,a,s^{'})$. We have for any $\theta_{1} ,\theta_{2},\gamma,G$, the following:

\begin{align*}
    &\hat{Q}(O,\theta_1,\gamma,G) - \hat{Q}(O,\theta_2,\gamma,G) \\
&= \langle \nabla L(\theta_1,\gamma) , G^{-1}\bar{H}(O,\theta_1,\gamma) - E_{O^{'}}[G^{-1}\bar{H}(O^{'},\theta_1,\gamma) ]\rangle - \langle \nabla L(\theta_2,\gamma) , G^{-1}\bar{H}(O,\theta_2,\gamma) - E_{O^{'}}[G^{-1}\bar{H}(O^{'},\theta_2,\gamma) ]\rangle\\
&=\langle \nabla L(\theta_1,\gamma) , G^{-1}\bar{H}(O,\theta_1,\gamma) - E_{O^{'}}[G^{-1}\bar{H}(O^{'},\theta_1,\gamma) ]\rangle -\langle \nabla L(\theta_1,\gamma) , G^{-1}\bar{H}(O,\theta_2,\gamma) - E_{O^{'}}[G^{-1}\bar{H}(O^{'},\theta_2,\gamma) ]\rangle \\
&\qquad +\langle \nabla L(\theta_1,\gamma) , G^{-1}\bar{H}(O,\theta_2,\gamma) - E_{O^{'}}[G^{-1}\bar{H}(O^{'},\theta_2,\gamma) ]\rangle - \langle \nabla L(\theta_2,\gamma) ,G^{-1} \bar{H}(O,\theta_2,\gamma) - E_{O^{'}}[G^{-1}\bar{H}(O^{'},\theta_2,\gamma) ]\rangle\\
&= \underbrace{\langle \nabla L(\theta_1,\gamma) , G^{-1}\bar{H}(O,\theta_1,\gamma)  - G^{-1} \bar{H}(O,\theta_2,\gamma) - E_{O^{'}}[G^{-1}\bar{H}(O^{'},\theta_1,\gamma)] + E_{O^{'}}[G^{-1}\bar{H}(O^{'},\theta_2,\gamma)] \rangle}_{I_1}\\
&\qquad + \underbrace{\langle \nabla L(\theta_1,\gamma) - \nabla  L(\theta_2,\gamma),G^{-1}\bar{H}(O,\theta_2,\gamma) - E_{O^{'}}[G^{-1}\bar{H}(O^{'},\theta_2,\gamma) ]\rangle}_{I_2}.
\end{align*}
Now,

\begin{align*}
    &\Vert \bar{H}(O,\theta_1,\gamma) - \bar{H}(O,\theta_2,\gamma)\Vert  \\
    &=\Vert (c(s,a,s^{'},\gamma) -L(\theta_1,\gamma)  +({f_{s^{'}}}^{T} - {f_{s}}^{T})v^{*}(\theta_1,\gamma))\nabla \log\pi_{\theta_1}(a|s)\\
    &\qquad- (c(s,a,s^{'},\gamma) -L(\theta_2,\gamma)  +({f_{s^{'}}}^{T} - {f_{s}}^{T})v^{*}(\theta_2,\gamma))\nabla \log\pi_{\theta_2}(a|s) \Vert\\
    &\leq \Vert (c(s,a,s^{'},\gamma) -L(\theta_1,\gamma)  +({f_{s^{'}}}^{T} - {f_{s}}^{T})v^{*}(\theta_1,\gamma))(\nabla \log\pi_{\theta_1}(a|s) - \nabla \log\pi_{\theta_2}(a|s)) \Vert\\
    &\qquad+ \Vert (L(\theta_2,\gamma) - L(\theta_1,\gamma) + ({f_{s^{'}}}^{T} - {f_{s}}^{T})(v^{*}(\theta_1,\gamma) - v^{*}(\theta_1,\gamma)))\nabla \log\pi_{\theta_2}(a|s)\Vert\\
    &\leq 2(U_r + U_v)M_m\Vert \theta_1 - \theta_2\Vert + D(\vert L(\theta_2,\gamma) - L(\theta_1,\gamma) \vert + 2L_1\Vert \theta_1 - \theta_2 \Vert).
\end{align*}
Also, clearly

\begin{align*}
    &\vert L(\theta_1,\gamma) - L(\theta_2,\gamma) \vert\\
    &= \vert \sum\limits_{s \in S}\mu_{\theta_1}(s)\sum\limits_{a \in A(s)}\pi_{\theta_1}(s,a) (d(s,a) + \sum\limits_{k=1}^{N}\gamma(k)(h_k(s,a) -\alpha_k)) - \sum\limits_{s \in S}\mu_{\theta_2}(s)\sum\limits_{a \in A(s)}\pi_{\theta_2}(s,a) (d(s,a) + \sum\limits_{k=1}^{N}\gamma(k)(h_k(s,a) -\alpha_k)) \vert\\
    &\leq 2U_r d_{TV}(\mu_{\theta_1} \otimes \pi_{\theta_1},\mu_{\theta_2} \otimes \pi_{\theta_2})\\
    & \leq 2U_r\vert A\vert L\bigg( 1  + \lceil\log_{k}b^{-1}\rceil + 1/(1-k)\bigg)\Vert \theta_1 - \theta_2 \Vert\\
    &=C_{L}\Vert \theta_1 - \theta_2 \Vert.
\end{align*}

Hence,
\begin{align*}
\Vert \bar{H}(O,\theta_1,\gamma) - \bar{H}(O,\theta_2,\gamma)\Vert \leq  2(U_r + U_v)M_m\Vert \theta_1 - \theta_2\Vert + DC_{L}\Vert \theta_1 - \theta_2 \Vert +2L_1D\Vert \theta_1 - \theta_2\Vert.
\end{align*}

Also, note that 
    \begin{align*}
    &\Vert E_{O^{'}}[\bar{H}(O^{'},\theta_1,\gamma)] - E_{O^{'}}[\bar{H}(O^{'},\theta_2,\gamma)] \Vert\\ &= \Vert E_{\theta_1}[\bar{H}(O^{'},\theta_1,\gamma)] - E_{\theta_2}[\bar{H}(O^{'},\theta_2,\gamma)] \Vert\\
    &\leq \Vert E_{\theta_1}[\bar{H}(O^{'},\theta_1,\gamma)] - E_{\theta_1}[\bar{H}(O^{'},\theta_2,\gamma)] \Vert + \Vert E_{\theta_1}[\bar{H}(O^{'},\theta_2,\gamma)] - E_{\theta_2}[\bar{H}(O^{'},\theta_2,\gamma)] \Vert\\
    &\leq E_{\theta_1}\Vert \bar{H}(O^{'},\theta_1,\gamma)-\bar{H}(O^{'},\theta_2,\gamma)\Vert + 4D(U_r + U_v)d_{TV}(\mu_{\theta_1} \otimes \pi_{\theta_1} , \mu_{\theta_2} \otimes \pi_{\theta_2} )\\
    &\leq \bigg[2(U_r + U_v)M_m + DC_L + 2L_1D + 4D(U_r + U_v)|A|U_rL \bigg(1 + \lceil \log_{k}b^{-1} \rceil + \frac{1}{1-k} \bigg) \bigg]\norm{\theta_1 - \theta_2}\\
    &=A_2\norm{\theta_1-\theta_2}.
\end{align*}

Thus, we have
\begin{align*}
    I_1 \leq 2U_{L}U_{G}(U_r + U_v)M_m\Vert \theta_1 - \theta_2\Vert + U_{L}U_{G}DC_{L}\Vert \theta_1 - \theta_2 \Vert +2U_{L}U_{G}L_1D\Vert \theta_1 - \theta_2\Vert + U_{G}U_{L}A_2\norm{\theta_1 - \theta_2}.
\end{align*}

For the term $I_2$, we have
\begin{align*}
    I_2 &\leq \Vert \nabla L(\theta_1,\gamma) - \nabla  L(\theta_2,\gamma) \Vert \Vert G^{-1}\bar{H}(O,\theta_2,\gamma) - E_{O^{'}}[G^{-1}\bar{H}(O^{'},\theta_2,\gamma) ] \Vert\\
    &\leq 4D(U_r + U_v)U_{G}M_{L}\Vert \theta_1 - \theta_2\Vert.
\end{align*}

The last inequality follows from Lemma \ref{Lemma:L-smooth}.

Finally, we have
\begin{align*}
    \hat{Q}(O,\theta_1,\gamma,G) - \hat{Q}(O,\theta_2,\gamma,G) &\leq T_2\Vert \theta_1 - \theta_2 \Vert,
\end{align*}
where,
\begin{align*}
    T_2 &= 2U_{L}U_{G}(U_r + U_v)M_m + U_{L}U_{G}DC_{L} +2U_{L}U_{G}L_1D + U_{G}U_{L}A_2 +4D(U_r + U_v)U_GM_{L},\\
    A_2 &=2(U_r + U_v)M_m + DC_L + 2L_1D + 4D(U_r + U_v)|A|U_rL \bigg(1 + \lceil \log_{k}b^{-1} \rceil + \frac{1}{1-k} \bigg).
\end{align*} 

\begin{sublemma}\label{Lemma:zeta_first}

For any $t \geq 0$,conditioned on $\theta_{t-\tau},\gamma(t-\tau)$ and $G(t-\tau)$,
\begin{align*}
     &\vert E[(\hat{Q}(O_t,\theta_{t-\tau},\gamma(t-\tau),G(t-\tau)) - \hat{Q}(\tilde{O_{t}},\theta_{t-\tau},\gamma(t-\tau),G(t-\tau)))|\theta_{t-\tau},\gamma(t-\tau),G(t-\tau)]\vert \\
     &\qquad\leq 2DU_{G}(U_{r} + U_{v})U_{L}\vert A \vert L\sum_{i=t-\tau}^{t} E\Vert\theta_i - \theta_{t-\tau} \Vert 
\end{align*}

\end{sublemma}
\textit{Proof}
By the definition of $\hat{Q}(O,\theta,\gamma)$,
\begin{align*}
    &E[(\hat{Q}(O_t,\theta_{t-\tau},\gamma(t-\tau),G(t-\tau)) - \hat{Q}(\tilde{O_{t}},\theta_{t-\tau},\gamma(t-\tau),G(t-\tau))|\theta_{t-\tau},\gamma(t-\tau),G(t-\tau)]\\
&= E[\langle \nabla L(\theta_{t-\tau},\gamma(t-\tau)),G(t-\tau)^{-1}\bar{H}(O_t,\theta_{t-\tau},\gamma(t-\tau)- G(t-\tau)^{-1}\bar{H}(\tilde{O_{t}},\theta_{t-\tau},\gamma(t-\tau)\rangle|\theta_{t-\tau},\gamma(t-\tau),G(t-\tau)]\\
&=  E[(\langle \nabla L(\theta_{t-\tau},\gamma(t-\tau)),G(t-\tau)^{-1}\bar{H}(O_t,\theta_{t-\tau},\gamma(t-\tau)\rangle\\
&\qquad- \langle \nabla L(\theta_{t-\tau},\gamma(t-\tau)) , G(t-\tau)^{-1}\bar{H}(\tilde{O_{t}},\theta_{t-\tau},\gamma(t-\tau))\rangle)|\theta_{t-\tau},\gamma(t-\tau),G(t-\tau)]\\
&\leq 4U_{G} D(U_{r} + U_{v})U_{L}d_{TV}(P(O_t = .|s_{t-\tau +1},\theta_{t-\tau}),(P(\tilde{O_{t}} = .|s_{t-\tau +1},\theta_{t-\tau})).
\end{align*}

Now,
\begin{align*}
    d_{TV}(P(O_t = .|s_{t-\tau +1},\theta_{t-\tau}),(P(\tilde{O_{t}} = .|s_{t-\tau +1},\theta_{t-\tau})) \leq \frac{1}{2}\vert A \vert L\sum_{i=t-\tau}^{t} E\Vert\theta_i - \theta_{t-\tau} \Vert.
\end{align*}
This inequality follows in a similar manner as the proof of Lemma D.2 of \cite{fta_2_timescale}.
Hence,
\begin{align*}
    &E[\hat{Q}(O_t,\theta_{t-\tau},\gamma(t-\tau),G(t-\tau)) - \hat{Q}(\tilde{O_{t}},\theta_{t-\tau},\gamma(t-\tau),G(t-\tau))|\theta_{t-\tau},\gamma(t-\tau),G(t-\tau)] \\
    &\qquad\leq 2DU_{G}(U_{r} + U_{v})U_{L}\vert A \vert L\sum_{i=t-\tau}^{t} E\Vert\theta_i - \theta_{t-\tau} \Vert. 
\end{align*}
The claim follows.

\begin{sublemma}\label{Lemma:zeta_second}

For any $t \geq 0$, conditioned on $\theta_{t-\tau},\gamma(t-\tau)$ and $G(t-\tau)$,
\begin{align*}
    &\vert E[(\hat{Q}(\tilde{O_{t}},\theta_{t-\tau},\gamma(t-\tau),G(t-\tau)) - \hat{Q}(O_{t}^{'},\theta_{t-\tau},\gamma(t-\tau),G(t-\tau)))|\theta_{t-\tau},\gamma(t-\tau),G(t-\tau)]\vert\\
    &\qquad\leq 4DU_{G}(U_{r} + U_{v})U_{L} bk^{\tau - 1}.
\end{align*} 

\end{sublemma}
\textit{Proof}

\begin{align*}
    &E[(\hat{Q}(\tilde{O_{t}},\theta_{t-\tau},\gamma(t-\tau),G(t-\tau)) - \hat{Q}(O_{t}^{'},\theta_{t-\tau},\gamma(t-\tau),G(t-\tau)))|\theta_{t-\tau},\gamma(t-\tau),G(t-\tau)]\\
&= E[\langle \nabla L(\theta_{t-\tau},\gamma(t-\tau)),G(t-\tau)^{-1}\bar{H}(\tilde{O_{t}},\theta_{t-\tau},\gamma(t-\tau))- G(t-\tau)^{-1}\bar{H}(O_{t}^{'},\theta_{t-\tau},\gamma(t-\tau))\rangle|\theta_{t-\tau},\gamma(t-\tau),G(t-\tau)]\\
&= E[(\langle \nabla L(\theta_{t-\tau},\gamma(t-\tau)),G(t-\tau)^{-1}\bar{H}(\tilde{O_{t}},\theta_{t-\tau},\gamma(t-\tau))\rangle\\
&\qquad- \langle \nabla L(\theta_{t-\tau},\gamma(t-\tau)) , G(t-\tau)^{-1}\bar{H}(O_{t}^{'},\theta_{t-\tau},\gamma(t-\tau))\rangle)|\theta_{t-\tau},\gamma(t-\tau),G(t-\tau)]\\
&\leq 4DU_{G}(U_{r} + U_{v})U_{L}d_{TV}(P(\tilde{O_{t}} = .|s_{t-\tau +1},\theta_{t-\tau}),\mu_{\theta_{t-\tau}} \otimes \pi_{\theta_{t-\tau}} \otimes P)\\
&\leq 4DU_{G}(U_{r} + U_{v})U_{L} bk^{\tau - 1}.
\end{align*}

The last inequality comes using an inequality on the total variation distance that is shown in the proof of Lemma D.3  of \cite{fta_2_timescale}.

\textit{Proof of Lemma \ref{lemma:zeta}:}

We decompose $E[\hat{Q}(O_t,\theta_t,\gamma(t),G(t))]$ as:

\begin{align*}
    &E[\hat{Q}(O_t,\theta_t,\gamma(t),G(t))] \\
    &= E[\hat{Q}(O_t,\theta_t,\gamma(t),G(t)) - \hat{Q}(O_t,\theta_t,\gamma(t),G(t-\tau))] + E[\hat{Q}(O_t,\theta_t,\gamma(t),G(t-\tau)) - \hat{Q}(O_t,\theta_t,\gamma(t-\tau),G(t-\tau))]\\
    &\qquad+ E[\hat{Q}(O_t,\theta_t,\gamma(t-\tau),G(t-\tau)) - \hat{Q}(O_t,\theta_{t-\tau},\gamma(t-\tau),G(t-\tau))]\\
     &\qquad+ E[\hat{Q}(O_t,\theta_{t-\tau},\gamma(t-\tau),G(t-\tau)) - \hat{Q}(\tilde{O_t},\theta_{t-\tau},\gamma(t-\tau),G(t-\tau))]\\
     &\qquad+ E[\hat{Q}(\tilde{O_t},\theta_{t-\tau},\gamma(t-\tau),G(t-\tau)) - \hat{Q}(O_t^{'},\theta_{t-\tau},\gamma(t-\tau),G(t-\tau))] 
    + E[\hat{Q}(O_t^{'},\theta_{t-\tau},\gamma(t-\tau))],
\end{align*}
  where $\tilde{O_t}$  is from the auxiliary  Markov chain  and $O_t^{'} = (s_t,a_t,s_{t+1})$ is  from the stationary distribution with $s_t \sim \mu_{\theta_{t-\tau}},a_t \sim \pi_{\theta_{t-\tau}},s_{t+1} \sim p(s_t,.,a_t)$ and which
actually satisfies $E[\hat{Q}(O_t^{'},\theta_{t-\tau},\gamma(t-\tau))] = 0$.
By collecting the corresponding bounds from Lemmas \ref{Lemma:zeta_G}--\ref{Lemma:zeta_second}, we have 
 
\begin{align*}
    E[\hat{Q}(O_t,\theta_t,\gamma(t),G(t))] &\geq -T_1E|\gamma_m(t) - \gamma_m(t-\tau)| - T_2E\Vert \theta_t - \theta_{t-\tau} \Vert - 2DU_G(U_{r} + U_{v})U_{L}\vert A \vert L\sum_{i=t-\tau}^{t} E\Vert\theta_i - \theta_{t-\tau} \Vert \\
    &\qquad- 4DU_G(U_{r} + U_{v})U_{L} bk^{\tau - 1} - T_{G}E\Vert G(t)^{-1} - G(t-\tau)^{-1}\Vert \\ 
   &\geq -T_1\sum_{i=t-\tau +1 }^{ t} E\vert \gamma_m(i) - \gamma_m(i-1)\vert - T_2\sum_{i=t-\tau +1 }^{ t}E\Vert \theta_i - \theta_{i-1} \Vert\\ 
   &\qquad- 2D(U_{r} + U_{v})U_{L}\vert A \vert L\sum_{i=t-\tau +1 }^{t}\sum_{j=t-\tau +1 }^{ i}E\Vert\theta_j - \theta_{j-1} \Vert - 4D(U_{r} + U_{v})U_{L} bk^{\tau - 1} \\
   &\qquad - T_{G}E\Vert G(t)^{-1} - G(t-\tau)^{-1}\Vert\\
   &\geq -T_1\sum_{i=t-\tau +1 }^{ t} E\vert \gamma_m(i) - \gamma_m(i-1)\vert - T_2\sum_{i=t-\tau +1 }^{ t}E\Vert \theta_i - \theta_{i-1} \Vert\\ 
   &\qquad- 2D(U_{r} + U_{v})U_{L}\vert A \vert L\tau\sum_{j=t-\tau +1 }^{ t}E\Vert\theta_j - \theta_{j-1} \Vert - 4D(U_{r} + U_{v})U_{L} bk^{\tau - 1} \\
   &\qquad - T_{G}E\Vert G(t)^{-1} - G(t-\tau)^{-1}\Vert\\
   &\geq -(D_1(\tau + 1)\sum_{k=t-\tau + 1}^{t} E\Vert \theta_k - \theta_{k-1} \Vert + D_2bk^{\tau - 1} + T_1\sum_{i=t-\tau +1 }^{ t} E\vert \gamma_m(i) - \gamma_m(i-1)\vert\\
   &\qquad + T_{G}\sum_{i=t-\tau +1 }^{ t} E\Vert G(i)^{-1} - G(i-1)^{-1}\Vert),
\end{align*} 
where $D_1 := \max \{T_2 , 2DU_G(U_{r} + U_{v})U_{L}\vert A \vert L\}$ and $D_2 := 4DU_G(U_{r} + U_{v})U_{L}$, respectively,
which completes the proof.
\end{proof}

Under the update rule of Algorithm 2 for the actor, we have:

\begin{align*}
    \theta_{t+1} = \theta_t + b(t)G(t)^{-1}\delta_{t}\Psi_{s_{t}a_{t}}.
\end{align*}
Now,
\begin{align*}
 &G(t)^{-1}\delta_{t}\nabla \log\pi_{\theta_t}(a_t|s_t) \\
&=
G(t)^{-1}( q(t) + \sum_{k=1}^{k=N}\gamma_{k}(t)(h_k(t)-\alpha_k) - L_t + v_{t}^T(f(s_{t+1}) - f(s_t))\nabla \log\pi_{\theta_t}(a_t|s_t) \\
&=
G(t)^{-1}(q(t) + \sum_{k=1}^{k=N}\gamma_{k}(t)(h_k(t)-\alpha_k) - L(\theta_t,\gamma(t)) + L(\theta_t,\gamma(t)) - L_t + (f(s_{t+1})^T - f(s_t)^T)(v_{t} - v_{t}^*) \\
&\qquad+ (f(s_{t+1})^T - f(s_t)^T)v_{t}^*)\nabla \log\pi_{\theta_t}(a_t|s_t) \\
&= G(t)^{-1}(L(\theta_t,\gamma(t)) - L_t + (f(s_{t+1})^T - f(s_t)^T)(v_{t} - v_{t}^*))\nabla \log\pi_{\theta_t}(a_t|s_t)\\ 
&\qquad  + G(t)^{-1}(q(t) + \sum_{k=1}^{k=N}\gamma_{k}(t)(h_k(t)-\alpha_k) - L(\theta_t,\gamma(t)) + (f(s_{t+1})^T - f(s_t)^T)v_{t}^* )\nabla \log\pi_{\theta}(a_t|s_t) \\
&=  G(t)^{-1}(\Delta H(O_t,L_t,v_{t},\theta_t,\gamma(t)) + H(O_t,\theta_t,\gamma(t),q(t),h(t))). 
\end{align*}
We have,

\begin{align*}
 & E_{O^{'},q,h}[G^{-1}(H(O^{'},\theta,\gamma,q,h) - \Delta H^{'}(O^{'},\theta,\gamma)) ] \\
 & =  G^{-1}E_{O^{'},q,h}[q + \sum_{k=1}^{k=N}\gamma(k)(h_k-\alpha_k) - L(\theta,\gamma) + V^{(\theta,\gamma)}(s^{'}) - V^{(\theta,\gamma)}(s))\nabla \log\pi_{\theta}(a|s)] \\
 &= G^{-1}\nabla L(\theta,\gamma), 
\end{align*}
where $E_{O^{'},q,h}[\cdot]$ denotes the expectation w.r.t $s \sim \mu_{\theta},a \sim \pi_{\theta},s^{'} \sim p(s,.,a),q \sim \bar{p}(.|s,a,s^{'}),h_{i} \sim p_{i}(.|s,a,s^{'})$.
Hence,
\begin{align*}
    L(\theta_{t+1},\gamma(t)) 
    &\geq L(\theta_t,\gamma(t)) + b(t) \langle \nabla L(\theta_t,\gamma(t)) ,  G(t)^{-1}\Delta H(O_t,L_t,v_{t},\theta_t,\gamma(t)) + G(t)^{-1}H(O_t,\theta_t,\gamma(t),q(t),h(t))\rangle\\ 
    &\qquad  - M_{L}b(t)^2\Vert G(t)^{-1}\delta_{t}\nabla \log\pi_{\theta_t}(a_t|s_t)\Vert^2\\
    &\geq L(\theta_t,\gamma(t)) + b(t) \langle \nabla L(\theta_t,\gamma(t)) ,  G(t)^{-1}\Delta H(O_t,L_t,v_{t},\theta_t,\gamma(t)) \rangle \\
    &\qquad +b(t)\langle \nabla L(\theta_t,\gamma(t)) , G(t)^{-1}H(O_t,\theta_t,\gamma(t),q(t),h(t)) - E_{O^{'},q(t),h(t)}[G(t)^{-1}H(O^{'},\theta_t,\gamma(t),q(t),h(t))]\rangle\\ &\qquad +b(t)\langle \nabla L(\theta_t,\gamma(t)) ,  E_{O^{'},q(t),h(t)}[G(t)^{-1}H(O^{'},\theta_t,\gamma(t),q(t),h(t))]\rangle - M_{L}b(t)^2\Vert G(t)^{-1}\delta_{t}\nabla \log\pi_{\theta_t}(a_t|s_t)\Vert^2\\
   & \geq L(\theta_t,\gamma(t)) + b(t) \langle \nabla L(\theta_t,\gamma(t)) ,  G(t)^{-1}\Delta H(O_t,L_t,v_{t},\theta_t,\gamma(t)) \rangle \\
   &\qquad+ b(t)\zeta(O_t,\theta_t,\gamma(t),q(t),h(t),G(t)) + b(t)\langle \nabla L(\theta_t,\gamma(t)) ,G(t)^{-1}\nabla L(\theta_t,\gamma(t))\rangle\\
   &\qquad+  b(t)\langle \nabla L(\theta_t,\gamma(t)) ,  E_{O^{'}}[G(t)^{-1}\Delta H^{'}(O^{'} , \theta_t,\gamma(t))]\rangle- M_{L}b(t)^2\Vert G(t)^{-1}\delta_{t}\nabla \log\pi_{\theta_t}(a_t|s_t)\Vert^2\\
   & \geq L(\theta_t,\gamma(t)) + b(t) \langle \nabla L(\theta_t,\gamma(t)) ,  G(t)^{-1}\Delta H(O_t,L_t,v_{t},\theta_t,\gamma(t)) \rangle \\
   &\qquad+ b(t)\zeta(O_t,\theta_t,\gamma(t),q(t),h(t),G(t)) + b(t)\lambda\Vert  \nabla L(\theta_t,\gamma(t))  \Vert^2\\
   &\qquad+  b(t)\langle \nabla L(\theta_t,\gamma(t)) ,  E_{O^{'}}[G(t)^{-1}\Delta H^{'}(O^{'} , \theta_t,\gamma(t))]\rangle- M_{L}b(t)^2\Vert G(t)^{-1}\delta_{t}\nabla \log\pi_{\theta_t}(a_t|s_t)\Vert^2.
\end{align*}
The last inequality holds as $G(t)^{-1}$ is a positive definite and symmetric matrix with minimum eigenvalue $\geq \lambda$.
After rearranging the terms and summing the expectation of the terms from $\tau_t$ to $t$, we have
\begin{align*}
    \lambda\sum\limits_{k=\tau_t}^{t}E \Vert  \nabla L(\theta_k,\gamma(k))\Vert^2 &\leq \underbrace{\sum\limits_{k=\tau_t}^{t}\frac{1}{b(k)}(E[L(\theta_{k+1},\gamma(k))] - E[L(\theta_k,\gamma(k))])}_{I_1}\\
&\qquad-\underbrace{ \sum\limits_{k=\tau_t}^{t}E\langle \nabla L(\theta_k,\gamma(k)) ,  G(k)^{-1}\Delta H(O_k,L_k,v_{k},\theta_k,\gamma(k))\rangle}_{I_2}\\
&\qquad- \underbrace{ \sum\limits_{k=\tau_t}^{t}E[\zeta(O_k,\theta_k,\gamma(k),q(k),h(k),G(k))] }_{I_3} \\
&\qquad  - \underbrace{\sum\limits_{k=\tau_t}^{t}E \langle \nabla L(\theta_k,\gamma(k)) ,  E_{O^{'}}[G(k)^{-1}\Delta H^{'}(O^{'} , \theta_k,\gamma(k))]\rangle}_{I_4}\\
 &\qquad +\underbrace{\sum\limits_{k=\tau_t}^{t}b(k)E[ M_{L}\Vert G(k)^{-1}\delta_{k}\nabla \log\pi_{\theta_k}(a_k|s_k)\Vert^2]}_{I_5}.
\end{align*} 
 \hfill\break
 For term $I_1$,
\begin{align*}
    \sum\limits_{k=\tau_t}^{t}\frac{1}{b(k)}(E[L(\theta_{k+1},\gamma(k))] - E[L(\theta_k,\gamma(k))]) \leq B_1(t-\tau_t + 1)^{1 - \beta + \sigma} + B_2(1+t)^{\sigma}.
\end{align*}
 This inequality comes from part $I_1$ of Section \ref{thmm:actor_1}.

 \hfill\break
 For term $I_2$,
\begin{align*}
    &- \sum\limits_{k=\tau_t}^{t}E\langle \nabla L(\theta_k,\gamma(k)) ,  G(k)^{-1}\Delta H(O_k,L_k,v_{k},\theta_k,\gamma(k))\rangle\\
& \leq \sqrt{\sum\limits_{k=\tau_t}^{t}E\Vert \nabla L(\theta_k,\gamma(k)) \Vert^2}\sqrt{\sum\limits_{k=\tau_t}^{t} E\Vert G(k)^{-1}\Delta H(O_k,L_k,v_{k},\theta_k,\gamma(k))\Vert^2}.
 \end{align*}
 
 Now,
\begin{align*}
&E\Vert G(k)^{-1}\Delta H(O_k,L_k,v_{k},\theta_k,\gamma(k)) \Vert^2\\
& = E\Vert G(k)^{-1}(L(\theta_k,\gamma(k)) - L_k + (f(s_{k+1})^T - f(s_k)^T)(v(k) - v(\theta_k,\gamma(k))^*))\nabla \log\pi_{\theta_k}(a_k|s_k)   \Vert^2\\
& \leq U_G^2D^2(2E\Vert A_k\Vert ^2 + 8E\Vert B_k\Vert^2).
\end{align*}

Hence,
\begin{align*}
    I_2 \leq DU_G\sqrt{\sum\limits_{k=\tau_t}^{t}E\Vert \nabla L(\theta_k,\gamma(k)) \Vert^2}\sqrt{\sum\limits_{k=\tau_t}^{t}(2E\Vert A_k\Vert ^2 + 8E\Vert B_k\Vert^2)}.
\end{align*}
 
 Next, for the term $I_3$, we have

\begin{align*}
E[\zeta(O_t,\theta_t,\gamma(t),q(t),h(t),G(t))] &\geq -(D_1(\tau + 1)\sum_{k=t-\tau + 1}^{t} E\Vert \theta_k - \theta_{k-1} \Vert + D_2bk^{\tau - 1} + T_1\sum_{i=t-\tau +1 }^{ t} E\vert \gamma_m(i) - \gamma_m(i-1)\vert\\
   &\qquad + T_{G}\sum_{i=t-\tau +1 }^{ t} E\Vert G(i)^{-1} - G(i-1)^{-1}\Vert).
\end{align*}

We can write 
\begin{align*}
    G(i)^{-1} - G(i-1)^{-1} = G(i)^{-1}(G(i-1) - G(i))G(i-1)^{-1}.
\end{align*}
By letting $\tau \stackrel{\triangle}{=} \tau_t$, we have

\begin{align*}
E[\zeta(O_k,\theta_k,\gamma(k),q(k),h(k),G(k))] &\geq -(2D_1D(\tau_t + 1)^2(U_r + U_v)U_G b(k-\tau_t) + D_2b(t)\\
&\qquad+ T_1(U_c + U_\alpha)(\tau_t + 1)c(k-\tau_t) +T_{G}(\tau_t + 1)U_G^2(U_G + D^2)a(k-\tau_t)).
\end{align*}

After simplifying, we have
\begin{align*}
\sum\limits_{k=\tau_t}^{t}E[\zeta(O_k,\theta_k,\gamma(k),q(k),h(k),G(k))] &\geq -A_1(\tau_t + 1)^2(1+t)^{1-\omega}
\end{align*}
for some $A_1 > 0$. Now, for term $I_4$, we have

\begin{align*}
   \sum\limits_{k=\tau_t}^{t}E \langle \nabla L(\theta_k,\gamma(k)) ,  E_{O^{'}}[G(k)^{-1}\Delta H^{'}(O^{'} , \theta_k,\gamma(k))]\rangle \geq -2DU_{L}U_{G}\epsilon_{app}( 1 + t -\tau_t).
\end{align*}

Also, for the term $I_5$, we have

\begin{align*}
    \sum\limits_{k=\tau_t}^{t}b(k)E[ M_{L}\Vert G(k)^{-1}\delta_{k}\nabla \log\pi_{\theta_k}(a_k|s_k)\Vert^2] \leq C_1(1 + t -\tau_t)^{1-\sigma},
\end{align*}

where $C_1$ is a positive constant. After combining all terms, we have

\begin{align*}
    \lambda\sum\limits_{k=\tau_t}^{t}E \Vert  \nabla L(\theta_k,\gamma(k))\Vert^2 &\leq \mathcal{O}((t + 1)^{1 - \beta + \sigma}) + \mathcal{O}((\tau_t + 1)^2(1+t)^{1-\omega})\\ 
&\qquad + DU_G\sqrt{\sum\limits_{k=\tau_t}^{t}E\Vert \nabla L(\theta_k,\gamma(k)) \Vert^2}\sqrt{\sum\limits_{k=\tau_t}^{t}(2E[A_k^2] + 8E[B_k^2])}\\
&\qquad   + 2DU_{L}U_{G}\epsilon_{app}( 1 + t -\tau_t).
\end{align*}
Dividing now  both sides by $(1 + t -\tau_t)$ and assuming $t \geq 2\tau_t - 1$, we obtain

\begin{align*}
    \lambda\sum\limits_{k=\tau_t}^{t}E \Vert  \nabla L(\theta_k,\gamma(k))\Vert^2/(1 + t -\tau_t)&\leq \mathcal{O}(1 +t)^{\sigma - \beta}  +  \mathcal{O}((\tau_t + 1)^2( 1 + t)^{ - \omega}) + 2DU_{L}U_{G}\epsilon_{app}\\
&\qquad+DU_G\sqrt{\frac{1}{ 1+ t -\tau_t}\sum_{k=\tau_t}^{t}E\Vert\nabla L(\theta_k,\gamma(k))\Vert^2}\sqrt{\frac{\sum\limits_{k=\tau_t}^{t}(2E[A_k^2] + 8E[B_k^2])}{1+t-\tau_t}}. 
\end{align*}

After applying the earlier square technique, we obtain

\begin{align*}
    \sum\limits_{k=\tau_t}^{t}E \Vert  \nabla L(\theta_k,\gamma(k))\Vert^2/(1 + t -\tau_t) &= \mathcal{O}(t^{\sigma - \beta}) +  \mathcal{O}(\log^2 t \cdot  t^{ - \omega})  + \mathcal{O}(\epsilon_{app}) + \mathcal{O}(\epsilon(t)),
\end{align*}

where $\epsilon(t) = \sum\limits_{k=\tau_t}^{t}(2E[A_k^2] + 8E[B_k^2])/(1+t-\tau_t)$.

\subsection{Proof of Theorem 4: Estimating the Average Reward for Constrained Natural Actor Critic}

We will use the same  notations as used in Section \ref{thm2_avg_reward}.

\textit{Proof of Theorem 4}

From the algorithm we have the update rule as 
\begin{align*}
    L_{t+1} = L_t + a(t)(q(t) + \sum_{k=1}^{k=N}\gamma_{k}(t)(h_k(t)-\alpha_k)) - L_t).
\end{align*}

\hfill\break
Unrolling the above recursion, we have
\begin{align*}
    y_{t+1}^2 &= (L_{t+1} - L_{t+1}^{*})^2\\
& =\left(L_t + a(t)\left(q(t) + \sum_{k=1}^{k=N}\gamma_{k}(t)(h_k(t)-\alpha_k) - L_t\right) - L_{t+1}^{*}\right)^2 \\
&= \left(y_t + L_t^* - L_{t+1}^* + a(t)\left(q(t) + \sum_{k=1}^{k=N}\gamma_{k}(t)(h_k(t)-\alpha_k) - L_t\right)\right)^2\\
&= y_t^2 + 2a(t)y_t(C_t - L_t) + 2y_t(L_t^* - L_{t+1}^*) + (L_t^* - L_{t+1}^* + a(t)(C_t - L_t))^2\\
&\leq y_t^2 + 2a(t)y_t(C_t - L_t) + 2y_t(L_t^* - L_{t+1}^*) + 2(L_t^* - L_{t+1}^*)^2 + 2a(t)^2(C_t - L_t)^2\\
&= y_t^2 - 2a(t)y_t^2 + 2a(t)y_t^2+ 2a(t)y_t(C_t - L_t) + 2y_t(L_t^* - L_{t+1}^*) + 2(L_t^* - L_{t+1}^*)^2 + 2a(t)^2(C_t - L_t)^2\\
&= y_t^2 - 2a(t)y_t^2 +  2a(t)y_t(C_t - L_t + y_t) + 2y_t(L_t^* - L_{t+1}^*) + 2(L_t^* - L_{t+1}^*)^2 + 2a(t)^2(C_t - L_t)^2\\
&= (1-2a(t))y_t^2 + 2a(t)y_t(C_t - L_t^*) +  2y_t(L_t^* - L_{t+1}^*) + 2(L_t^* - L_{t+1}^*)^2 + 2a(t)^2(C_t - L_t)^2,
\end{align*}

where $C_t = q(t) + \sum_{k=1}^{k=N}\gamma_{k}(t)(h_k(t)-\alpha_k) $. The first inequality is due to $( x + y )^2 \leq 2x^2 + 2y^2$.

Rearranging and summing from $\tau_t$ to $t$, we have

\begin{align*}
  \sum_{k = \tau_t}^{t}E[y_k^2] &\leq \underbrace{\sum_{k = \tau_t}^{t} \frac{1}{2a(k)}E(y_k^2 - y_{k+1}^2)}_{I_1} + \underbrace{\sum_{k=\tau_t}^{t}E[\hat{\Xi}(O_k,L_k,\theta_k,\gamma(k),q(k),h(k))] }_{I_2}\\
&\qquad+ \underbrace{\sum_{k = \tau_t}^{t} \frac{1}{a(k)}E[y_k(L_k^* - L_{k+1}^*)]}_{I_3} + \underbrace{\sum_{k = \tau_t}^{t} \frac{1}{a(k)}E[(L_k^* - L_{k+1}^*)^2]}_{I_4}\\
&\qquad+ \underbrace{\sum_{k = \tau_t}^{t}a(k)E[(C_k - L_k)^2]}_{I_5}.
\end{align*}
After carrying out an analysis similar to that of section \ref{thm2_avg_reward}, we obtain

\begin{align*}
    \sum_{k=\tau_t}^{t}E[y_k^2]/(1+t-\tau_t) = \mathcal{O}(t^{\omega-1}) + \mathcal{O}(\log^2t \cdot t^{-\omega}) + \mathcal{O}(t^{-2(\sigma - \omega)}).
\end{align*}

\subsection{Proof of Theorem 4: Estimating the convergence point of Critic  for Constrained Natural Actor Critic}

The update rule for the critic in Algorithm 2 is similar to the one in Algorithm 1. Hence we will get the inequality (\ref{inequality:critic}) for natural constrained actor critic also. After carrying out an analysis similar to Section \ref{thm2_critic_converge} we get

\begin{align*}
    \bigg(\sum_{k=\tau_t}^{t} E\Vert v_{k} - v^*(k) \Vert^2 \bigg)/(1+t-\tau_t)=\mathcal{O}\bigg(\frac{1}{t^{1-\omega}}\bigg) + \mathcal{O}\bigg(\frac{\log t}{t^{\omega}}\bigg) + \mathcal{O}\bigg(\frac{1}{t^{2(\sigma - \omega)}}\bigg).
\end{align*}

\subsection{Proof of Corollary 2}

We have the following result from Theorem 3:
\begin{align}\label{thm3_results}
    \min\limits_{0\leq k\leq t}E[\Vert\nabla L(\theta_k,\gamma(k))\Vert^2 ]  = \mathcal{O}(t^{\sigma - \beta}) ) + \mathcal{O}((\log t)^2 t^{-\omega }) + \mathcal{O}(\epsilon_{app}) + \mathcal{O}(\varepsilon(t)),
\end{align}
where,
\begin{align*}
    \varepsilon(t) &= (2\sum_{k=\tau_t}^{t}E\Vert A_k\Vert^2 + 8\sum_{k=\tau_t}^{t}E\Vert B_k\Vert^2)/(1 + t -\tau_t), \\
     A_k  &=  L_k - L(\theta_k,\gamma(k)),\\
     B_k &= v_k - v(\theta_k,\gamma(k)).
\end{align*}

From the results of Theorem 4, we have

\begin{align*}
    \varepsilon(t) = \mathcal{O}(t^{\omega-1}) + \mathcal{O}(\log t \cdot t^{-\omega}) + \mathcal{O}(t^{-2(\sigma - \omega)}).
\end{align*}
Putting this back in (\ref{thm3_results}), we obtain

\begin{align*}
    \min\limits_{0\leq k\leq t}E[\Vert\nabla L(\theta_k,\gamma(k))\Vert^2 ]   &=  \mathcal{O}(t^{\sigma - \beta})    + \mathcal{O}(\log^2t \cdot t^{-\omega})+ \mathcal{O}(t^{\omega-1}) + \mathcal{O}(\log t \cdot t^{-\omega}) + \mathcal{O}(t^{-2(\sigma - \omega)}) +\mathcal{O}(\epsilon_{app}) \\
    & =  \mathcal{O}(t^{\sigma - \beta})    + \mathcal{O}(\log^2t \cdot t^{-\omega})+ \mathcal{O}(t^{\omega-1}) +  \mathcal{O}(t^{-2(\sigma - \omega)}) +\mathcal{O}(\epsilon_{app}) \\
    &= \mathcal{O}(t^{\sigma - \beta})    + \mathcal{O}(\log^2t \cdot t^{-\omega}) +\mathcal{O}(t^{-2(\sigma - \omega)}) +\mathcal{O}(\epsilon_{app})
\end{align*}

The last equality above again holds because $(\sigma - \beta) > (\omega - 1)$.
Optimising over the choice of $\omega$, $\sigma$ and $\beta$, we have $\omega$ = 0.4, $\sigma$ = 0.6 and $\beta$ = 1, respectively. Hence,
\begin{align*}
    \min\limits_{0\leq k\leq t}E[\Vert\nabla L(\theta_k,\gamma(k))\Vert^2 ]   = \mathcal{O}(\log^2 t \cdot t^{-0.4}) +\mathcal{O}(\epsilon_{app}).
\end{align*}

Therefore, in order to obtain an $\epsilon$-approximate (ignoring the approximation error as with \cite{fta_2_timescale}) stationary point of the performance function (L($\theta,\gamma$)), namely,

\begin{align*}
    \min\limits_{0\leq k\leq T}E[\Vert\nabla L(\theta_k,\gamma(k))\Vert^2 ]  = \mathcal{O}(\log^2 T \cdot T^{-0.4}) +\mathcal{O}(\epsilon_{app})  \leq \mathcal{O}(\epsilon_{app}) + \epsilon,
\end{align*}
we need to set $T = \tilde{\mathcal{O}}(\epsilon^{-2.5})$.

\section{Experimental Setting}

For detailed information about the settings involved for the three Safety-Gym environments, Safety-PointGoal1-v0 (SPG1-v0), Safety-CarGoal1-v0 (SCG1-vo) and SafetyPointPush1-v0 (SPP1-v0), respectively, please see \href{https://safety-gymnasium.readthedocs.io/en/latest/environments/safe_navigation/goal.html}{Safety Gymnasium}.
%In addition, we consider two Grid-World settings with 100 and 400 states respectively. 
%For the Grid-World settings, corresponding to each state, there is a choice of four possible actions for the actions, viz., explore up, down, right or left. The agent probabilistically moves one step at any instant based on the current state and action and receives a cost and a constraint cost upon transitioning to a particular state. 
%We consider a single constraint in each of these settings. The environment consists of a goal state as well as a fixed starting state. When the agent reaches the goal state, we assume that any action sends the agent back to the same starting state. Some states are assigned a high constraint cost that the agent must avoid. The goal state has single-stage reward as well as the constraint cost as 0. 
We experimentally compare C-AC and C-NAC algorithms with C-DQN on the three settings. 

The C-DQN algorithm is obtained from the algorithm Deep Q-Network (DQN) \cite{mnih} by (a) modifying the basic setting to incorporate the average reward framework from the discounted reward setting considered there and (b) relaxing the constraints to form a Lagrangian in a similar manner as the C-AC and C-NAC algorithms. We also update the Lagrange parameter using the same updates as C-AC and C-NAC, respectively. This ensures a fair comparison across all the algorithms. Note, however, that such an update of the Lagrange parameter had not been used previously in the context of the DQN algorithm. 
We have taken the threshold level to be 0.1 for all the settings.

\begin{table*}
  \centering
  \caption{Comparision of C-AC , C-NAC and C-DQN in terms of constraint cost $\pm$ standard error.}
  \begin{tabular}{|c|c|c|c|}
    \hline
  Algorithm & SafetyPointGoal1-v0 & SafetyCarGoal1-v0 & SafetyPointPush1-v0\\ \hline 
  C-AC &  0.038 $\pm$ 0.026 & 0.0195 $\pm$  0.027 & 0.028 $\pm$ 0.018\\ \hline
  C-NAC & 0.049 $\pm$ 0.045 & 0.047 $\pm$ 0.046 & 0.0295 $\pm$ 0.032\\ \hline
  C-DQN & 0.039 $\pm$  0.023 & 0.00872 $\pm$ 0.0076 & 0.035 $\pm$ 0.022\\ \hline
  \end{tabular}
  \label{tab:experiment_avgconstraint}
\end{table*}

We observe that the constraint threshold is satisfied by all the three algorithms in the three different environments. Table \ref{tab:experiment_avgconstraint} exhibits  the values of the constraint costs (both average and standard error) over ten independent runs of each algorithm. As can be seen from the table as well as the bottom row of plots in Figure \ref{fig:experiment}, the constraint threshold is met by all the three algorithms on each of the settings.

\end{document}